\documentclass{article}

\PassOptionsToPackage{numbers, compress}{natbib}
\usepackage[preprint]{neurips_2026}

\usepackage{inconsolata}
\usepackage[utf8]{inputenc} %
\usepackage[T1]{fontenc}    %
\usepackage{hyperref}       %
\usepackage{url}            %
\usepackage{booktabs}       %
\usepackage{amsfonts}       %
\usepackage{nicefrac}       %
\usepackage{microtype}      %
\usepackage{xcolor}         %
\usepackage{wrapfig}
\usepackage{enumitem}

\usepackage{siunitx}
\usepackage{mathtools}
\usepackage{bm}
\usepackage{amssymb,amsfonts}
\usepackage{xfrac}

\usepackage[capitalize,noabbrev]{cleveref}
\DeclareRobustCommand{\abbrevcrefs}{%
\Crefname{equation}{Eq.}{Eqs.}%
}
\crefname{problem}{problem}{problems}
\Crefname{problem}{Problem}{Problems}
\DeclareRobustCommand{\Cabbref}[1]{{\abbrevcrefs\Cref{#1}}}

\usepackage{tikz}
\usepackage{pgfplots,pgfplotstable}
\usetikzlibrary{arrows,arrows.meta,bending,positioning,patterns,patterns.meta,calc,decorations.pathreplacing}
\pgfplotsset{compat=newest}
\usepgfplotslibrary{fillbetween,groupplots,statistics,colorbrewer}

\usepackage{amsthm}
\theoremstyle{definition}
\newtheorem{definition}{Definition}

\newtheorem{remark}{Remark}

\newtheorem{problem}{Problem}
\newtheorem*{goal}{Goal}

\newtheorem*{complexity}{Complexity}
\usepackage{thmtools}
\usepackage{thm-restate}
\declaretheoremstyle[spaceabove=7pt,spacebelow=0pt]{thmstyle}

\usepackage{algorithm}
\usepackage[noend]{algpseudocode}

\algrenewcommand\algorithmicrequire{\textbf{Input}}
\algrenewcommand\algorithmicensure{\textbf{Output}}
\algnewcommand{\IIf}[2]{\State\algorithmicif\ #1\ \algorithmicthen\ #2}
\algnewcommand{\IElse}[1]{\State\algorithmicelse\ #1}
\algnewcommand{\IElseIf}[2]{\State\algorithmicelse\ \algorithmicif\ #1\ \algorithmicthen\ #2}
\makeatletter
\patchcmd{\ALG@doentity}{\item[]\nointerlineskip}{}{}{}
\ifthenelse{\equal{\ALG@noend}{t}}%
  {\algtext*{EndNIElse}}
  {}%
\ifthenelse{\equal{\ALG@noend}{t}}%
  {\algtext*{EndNIElseIf}}
  {}%
\makeatother
\makeatletter
\AddToHook{env/algorithmic/before}{\def\@currentcounter{ALG@line}}
\makeatother
\crefalias{ALG@line}{line}

\renewcommand{\Comment}[2][.6\linewidth]{%
  \leavevmode\hfill\makebox[#1][l]{$\triangleright$~#2}}

\definecolor{palette-blue}{HTML}{2274a5}
\definecolor{palette-orange}{HTML}{f75c03}
\definecolor{palette-yellow}{HTML}{f1c40f}
\definecolor{palette-pink}{HTML}{d90368}
\definecolor{palette-green}{HTML}{00cc66}
\definecolor{palette-purple}{HTML}{662c91}
\definecolor{palette-gray}{HTML}{6d545d}
\definecolor{palette-dgreen}{HTML}{5b8c5a}
\definecolor{palette-ddgreen}{HTML}{293f14}
\definecolor{palette-violet}{HTML}{1c0b19}
\definecolor{palette-gold}{HTML}{ffbf00}
\definecolor{palette-brown}{HTML}{492c1d}
\definecolor{palette-red}{HTML}{ff3a20}
\pgfplotscreateplotcyclelist{cpalette}{{palette-blue},{palette-orange},{palette-yellow},{palette-pink},{palette-green},{palette-purple},{palette-gray},{palette-brown}, {palette-dgreen}}

\usepackage[breakable]{tcolorbox}
\tcolorboxenvironment{example}{
    colback=palette-yellow!5!white,
    boxrule=0pt,
    boxsep=2pt,
    before skip=\topsep,
    after skip=\topsep,
    boxrule=0pt,
    sharp corners,
    colframe=white,
    breakable,
    left=5pt, right=5pt,
}
\tcolorboxenvironment{remark}{
    colback=palette-green!5!white,
    boxrule=0pt,
    boxsep=2pt,
    before skip=\topsep,
    after skip=\topsep,
    boxrule=0pt,
    sharp corners,
    colframe=white,
    breakable,
    left=5pt, right=5pt,
}
\tcolorboxenvironment{assumption}{
    colback=palette-purple!5!white,
    boxrule=0pt,
    boxsep=2pt,
    before skip=\topsep,
    after skip=\topsep,
    boxrule=0pt,
    sharp corners,
    colframe=white,
    breakable,
    left=5pt, right=5pt,
}
\tcolorboxenvironment{problem}{
    colback=palette-orange!5!white,
    boxrule=0pt,
    boxsep=2pt,
    before skip=\topsep,
    after skip=\topsep,
    boxrule=0pt,
    sharp corners,
    colframe=white,
    breakable,
    left=5pt, right=5pt,
}
\tcolorboxenvironment{goal}{
    colback=palette-blue!5!white,
    boxrule=0pt,
    boxsep=2pt,
    before skip=\topsep,
    after skip=\topsep,
    boxrule=0pt,
    sharp corners,
    colframe=white,
    breakable,
    left=5pt, right=5pt,
}

\newcommand{\supp}{\text{supp}}

\newcommand{\defeq}{\vcentcolon=}

\newcommand{\implws}{\,\textvisiblespace\,}

\DeclareRobustCommand{\bigo}{%
  \text{\usefont{OMS}{cmsy}{m}{n}O}%
}
\makeatletter
\newcommand{\circminus}{\mathbin{\text{\@circminus}}}
\newcommand{\@circminus}{%
  \ooalign{\hidewidth\raise1ex\hbox{$\circ$}\hidewidth\cr$\m@th-$\cr}%
}
\makeatother
\DeclareFontFamily{U}{matha}{\hyphenchar\font45}
\DeclareFontShape{U}{matha}{m}{n}{ <-6> matha5 <6-7> matha6 <7-8>
matha7 <8-9> matha8 <9-10> matha9 <10-12> matha10 <12-> matha12 }{}
\DeclareSymbolFont{matha}{U}{matha}{m}{n}
\DeclareFontFamily{U}{mathx}{\hyphenchar\font45}
\DeclareFontShape{U}{mathx}{m}{n}{ <-6> mathx5 <6-7> mathx6 <7-8>
mathx7 <8-9> mathx8 <9-10> mathx9 <10-12> mathx10 <12-> mathx12 }{}
\DeclareSymbolFont{mathx}{U}{mathx}{m}{n}
\DeclareMathDelimiter{\liv} {4}{matha}{"76}{mathx}{"30}
\DeclareMathDelimiter{\riv} {5}{matha}{"77}{mathx}{"38}

\newcommand{\softmax}{\text{softmax}}
\newcommand*{\tran}{^{^{\mkern-1.5mu\mathsf{T}}}} %

\newcommand{\pimin}{\pi_{\text{min}}}
\newcommand{\piminh}{\pi_{\text{min-h}}}
\newcommand{\pihto}{\pi_{\text{H2O}}}
\newcommand{\pihtoh}{\pi_{\text{H2O-h}}}

\newcommand{\llama}{\texttt{\textbf{\color{palette-red}Llama3}}}
\newcommand{\qwen}{\texttt{\textbf{\color{palette-dgreen}Qwen3}}}

\pgfplotstableread[col sep=comma]{plots/mae_local/t1.csv}\maelocaltadata
\pgfplotstableread[col sep=comma]{plots/mae_local/t2.csv}\maelocaltbdata
\pgfplotstableread[col sep=comma]{plots/mae_local/t5.csv}\maelocaltcdata
\pgfplotstableread[col sep=comma]{plots/mae_local/t10.csv}\maelocaltddata
\pgfplotstableread[col sep=comma]{plots/mae_local/t50.csv}\maelocaltedata

\pgfplotstableread[col sep=comma]{plots/ruler/llama3_3b/normal/estimated_attention__h2o.csv}\rleah
\pgfplotstableread[col sep=comma]{plots/ruler/llama3_3b/normal/estimated_attention__h2o_norm.csv}\rleahn
\pgfplotstableread[col sep=comma]{plots/ruler/llama3_3b/normal/estimated_attention__min_var.csv}\rleamv
\pgfplotstableread[col sep=comma]{plots/ruler/llama3_3b/normal/estimated_attention__min_var_full_harmonic.csv}\rleamvfh
\pgfplotstableread[col sep=comma]{plots/ruler/llama3_3b/normal/h2o.csv}\rlho
\pgfplotstableread[col sep=comma]{plots/ruler/llama3_3b/normal/knorm.csv}\rlkn
\pgfplotstableread[col sep=comma]{plots/ruler/llama3_3b/normal/snapkv.csv}\rlsnap
\pgfplotstableread[col sep=comma]{plots/ruler/llama3_3b/normal/streaming_llm.csv}\rlsllm
\pgfplotstableread[col sep=comma]{plots/ruler/llama3_3b/normal/tova.csv}\rltova
\pgfplotstableread[col sep=comma]{plots/ruler/qwen3_4b/normal/estimated_attention__h2o.csv}\rqeah
\pgfplotstableread[col sep=comma]{plots/ruler/qwen3_4b/normal/estimated_attention__h2o_norm.csv}\rqeahn
\pgfplotstableread[col sep=comma]{plots/ruler/qwen3_4b/normal/estimated_attention__min_var.csv}\rqeamv
\pgfplotstableread[col sep=comma]{plots/ruler/qwen3_4b/normal/estimated_attention__min_var_full_harmonic.csv}\rqeamvfh
\pgfplotstableread[col sep=comma]{plots/ruler/qwen3_4b/normal/h2o.csv}\rqho
\pgfplotstableread[col sep=comma]{plots/ruler/qwen3_4b/normal/knorm.csv}\rqkn
\pgfplotstableread[col sep=comma]{plots/ruler/qwen3_4b/normal/snapkv.csv}\rqsnap
\pgfplotstableread[col sep=comma]{plots/ruler/qwen3_4b/normal/streaming_llm.csv}\rqsllm
\pgfplotstableread[col sep=comma]{plots/ruler/qwen3_4b/normal/tova.csv}\rqtova
\pgfplotstableread[col sep=comma]{plots/longbench/llama3_3b/normal/estimated_attention__h2o.csv}\lbleah
\pgfplotstableread[col sep=comma]{plots/longbench/llama3_3b/normal/estimated_attention__h2o_norm.csv}\lbleahn
\pgfplotstableread[col sep=comma]{plots/longbench/llama3_3b/normal/estimated_attention__min_var.csv}\lbleamv
\pgfplotstableread[col sep=comma]{plots/longbench/llama3_3b/normal/estimated_attention__min_var_full_harmonic.csv}\lbleamvfh
\pgfplotstableread[col sep=comma]{plots/longbench/llama3_3b/normal/h2o.csv}\lblho
\pgfplotstableread[col sep=comma]{plots/longbench/llama3_3b/normal/knorm.csv}\lblkn
\pgfplotstableread[col sep=comma]{plots/longbench/llama3_3b/normal/snapkv.csv}\lblsnap
\pgfplotstableread[col sep=comma]{plots/longbench/llama3_3b/normal/streaming_llm.csv}\lblsllm
\pgfplotstableread[col sep=comma]{plots/longbench/llama3_3b/normal/tova.csv}\lbltova
\pgfplotstableread[col sep=comma]{plots/longbench/qwen3_4b/normal/estimated_attention__h2o.csv}\lbqeah
\pgfplotstableread[col sep=comma]{plots/longbench/qwen3_4b/normal/estimated_attention__h2o_norm.csv}\lbqeahn
\pgfplotstableread[col sep=comma]{plots/longbench/qwen3_4b/normal/estimated_attention__min_var.csv}\lbqeamv
\pgfplotstableread[col sep=comma]{plots/longbench/qwen3_4b/normal/estimated_attention__min_var_full_harmonic.csv}\lbqeamvfh
\pgfplotstableread[col sep=comma]{plots/longbench/qwen3_4b/normal/h2o.csv}\lbqho
\pgfplotstableread[col sep=comma]{plots/longbench/qwen3_4b/normal/knorm.csv}\lbqkn
\pgfplotstableread[col sep=comma]{plots/longbench/qwen3_4b/normal/snapkv.csv}\lbqsnap
\pgfplotstableread[col sep=comma]{plots/longbench/qwen3_4b/normal/streaming_llm.csv}\lbqsllm
\pgfplotstableread[col sep=comma]{plots/longbench/qwen3_4b/normal/tova.csv}\lbqtova

\pgfplotstableread[col sep=comma]{plots/ruler/llama3_3b/normal/estimated_attention__h2o_scores.csv}\srleah
\pgfplotstableread[col sep=comma]{plots/ruler/llama3_3b/normal/estimated_attention__h2o_norm_scores.csv}\srleahn
\pgfplotstableread[col sep=comma]{plots/ruler/llama3_3b/normal/estimated_attention__min_var_scores.csv}\srleamv
\pgfplotstableread[col sep=comma]{plots/ruler/llama3_3b/normal/estimated_attention__min_var_full_harmonic_scores.csv}\srleamvfh
\pgfplotstableread[col sep=comma]{plots/ruler/llama3_3b/normal/h2o_scores.csv}\srlho
\pgfplotstableread[col sep=comma]{plots/ruler/llama3_3b/normal/knorm_scores.csv}\srlkn
\pgfplotstableread[col sep=comma]{plots/ruler/llama3_3b/normal/snapkv_scores.csv}\srlsnap
\pgfplotstableread[col sep=comma]{plots/ruler/llama3_3b/normal/streaming_llm_scores.csv}\srlsllm
\pgfplotstableread[col sep=comma]{plots/ruler/llama3_3b/normal/tova_scores.csv}\srltova
\pgfplotstableread[col sep=comma]{plots/ruler/qwen3_4b/normal/estimated_attention__h2o_scores.csv}\srqeah
\pgfplotstableread[col sep=comma]{plots/ruler/qwen3_4b/normal/estimated_attention__h2o_norm_scores.csv}\srqeahn
\pgfplotstableread[col sep=comma]{plots/ruler/qwen3_4b/normal/estimated_attention__min_var_scores.csv}\srqeamv
\pgfplotstableread[col sep=comma]{plots/ruler/qwen3_4b/normal/estimated_attention__min_var_full_harmonic_scores.csv}\srqeamvfh
\pgfplotstableread[col sep=comma]{plots/ruler/qwen3_4b/normal/h2o_scores.csv}\srqho
\pgfplotstableread[col sep=comma]{plots/ruler/qwen3_4b/normal/knorm_scores.csv}\srqkn
\pgfplotstableread[col sep=comma]{plots/ruler/qwen3_4b/normal/snapkv_scores.csv}\srqsnap
\pgfplotstableread[col sep=comma]{plots/ruler/qwen3_4b/normal/streaming_llm_scores.csv}\srqsllm
\pgfplotstableread[col sep=comma]{plots/ruler/qwen3_4b/normal/tova_scores.csv}\srqtova
\pgfplotstableread[col sep=comma]{plots/longbench/llama3_3b/normal/estimated_attention__h2o_scores.csv}\slbleah
\pgfplotstableread[col sep=comma]{plots/longbench/llama3_3b/normal/estimated_attention__h2o_norm_scores.csv}\slbleahn
\pgfplotstableread[col sep=comma]{plots/longbench/llama3_3b/normal/estimated_attention__min_var_scores.csv}\slbleamv
\pgfplotstableread[col sep=comma]{plots/longbench/llama3_3b/normal/estimated_attention__min_var_full_harmonic_scores.csv}\slbleamvfh
\pgfplotstableread[col sep=comma]{plots/longbench/llama3_3b/normal/h2o_scores.csv}\slblho
\pgfplotstableread[col sep=comma]{plots/longbench/llama3_3b/normal/knorm_scores.csv}\slblkn
\pgfplotstableread[col sep=comma]{plots/longbench/llama3_3b/normal/snapkv_scores.csv}\slblsnap
\pgfplotstableread[col sep=comma]{plots/longbench/llama3_3b/normal/streaming_llm_scores.csv}\slblsllm
\pgfplotstableread[col sep=comma]{plots/longbench/llama3_3b/normal/tova_scores.csv}\slbltova
\pgfplotstableread[col sep=comma]{plots/longbench/qwen3_4b/normal/estimated_attention__h2o_scores.csv}\slbqeah
\pgfplotstableread[col sep=comma]{plots/longbench/qwen3_4b/normal/estimated_attention__h2o_norm_scores.csv}\slbqeahn
\pgfplotstableread[col sep=comma]{plots/longbench/qwen3_4b/normal/estimated_attention__min_var_scores.csv}\slbqeamv
\pgfplotstableread[col sep=comma]{plots/longbench/qwen3_4b/normal/estimated_attention__min_var_full_harmonic_scores.csv}\slbqeamvfh
\pgfplotstableread[col sep=comma]{plots/longbench/qwen3_4b/normal/h2o_scores.csv}\slbqho
\pgfplotstableread[col sep=comma]{plots/longbench/qwen3_4b/normal/knorm_scores.csv}\slbqkn
\pgfplotstableread[col sep=comma]{plots/longbench/qwen3_4b/normal/snapkv_scores.csv}\slbqsnap
\pgfplotstableread[col sep=comma]{plots/longbench/qwen3_4b/normal/streaming_llm_scores.csv}\slbqsllm
\pgfplotstableread[col sep=comma]{plots/longbench/qwen3_4b/normal/tova_scores.csv}\slbqtova

\pgfplotstableread[col sep=comma]{plots/ruler/llama3_3b/normal/effective_cr/estimated_attention__h2o_effective_cr.csv}\echto
\pgfplotstableread[col sep=comma]{plots/ruler/llama3_3b/normal/effective_cr/estimated_attention__h2o_norm_effective_cr.csv}\echtoh
\pgfplotstableread[col sep=comma]{plots/ruler/llama3_3b/normal/effective_cr/estimated_attention__min_var_effective_cr.csv}\ecmin
\pgfplotstableread[col sep=comma]{plots/ruler/llama3_3b/normal/effective_cr/estimated_attention__min_var_full_harmonic_effective_cr.csv}\ecminfh

\title{A Probabilistic Interpretation of KV Cache Eviction}

\author{%
    Renato Geh\\
    University of California, Los Angeles\\
    \texttt{renatolg@cs.ucla.edu}\\
    \And
    Alex Chen\\
    University of California, Los Angeles\\
    \texttt{itisalex@ucla.edu}
    \And
    Daniel Israel\\
    University of California, Los Angeles\\
    \texttt{disrael@cs.ucla.edu}\\
    \And
    Aditya Grover\\
    University of California, Los Angeles\\
    \texttt{adityag@cs.ucla.edu}
    \And
    Guy Van den Broeck\\
    University of California, Los Angeles\\
    \texttt{guyvdb@cs.ucla.edu}\\
}

\begin{document}

\maketitle

\begin{abstract}
    The premise and promise of KV (cache) eviction is simple: higher throughput can be achieved by evicting some entries from the KV cache, at a negligible cost to quality.
    This holds empirically for many existing methods, though most rely on creative heuristics for selecting which entries to drop.
    Despite recent advances, the problem of KV eviction has remained informal in the literature.
    This paper aims to properly formalize this problem through the lens of probabilistic reasoning and reveal what can be learned from this perspective. %
    Concretely, we (1) formalize the problem of KV eviction and, unfortunately, prove that it is computationally hard, (2) show that by framing it probabilistically, KV eviction reduces to the problem of expectation estimation, which can be approximated through sampling, (3) show that through this probabilistic interpretation, correcting for evicted entries during decoding---a previously ignored problem---becomes feasible, and (4) reveal that existing methods in the literature are zero-variance biased estimators that can be easily adapted in order to enable decode time correction.
    In practice, we show that this probabilistic version of KV eviction coupled with decode time correction is more robust to different tasks compared to existing eviction methods and achieves competitive performance at the same compression budget.
\end{abstract}

\section{Introduction}

Key-Value (KV) cache compression poses a natural proposition: increase throughput and decrease memory footprint by sacrificing some precision in the KV cache.
Popular methods for this utilize quantization \citep{hooper24,zandieh26}, compaction \citep{zweiger26,eyuboglu25}, eviction \citep{xiao24,li24,oren24,zhang23}, \emph{inter alia}, to shrink down the size of the cache and speed up inference.
Among these techniques, KV (cache) eviction presents the simplest solution: given a prompt and the KV cache associated with it, evict some entries of the cache at a (hopefully) negligible cost to generation quality.
Despite this simplicity, the problem itself is quite challenging, and existing KV eviction strategies elaborate sophisticated heuristics in order to identify which entries to evict, performing surprisingly well in practice.
Although these recent advances are indeed quite impressive, there is a distinct lack of formality in the literature.
This paper aims to address this, and in its wake expose key insights on the structure of KV eviction.

As a first result, by formalizing KV eviction we reveal that the problem of KV eviction is unfortunately computationally hard. 
However, through a probabilistic interpretation of the attention mechanism, eviction can be viewed as a problem of expectation estimation, which is well studied in the statistical methods literature \citep{owen13}.
We then reveal that the usual KV eviction setup ignores the distributional shift that occurs after eviction: because of missing entries, the distribution over all entries is distorted, and therefore so is the expectation.
Thus, within this framing an expectation estimator (and KV eviction strategy) aims to achieve low error at decode time by either reducing its bias, variance, or both.
Interestingly, most existing KV eviction methods are trivially subsumed within this framework.

The common approach in KV eviction is to score entries by using the attention weights and values, and then take the top-$k$ entries according to these scores \citep{li24,zhang23,oren24}; this estimator is zero-variance as it is deterministic. However, it can have an arbitrarily large bias.
We instead suggest that these scores should be viewed as unnormalized distributions on entries which to evict, allowing us to later reuse them as a proposal distribution to correct for the distortion caused by eviction through self-normalized importance sampling.
This is an attractive option because self-normalized importance sampling has reasonable upper bounds on bias and variance \citep{owen13,agapiou17}.
\Cref{fig:big-picture} summarizes our approach.

\begin{figure}[t]
\centering%
\begin{tikzpicture}
    \node[anchor=north west] (head-comp) at (0, 0) {$\quad\overbrace{\softmax\underbrace{\left(\frac{\mathbf{Q}\cdot\mathbf{K}\tran}{\sqrt{d}}\right)}_{\tilde{p}(\mathbf{V})}}^{p(\mathbf{V})}\cdot\mathbf{V}=\mathbb{E}_{p}[\mathbf{V}]$};

    \node[anchor=north west] (pt) at ($(head-comp.south west) + (0, 0.0)$) {$\begin{aligned}
        \quad p_t(\mathbf{V})=\,&[\overbracket[0.1ex]{0.3}^{1}, \overbracket[0.1ex]{0.1}, \overbracket[0.1ex]{0.4}^{3}, \overbracket[0.1ex]{0.2}^{4}]\\
        \mathbf{V}=\,&[\underbracket[0.1ex]{2.0}, \underbracket[0.1ex]{3.0}, \underbracket[0.1ex]{1.0}, \underbracket[0.1ex]{5.0}]
    \end{aligned}$};
    \draw[draw=none,fill opacity=0.25,fill=palette-gray] ($(pt.north west) + (2.46, -0.4)$) -- +(0.50, 0.0) -- +(0.50, -1.00) -- +(0.0, -1.00) -- cycle;

    \node[anchor=north west] (topk) at ($(pt.south west) + (0, -0.1)$) {\color{palette-red}\textbf{Top-$\bm{k}$:}};
    \node[anchor=north west] (topk-pt) at ($(topk.south west) + (0.5, 0.0)$) {$\begin{aligned}
        \bm{I}=\text{top-}k(p_t(\mathbf{V}))=\{1,3,4\}
    \end{aligned}$};

    \node[anchor=north west] (prob-eviction) at (topk-pt.south west -| topk.south west) {\color{palette-blue}\textbf{Probabilistic eviction:}};
    \node[anchor=north west] (prob-eviction-pt) at ($(prob-eviction.south west) + (0.15, 0.0)$) {$\begin{aligned}
        &\pi(\mathbf{v}_i)=\frac{p_t(\mathbf{v}_i)\cdot|\mathbf{v}_i|}{\sum_{j=1}^t p_t(\mathbf{v}_j)\cdot|\mathbf{v}_j|}\\
        &\bm{u}=\{u_1,u_2,\dots,u_m\}\sim\pi\\
        &c(\mathbf{v}_i)=\text{counts of $\mathbf{v}_i$ in $\bm{u}$}=[3, 0, 3, 14]\\
        &\bm{I}=\text{unique}(\bm{u})=\{1,3,4\}
    \end{aligned}$};

    \node (timeline-start) at ($(pt.south west) + (0, -4.7)$) {};
    \node (t) at ($(timeline-start.east) + (0.43\textwidth, -0.1)$) {};
    \node (n) at ($(timeline-start.east) + (0.85\textwidth, -0.1)$) {};
    \draw[thick,-{Stealth}] (timeline-start) -- +(0.95\textwidth, 0);
    \draw[thick] ($(timeline-start.east) + (0, -0.1)$) -- +(0, 0.2);
    \draw[thick] (t) -- +(0, 0.2) node[midway,below] {\strut$t=4$};
    \draw[thick] (n) -- +(0, 0.2) node[midway,below] {\strut$n=5$};
    \node at ($(timeline-start.east) + (0.225\textwidth, -0.3)$) {\color{palette-dgreen}\textbf{eviction time}};
    \node at ($(timeline-start.east) + (0.675\textwidth, -0.3)$) {\color{palette-purple}\textbf{decode time}};

    \draw[thick,dashed] (t) -- ($(head-comp.north -| t) + (0, -0.5)$);
    
    \node[anchor=north west] (pn) at ($(head-comp -| t) + (0.2, 0)$) {$\begin{aligned}
        \quad\mathbb{E}_{p_n}[\mathbf{V}]=\,&2.1
    \end{aligned}$};
    \node[anchor=south west] (pn-label) at ($(pn.north west) + (0, 0.1)$) {\textbf{Ground-truth}:};
    \node[anchor=south west] (pn-available-label) at ($(pn.south west) + (0, -0.5)$) {\textbf{Available}:};
    \node[anchor=north west] (pn-available) at (pn-available-label.south west) {$\begin{aligned}
        \quad\tilde{p}_n(\mathbf{V})=\,&[10, \text{\implws}, 20, 40, 10, 20]\\
        \mathbf{V}=\,&[2.0, \text{\implws}, 1.0, 5.0, 2.0]\\
        \bm{I}_n=\,&\{1,3,4\}\cup\{5\}\\
    \end{aligned}$};
    
    \node[anchor=north west] (dec-topk) at ($(pn-available.south west) + (0, 0.1)$) {\color{palette-red}\textbf{Top-$\bm{k}$} (uncorrected) \textbf{:}};
    \node[anchor=north west] (dec-topk-pt) at (dec-topk.south west) {$\displaystyle\mathbb{E}_{p_n}[\mathbf{V}|\bm{I}_n]=\sum_{i\in\bm{I}_n}\frac{\tilde{p}_n(\mathbf{v}_i)}{\sum_{j\in\bm{I}_n}\tilde{p}_n(\mathbf{v}_j)}\cdot\mathbf{v}_i\approx 1.84$};
    \draw[draw=none,fill opacity=0.25,fill=palette-red] ($(dec-topk-pt.south west) + (0, 1.25)$) -- ($(dec-topk-pt.south east) + (0, 1.25)$) -- (dec-topk-pt.south east) -- (dec-topk-pt.south west) -- cycle;
    \node[anchor=north] (dec-topk-pt-label) at ($(dec-topk-pt.south) + (0, -0.1)$) {\small\textbf{Bias:} unbounded \qquad \textbf{Variance:} 0};

    \node[anchor=north west] (dec-prob) at ($(dec-topk-pt.south west) + (0, -0.5)$) {\color{palette-blue}\textbf{Probabilistic eviction} (corrected) \textbf{:}};
    \node[anchor=north west] (dec-prob-pt) at ($(dec-prob.south west) + (0, -0.1)$) {$\begin{aligned}
        \mu(\mathbf{V};m)=\frac{\hat{\mathbb{E}}_{p_n}[\mathbf{V}_{\bm{I}_n}]+\mathbb{E}_{\tilde{p}_n}[\mathbf{V}_{t+1:n}]}{\hat{z}_{\bm{I}_n}+z_{t+1:n}}\approx 2.18
    \end{aligned}$};
    \draw[draw=none,fill opacity=0.25,fill=palette-blue] ($(dec-prob-pt.south west) + (0, 1.25)$) -- ($(dec-prob-pt.south east) + (0, 1.25)$) -- (dec-prob-pt.south east) -- (dec-prob-pt.south west) -- cycle;
    \node[anchor=north] (dec-prob-pt-label) at ($(dec-prob-pt.south) + (0, -0.1)$) {\small\textbf{Bias:} $\bigo(\sfrac{1}{m})$ \qquad \textbf{Variance:} $\bigo(\sfrac{1}{m})$};
    
\end{tikzpicture}
\caption{\textbf{Correction can attenuate the effects of eviction at decode time.} Evicting entries deforms the distribution and biases the expectation, as the normalizing constant is altered by the evicted entries.
Existing eviction methods deterministically evict the top-$k$ scored entries, which although zero-variance, may have arbitrarily large bias. Probabilistic eviction addresses this unintended deformation through self-normalized importance sampling, whose error is upper bounded by $\bigo(m^{-1})$, where $m$ is the number of samples.}\label{fig:big-picture}
\end{figure}
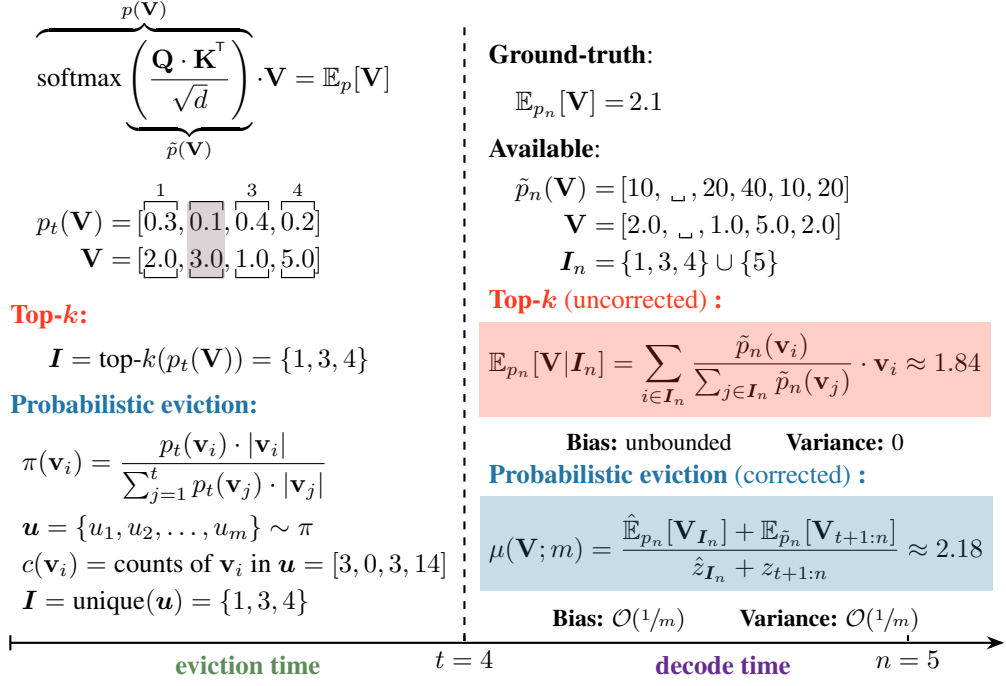

In summary, our contributions are as follows:
\begin{enumerate}[leftmargin=0.75cm,topsep=0pt,itemsep=-0.125ex]
    \item \textbf{We show that KV eviction is \textsf{NP}-complete} (\Cref{sec:hardness});
    \item \textbf{We formalize KV eviction as an expectation estimation problem} and show an asymptotically correct estimator for the decode time expectation with bounded bias and variance (\Cref{sec:prob-interp,sec:correction});
    \item \textbf{We show the interplay of bias versus variance} in this estimator and how existing eviction methods are subsumed by this framework (\Cref{sec:bias-vs-variance,app:subsumed-methods});
    \item \textbf{We empirically validate our findings} by showing that coupling probabilistic eviction and decode time correction is not only competitive at the same compression budget but is more robust across different tasks (\Cref{sec:robustness}).
\end{enumerate}

\section{Preliminaries}

We start by briefly introducing notation and reviewing the attention mechanism of \citet{vaswani17} as well as the concepts of KV caching and KV eviction.
Throughout this paper, we shall denote random variables (RVs) as single upper case letters (e.g.\ $X,Y,Z$), their values or scalars as single lower case letters (e.g.\ $x,y,z$), sets as bold letters (e.g.\ $\bm{X},\bm{Y},\bm{Z}$ for RVs or $\bm{x},\bm{y},\bm{z}$ for values), and sequences or matrices as bold upright lower and upper case letters respectively (e.g.\ $\mathbf{x},\mathbf{y},\mathbf{z}$ for sequences and vectors, and $\mathbf{X},\mathbf{Y},\mathbf{Z}$ for matrices).

\textbf{Attention.} Let $\mathbf{K},\mathbf{Q},\mathbf{V}\in\mathbb{R}^{n\times d}$ be the key, query and value embedding matrices for a particular attention head, where $n$ is the token sequence length and $d$ is the embedding dimension.
Attention is computed as
\begin{equation}\label{eq:attention}
    \mathbf{O}\defeq\softmax\left(\frac{\mathbf{Q}\cdot\mathbf{K}\tran}{\sqrt{d}}\right)\cdot\mathbf{V},
\end{equation}
where the denominator inside the softmax is often omitted as it is mostly used as a regularizer during training (and here we shall do so as well).
This computation ultimately yields a matrix $\mathbf{O}\in\mathbb{R}^{n\times d}$ corresponding to the attention values for each of the $n$ entries.
The key, query and value embeddings for the next layer are then computed as a function of $\mathbf{O}$.

\textbf{KV caching.} The intermediate softmax computation is $n\times n$ which can be prohibitively large at longer sequences.
Because generation depends only on the next token attention values, it is therefore useful to, instead of fully computing the softmax matrix in \Cref{eq:attention}, compute only the last row of the softmax matrix and cache $\mathbf{K}$ and $\mathbf{V}$ at each attention head 
\begin{equation}\label{eq:attention-kv-cache}
    \underbracket[0.1ex]{\mathbf{o}}_{1\times d}\defeq\overbracket[0.1ex]{\softmax\Big(\underbracket[0.1ex]{\mathstrut\mathbf{Q}_{-1}}_{1\times d}\cdot\underbracket[0.1ex]{\mathstrut\mathbf{K}\tran}_{d\times n}\Big)}^{1\times n}\cdot\overbracket[0.1ex]{\mathbf{V}}^{n\times d}=\softmax(\overbracket[0.1ex]{[\underbracket[0.1ex]{\mathstrut\mathbf{k}_i}_{1\times d}\cdot\underbracket[0.1ex]{\mathstrut\mathbf{q}_n}_{d\times 1}]_{i=1}^n}^{1\times n})\cdot\overbracket[0.1ex]{\mathbf{V}}^{n\times d};
\end{equation}
where here we annotate matrices and vectors with their dimensions for clarity.
The cached keys and values $\mathbf{k}_{1},\mathbf{k}_2,\dots,\mathbf{k}_n$ and $\mathbf{v}_1,\mathbf{v}_2,\dots,\mathbf{v}_n$ are then reused in the next $n+1$ forward pass.

\textbf{KV cache eviction.} One popular strategy to further speed up computation is to further reduce the softmax computation in \Cref{eq:attention-kv-cache} to a subset of entries.
Let $\bm{I}\subseteq [1..n]$ be the set of indices of which entries to keep, where here we denote $[i..j]$ as the set of positive integers in $[i,j]$.
The KV cache evicted attention of an attention head with kept entries $\bm{I}$ is given by
\begin{equation}\label{eq:attention-kv-cache-evicted}
    \mathbf{o}_{\bm{I}}\defeq\softmax\left(\left[\mathbf{k}_i\cdot\mathbf{q}_n\right]_{i\in\bm{I}}\right)\cdot\mathbf{V}_{\bm{I}},
\end{equation}
where $\mathbf{V}_{\bm{I}}$ indicates indexing the rows of $\mathbf{V}$ with $\bm{I}$.
The remaining entries not in $\bm{I}$ are evicted from the cache.
This is KV cache eviction.
The goal in eviction is to thus select an appropriate $\bm{I}$ such that downstream performance is (ideally) unaffected or only slightly impacted.
We refer the reader to \Cref{app:related-work} for a discussion on related work.

Now that we understand the setting and the goal of KV (cache) eviction, we are ready to formalize and study KV eviction.
We start by showing its hardness.

\section{KV Cache Eviction is Hard}\label{sec:hardness}

First and foremost we state the problem, its input and output.

\begin{problem}[\textsc{KVEviction}]\label{problem:kveviction}~\\
    \textbf{Input.} Matrices $\mathbf{K},\mathbf{Q},\mathbf{V}\in\mathbb{Q}^{n\times d}$, error $\varepsilon\in\mathbb{Q}^d$, and compression ratio $r\in(0,1)$.

    \textbf{Output.} Whether there exists a set $\bm{I}\subset[1..n]$ such that $|\bm{I}|=\lfloor n\cdot r\rfloor$ and
    \begin{equation}\label{eq:kveviction-problem}
        \left|\softmax\left(\left[\mathbf{k}_i\cdot\mathbf{q}_n\right]_{i\in\bm{I}}\right)\cdot\mathbf{V}_{\bm{I}}-\softmax\left(\left[\mathbf{k}_i\cdot\mathbf{q}_n\right]_{i=1}^n\right)\cdot\mathbf{V}\right|\leq\varepsilon.
    \end{equation}
\end{problem}

Our reduction is from the \textsc{Partition} problem.

\begin{problem}[\textsc{Partition}]\label{problem:partition}~\\
    \textbf{Input.} A sequence of $n$ integers $\mathbf{v}=(v_1,v_2,\dots,v_n)$.

    \textbf{Output.} Whether there exists a set $\bm{I}\subset[1..n]$ such that $|\bm{I}|=\lfloor\frac{n}{2}\rfloor$ and
    \begin{equation*}
        \sum_{v\in\mathbf{v}_{\bm{I}}}v=\sum_{u\in\mathbf{v}_{\overline{\bm{I}}}}u,
    \end{equation*}
    where we here denote $\mathbf{v}_{\overline{\bm{I}}}$ to mean the values in $\mathbf{v}$ \emph{not} indexed by $\bm{I}$.
\end{problem}

\begin{restatable}{thm}{hardness}
    \textsc{KVEviction} is \textsf{NP}-complete.\label[theorem]{thm:hardness}
\end{restatable}
\begin{proof}[Proof sketch]
Without loss of generality, assume $n$ to be even; set $d=1$, $r=\sfrac{1}{2}$, $\varepsilon=0$ and $\mathbf{K}=\mathbf{Q}=\mathbf{1}_{n\times d}$, where here we denote $\mathbf{1}_{m\times n}$ as the all-ones matrix in $\mathbb{Q}^{m\times n}$.
Hardness follows immediately by simply reducing the expression in \Cref{eq:kveviction-problem}, which then yields the \textsc{Partition} problem, which is \textsf{NP}-hard.
If one can solve \textsc{KVEviction} for this case, then one can solve any instance of \textsc{Partition}, meaning \textsc{KVEviction} is at least as hard as \textsc{Partition}.
Membership in \textsf{NP} requires some care to ensure that the softmax is computable in $\mathbb{Q}$; still, this operation can be approximated according to the precision of $\varepsilon$.
The full reduction and membership proof is stated in \Cref{app:hardness}.
\end{proof}

The attentive reader might recognize that \textsc{KVEviction} is in fact an \emph{easier} problem compared to real-world KV eviction.
In KV eviction, the goal is to preserve the attention head values at \emph{decode} time---i.e.\ the \emph{future} time steps \emph{after} entries have been evicted and are no longer available to compute.
Our results show that even preserving the attention head values at \emph{eviction} time---when we have access to all the ground-truth values---is hard.

With this in mind, our goal now is to exploit the structure of \textsc{KVEviction} in order to find principled KV eviction approximations.
We do so by a probabilistic interpretation of KV eviction, showing that through this lens we obtain a richer space of eviction strategies.

\section{A Probabilistic Interpretation of KV Cache Eviction}\label{sec:prob-interp}

A natural yet crucial observation is that the softmax matrix in \Cref{eq:attention} induces $n$ distributions (one for each row of the matrix) over the $n$ entries represented by each column.
The resulting attention value for an attention head is thus a matrix containing the dot product of these probabilities with the values $\mathbf{V}$.
This amounts to computing the expectation of $\mathbf{V}$ under a distribution $p_i(\mathbf{V})$ parameterized by $\mathbf{q}_i$ and $\mathbf{K}$.
It is important to note that this distribution is not over the space of values $\mathbb{R}^{n\times d}$, but rather is a categorical distribution over the $n$ entries (i.e.\ indices) that map to $\mathbf{V}$.
\begin{equation}\label{eq:prob-interp}
    \softmax\left(\mathbf{Q}\cdot\mathbf{K}\tran\right)\cdot\mathbf{V}=\big[\overbrace{\softmax\left(\mathbf{q}_i\cdot\mathbf{K}\tran\right)}^{p_i(\mathbf{V})}\big]_{i=1}^n\cdot\mathbf{V}=\left[\sum_{j=1}^n p_i(\mathbf{V}=\mathbf{v}_j)\cdot\mathbf{v}_j\right]_{i=1}^n=\left[\mathbb{E}_{p_i}\left[\mathbf{V}\right]\right]_{i=1}^n
\end{equation}
In short, attention is nothing more than computing expectations.
When using a KV cache, this further reduces to simply computing a single expectation of a $d$-dimensional vector on the last row.
This provides an interesting framework for KV eviction:
Each entry now has a clear distribution associated to it and can---within the setting of KV eviction---be interpreted as the probability distribution of keeping that entry.
The goal of KV eviction then becomes to find a subset $\bm{I}$ that estimates future expectations with least bias and variance.

Interestingly, existing eviction methods usually base their eviction decision on computing some score, either from $p_n(\mathbf{V})$ or $\mathbf{V}$, and then taking the top-$k$ highest scored entries to keep \citep{devoto24,li24,zhang23}.
This is a zero-variance yet clearly biased estimator (and in fact unboundedly biased), skewing the expectation even at eviction time.

A simple and easy solution to achieve unbiasedness at eviction time is to select $\bm{I}$ according to samples from $p_n(\mathbf{V})$.
For a given number of samples $m$ taken with replacement from $p_n(\mathbf{V})$, $\bm{I}$ can be chosen as the collection of unique entries in this set of samples.
This provides an unbiased eviction estimator at eviction time.
In fact, one can use any proposal distribution $\pi$ in order to probabilistically evict from the cache and later adjust the expectation computation with importance sampling.

The choice of which proposal to use is open to the user as long as its support covers all entries.
In fact, existing top-$k$ score KV eviction methods are easily subsumed in this framework by interpreting these scores as an unnormalized proposal distribution $\tilde{\pi}$ (\Cref{app:subsumed-methods}).
Indeed, given a score eviction strategy, there exists an infinite number of proposal distributions whose top-$k$ also correspond to the same top-$k$ of that strategy.
Top-$k$ eviction corresponds to computing the expectation conditioned on the $k$-maximum a posteriori (MAP) states as a proxy for the eviction time expectation, while probabilistic eviction---i.e.\ sampling from $\pi$ and choosing all entries which have been sampled at least once---estimates the expectation through importance sampling.

So far we have seen that top-$k$ is biased even at eviction time, while probabilistic eviction emerges as an unbiased eviction time alternative.
However, this discussion ignores the future; the distribution at later steps will have missing probabilities which can potentially cause further error to future expectations.
Thus, a natural question to ask ourselves is:
\begin{goal}
    If we evict at time step $t$, can we correct \Cabbref{eq:attention-kv-cache-evicted} for the missing entries at time steps $>t$?
\end{goal}
As we shall see in the next sections, we can make use of the proposal distribution and its samples used in probabilistic eviction to correct decoding.

\section{Correcting for Eviction}\label{sec:correction}

Given a sequence of length $t$, the task of KV eviction is usually to perform eviction on this sequence---i.e.\ at time step $t$---and then generate new tokens at future time steps $n>t$.
We shall call time step $t$ \emph{eviction time} and any time step after that \emph{decode time}.
While at eviction time we have access to all previous distributions $\left[p_i(\mathbf{V})\right]_{i=1}^t$, at decode time $n>t$ we only have access to the kept entries of $p_n(\mathbf{V})$ (as well as cached keys $\mathbf{K}_{\bm{I}}$ and values $\mathbf{V}_{\bm{I}}$).
The usual treatment in the literature is to ignore missing entries and compute attention as usual.
However, this severely changes the distribution, as this amounts to conditioning the distribution to restricting its support to only kept entries
\begin{equation}
    \softmax\left(\left[\mathbf{k}_i\cdot\mathbf{q}_n\right]_{i\in\bm{I}_t^{n}}\right)\cdot\mathbf{V}_{\bm{I}_t^n}
    =\left[\frac{p_n(\mathbf{v}_i)\cdot\liv i\in\bm{I}_t^n\riv}{\sum_{j\in\bm{I}_t^n} p_n(\mathbf{v}_j)}\right]_{i=1}^n\cdot\mathbf{V}_{\bm{I}_t^n}
    =p_n\left(\mathbf{V}|\bm{I}_t^n\right)\cdot\mathbf{V}_{\bm{I}_t^n}
    =\mathbb{E}_{p_n}\left[\mathbf{V}|\bm{I}_t^n\right],
\end{equation}
where $\bm{I}_t^n\defeq\bm{I}\cup[t+1..n]$---i.e.\ all the kept entries and all the entries generated after eviction time---, and $\liv\cdot\riv$ denotes the Iverson bracket.
This means that the contribution of each entry will be excessively boosted or diminished by the normalizing constant proportionally to the probability mass evicted. Furthermore, because $\mathbf{K}$, $\mathbf{V}$ and $\mathbf{Q}$ are all computed as a function of the attention values of the previous layer, any error---even if small---propagates and is amplified to all following layers and future generations.

To correct for this, we utilize self-normalized importance sampling to adjust the probabilities at decode time \citep{kloek78,geweke89,owen13}.
This requires we exploit the probabilistic interpretation introduced in \Cref{sec:prob-interp}.
In summary, given a proposal distribution $\pi$ at eviction time $t$, we sample entries from $\pi$, record their counts, and evict any entries that have not been sampled.
At decode time, we then use $\pi$ and the sampled counts to correct for eviction through self-normalized importance sampling.
We now describe this process in detail and show that the corrected expectation estimate is asymptotically correct. 
To do so, let us first properly define the tools we are going to use.

\begin{definition}[Self-normalized importance sampling]
    Let $\tilde{p}(X)$ be an unnormalized probability distribution over an RV $X$ that can take $n$ values, $p$ its normalized distribution, and $\pi(X)$ a proposal distribution also over $X$. 
    Given $m$ samples from $\pi$, the function $c(x)$ maps the value $x$ to the number of times $x$ has been sampled from $\pi$.
    The self-normalized importance sampling estimator for the expectation $\mathbb{E}_p\left[X\right]$ is
    \begin{equation}
        \hat{\mu}(X;\tilde{p},\pi,m)\defeq\frac{1}{m}\sum_{i=1}^n \frac{c(x_i)\frac{\tilde{p}(x_i)}{\pi(x_i)}}{\frac{1}{m}\sum_{j=1}^n c(x_j)\frac{\tilde{p}(x_j)}{\pi(x_j)}}x_i\stackrel{m\to\infty}{=}\mathbb{E}_\pi\left[\frac{p(X)}{\pi(X)}\cdot X\right]=\mathbb{E}_p\left[X\right].
    \end{equation}
\end{definition}

Our goal is to estimate the expectation at decode time $n$.
In order to do so, we slightly massage the expectation into a more appealing expression
\begin{align}\label{eq:exp-decode-time}
    \mathbb{E}_{p_n}\left[\mathbf{V}\right]
    &=\sum_{i=1}^n p_n(\mathbf{v}_i)\cdot\mathbf{v}_i
    =\frac{\sum_{i=1}^t \tilde{p}_n(\mathbf{v}_i)\cdot\mathbf{v}_i + \sum_{i=t+1}^n \tilde{p}_n(\mathbf{v}_i)\cdot\mathbf{v}_i}{\sum_{j=1}^n \tilde{p}_n(\mathbf{v}_i)}\nonumber\\
    &=\frac{\mathbb{E}_{\tilde{p}_n}\left[\mathbf{V}_{1:t}\right] + \mathbb{E}_{\tilde{p}_n}\left[\mathbf{V}_{t+1:n}\right]}{z}
    =\mathbb{E}_{p_n^{(1:t)}}\left[\mathbf{V}\right]\cdot\frac{z_{1:t}}{z_{1:t}+z_{t+1:n}}+\mathbb{E}_{p_n^{(t+1:n)}}\left[\mathbf{V}\right]\cdot\frac{z_{t+1:n}}{z_{1:t}+z_{t+1:n}}\nonumber\\
    &=\fbox{$\mathbb{E}_{p_n^{(1:t)}}[\mathbf{V}]\cdot p(T\leq t)+\mathbb{E}_{p_n^{(t+1:n)}}[\mathbf{V}]\cdot p(T>t)$.}
\end{align}
We use the superscript to denote the distribution restricted to the support specified by the interval and then renormalized accordingly $p^{(i:k)}(X)\defeq\left[\sfrac{p(x_j)}{\sum_{l=i}^k p(x_l)}\right]_{j=i}^k$.
Similarly, $\mathbf{V}_{i:k}$ is used to denote a restriction on the values of $\mathbf{V}$ to $\{\mathbf{v}_j\}_{j=i}^k$.
Lastly, we use $z$ to denote the normalizing constant and correspondingly $z_{i:k}\defeq\sum_{j=i}^k\tilde{p}_n(\mathbf{v}_j)$ to mean the normalizing constant for that interval.

The final expression in \Cref{eq:exp-decode-time} tells us that we may decompose the expectation into two local expectations: \emph{before} and \emph{after} eviction, as long as we weigh them appropriately according to $p(T\leq t)$ and $p(T>t)$.
The local expectation after eviction $\mathbb{E}_{p_n^{(t+1:n)}}[\mathbf{V}]$ can be computed exactly: all entries are available.
However, the local expectation before eviction contains missing entries and thus needs to be approximated.
Indeed, both $p(T\leq t)$ and $p(T>t)$ also need to be approximated as their denominator is the normalizing constant $z$.
To do so, we define the following estimator.

\begin{definition}[Attention estimator]\label{def:estimator}
    The estimator for corrected attention is given by
    \begin{equation}
        \mu_n(\mathbf{V};m)\defeq \hat{\mu}(\mathbf{V};\tilde{p}_n^{(1:t)},\pi,m)\cdot p(T\leq t)+\mathbb{E}_{p_n^{(t+1:n)}}[\mathbf{V}]\cdot p(T>t),
    \end{equation}
    where $\hat{\mu}$ is a self-normalized importance sampling estimator that uses the unnormalized distribution $\tilde{p}_n^{(1:t)}$ as the target distribution and $\pi$ as the proposal.
    The probabilities $p(T\leq t)$ and $p(T>t)$ are approximated by aggregating the importance weights from $\mu_n$: $\hat{z}_{1:t}=\frac{1}{m}\sum_{j=1}^t c(\mathbf{v}_j)\frac{\tilde{p}_n(\mathbf{v}_j)}{\pi(\mathbf{v}_j)}$.
\end{definition}

This estimator is asymptotically correct w.r.t.\ the desired expectation estimand (see \Cref{app:correctness}).

\begin{restatable}{prop}{correctness}
    If $\supp(\pi)\supseteq\supp(p_n)$, then $\lim_{m\to\infty}\mu_n(\mathbf{V};m)=\mathbb{E}_{p_n}[\mathbf{V}]$.
    \label[proposition]{thm:correctness}
\end{restatable}

\begin{algorithm}[t]
    \caption{\textsc{ImportanceAttention}}\label{alg:correction}
    \begin{algorithmic}[1] 
        \Require KV cache $\mathbf{K},\mathbf{V}\in\mathbb{R}^{n\times d}$, sample counts $[c(\mathbf{v}_i)]_{i\in\bm{I}}$, proposal $\pi$, current query $\mathbf{q}_n\in\mathbb{R}^d$.
        \Ensure Corrected attention $\mathbf{o}\in\mathbb{R}^d$
        \State $\tilde{p}_n\gets\exp\left(\mathbf{q}_n\cdot\mathbf{K}\tran\right)$\Comment{Unnormalized probabilities of non-evicted entries}
        \State $w_i\gets c(\mathbf{v}_i)/(m\cdot \pi(\mathbf{v}_i))$, $\forall i\in[1..t]$ \label{alg:correction:weights}\Comment{Weights to be applied to the unnormalized probabilities}
        \State $\hat{z}_{1:t}\gets\sum_{i=1}^t w_i\cdot\tilde{p}_n(\mathbf{v}_i)$\Comment{Normalizing constant for the eviction time entries}
        \State $z_{t+1:n}\gets\sum_{i=t+1}^n \tilde{p}_n(\mathbf{v}_i)$\Comment{Normalizing constant for the decode time entries}
        \State $\hat{\mu}_{1:t}\gets\sum_{i=1}^t w_i\cdot\tilde{p}_n(\mathbf{v}_i)\cdot\mathbf{v}_i$\Comment{Estimated unnormalized expectation for eviction time}
        \State $\mu_{t+1:n}\gets\sum_{i=t+1}^n\tilde{p}_n(\mathbf{v}_i)\cdot\mathbf{v}_i$\Comment{Exact unnormalized expectation for decode time}
        \State $\mathbf{o}\gets\frac{\hat{\mu}_{1:t}+\mu_{t+1:n}}{\hat{z}_{1:t}+z_{t+1:n}}$\Comment{Corrected attention value after normalization}
        \State \textbf{return} $\mathbf{o}$
    \end{algorithmic}
\end{algorithm}

Note that although $\mu_n(\mathbf{V};m)$ is still biased for a fixed number of samples $m$, this only comes from the self-normalized importance sampling estimator, which has both bias and (mean squared) error upper bounded by $O(\sfrac{1}{m})$ \citep[Theorem 2.1]{agapiou17}.
In other words, more samples reduce this error significantly.
This bias can be further reduced by applying any bias-reducing method for self-normalized importance sampling to $\mu_n(\mathbf{V};m)$ \citep{cardoso22}.
In contrast to this bound, existing KV eviction methods deterministically keep the top-$k$ entries according to some score on either $p_t(\mathbf{V})$ or $\mathbf{V}$ and thus are zero-variance yet unboundedly biased \citep{li24,zhang23,oren24}.

Equipped with \Cref{def:estimator,thm:correctness}, we can now explicitly construct the correction algorithm for decode time.
\Cref{alg:correction} shows how to compute the corrected attention value for a given attention head.
Note that here we simplify the algorithm and notation due to space constraints, but one must be careful when computing the corrected probabilities: an entry $i$ is evicted iff $c(\mathbf{v}_i)=0$, meaning that $\tilde{p}_n(\mathbf{v}_i)$ will be undefined.
Any entry which is missing due to eviction should be properly ignored in the computation.

\textbf{Ultimately, the pipeline is as follows:} at eviction time $t$, (1) construct a proposal distribution $\pi^{(h)}$ for each attention head $h$, and (2) sample counts $[c^{(h)}(\mathbf{v}_i)]_{i=1}^t$. At decode time $n$, (3) compute attention following \Cref{alg:correction}.

Some attention has to be paid when dealing with grouped query attention \citep{ainslie23}: Because every attention head in a group $\bm{G}=\{h_1,h_2,\dots,h_g\}$ shares the same $\mathbf{K}$ and $\mathbf{V}$ caches, ingroup eviction and correction have to be consistent.
To do so, we use one proposal $\pi^{(\bm{G})}$ per group, ensuring that the support of $\pi^{(\bm{G})}$ contains the support of every attention head distribution by constructing a mixture of each head proposal $\pi^{(\bm{G})}(\mathbf{v}_i)\defeq\sum_{j=1}^g\alpha_j\cdot\pi^{(h_j)}(\mathbf{v}_i)$, where every $\alpha_j\in\Delta^{g-1}$ comes from the $(g-1)$-simplex.

So far, our contributions have been either theoretical or philosophical.
In the next section, we empirically validate probabilistic eviction and correction, studying the role of bias and variance in KV eviction.

\section{Bias vs Variance}\label{sec:bias-vs-variance}

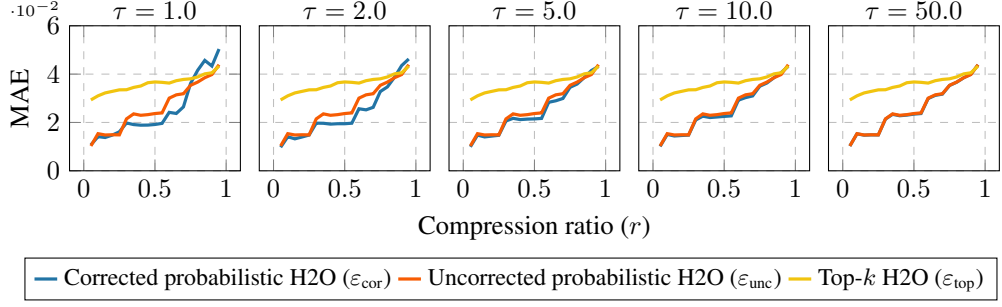
\begin{figure}[t]
\centering%
\begin{tikzpicture}
    \begin{groupplot}[
        group style={group size=5 by 1,horizontal sep=0.25cm,},
        height=3.5cm,
        width=0.275\textwidth,
        grid style=dashed,
        cycle list name=cpalette,
        xmajorgrids=true, ymajorgrids=true,
        every legend image post/.append style={scale=0.5},
        ymin=0.00, ymax=0.06,
        xmin=-0.1, xmax=1.1,
        title style={yshift=-0.25cm,font={}},
        every y tick scale label/.style={
            at={(0,1)},
            anchor=south east,
            font={\tiny},
        },
    ]

    \nextgroupplot[ylabel={MAE},title={$\tau=1.0$}]
        \addplot+[very thick] table[x=r,y=corrected]   {\maelocaltadata};
        \addplot+[very thick] table[x=r,y=uncorrected] {\maelocaltadata};
        \addplot+[very thick] table[x=r,y=h2o]         {\maelocaltadata};
        \coordinate (c1) at (rel axis cs:-0.275,1);
    \nextgroupplot[ymajorticks=false,title={$\tau=2.0$}]
        \addplot+[very thick] table[x=r,y=corrected]   {\maelocaltbdata};
        \addplot+[very thick] table[x=r,y=uncorrected] {\maelocaltbdata};
        \addplot+[very thick] table[x=r,y=h2o]         {\maelocaltbdata};
    \nextgroupplot[ymajorticks=false,title={$\tau=5.0$},xlabel={Compression ratio ($r$)}]
        \addplot+[very thick] table[x=r,y=corrected]   {\maelocaltcdata};
        \addplot+[very thick] table[x=r,y=uncorrected] {\maelocaltcdata};
        \addplot+[very thick] table[x=r,y=h2o]         {\maelocaltcdata};
    \nextgroupplot[ymajorticks=false,title={$\tau=10.0$}]
        \addplot+[very thick] table[x=r,y=corrected]   {\maelocaltddata};
        \addplot+[very thick] table[x=r,y=uncorrected] {\maelocaltddata};
        \addplot+[very thick] table[x=r,y=h2o]         {\maelocaltddata};
    \nextgroupplot[ymajorticks=false,title={$\tau=50.0$},legend cell align=left,legend style={font={\small},fill=none,draw=black,anchor=center,align=center},legend to name=leg,legend columns=3,]
        \addplot+[very thick] table[x=r,y=corrected]   {\maelocaltedata};
        \addplot+[very thick] table[x=r,y=uncorrected] {\maelocaltedata};
        \addplot+[very thick] table[x=r,y=h2o]         {\maelocaltedata};
        \legend{Corrected probabilistic H2O ($\varepsilon_\text{cor}$), Uncorrected probabilistic H2O ($\varepsilon_\text{unc}$),Top-$k$ H2O ($\varepsilon_\text{top}$)}
        \coordinate (c2) at (rel axis cs:1,1);
    \end{groupplot}
    \begin{scope}
        \coordinate (p) at ($(c1)!0.5!(c2)$);
        \node[below] at (p |- current bounding box.south) {\pgfplotslegendfromname{leg}};
    \end{scope}
\end{tikzpicture}
\caption{\textbf{KV eviction is a game of trading bias for variance and vice versa.} The plot shows mean absolute error ($y$-axis) of the attention values for the first token after an evicted prompt on \llama{} averaged across all heads (see \Cref{app:bias-vs-variance-exp} for details) for different compression ratios ($x$-axis). Bias and variance are low at lower $r$, since $m$ (number of samples) is high. As $r$ increases, $m$ decreases, causing variance to increase more than bias. Higher importance weight temperatures $\tau$ decrease the variance, causing the error to decrease at extremely high $r$'s.}\label{fig:bias-vs-variance}
\end{figure}

Although bias and variance share the same asymptotic upper bounds in self-normalized importance sampling, in practice they may present distinct behaviors \citep{agapiou17}.
In this section, we show that by decreasing the variance at the cost of more bias, we are able to achieve lower error under certain conditions.
To do so, we introduce a temperature scaling parameter $\tau\in[0,\infty)$ to the importance weights.
We then replace \Cref{alg:correction:weights} in \Cref{alg:correction} with
\begin{equation}\label{eq:is-temp-scaling}
    w_i\gets \exp\left(\frac{1}{\tau}\cdot\log\left(\frac{c(\mathbf{v}_i)}{m\cdot\pi(\mathbf{v}_i)}\right)\right),\forall i\in[1..t].
\end{equation}
When $\tau=1$, \Cref{eq:is-temp-scaling} reduces to the same expression as \Cref{alg:correction:weights}; when $\tau\to\infty$, it is equivalent to not applying correction since $w_i$ approaches $1$. This reduces variance, since the stochasticity comes from the samples of the proposal distribution, which can be amplified by the correction. By suppressing correction, variance is reduced.
Temperatures below $1$ cause \Cref{alg:correction} to give more importance (i.e.\ probability mass) to entries in $[1..t]$ compared to decode time entries.

Our goal is to show the impact of bias and variance in terms of the mean absolute error (MAE) between the estimated expectation and the ground-truth expectation.
Given a prompt $\mathbf{x}$ and a first response token $y$, we compute the KV cache $(\mathbf{K},\mathbf{V})$ on $\mathbf{x}$ and apply probabilistic eviction using a proposal $\pi$ constructed from the scores of the H2O eviction method \citep{zhang23} (see \Cref{app:subsumed-methods} for details).
This results in an evicted KV cache $(\mathbf{K}_{\bm{I}},\mathbf{V}_{\bm{I}})$.
We then compute the average error between the uncorrected $\mathbf{o}_{\bm{I}}$ and ground-truth $\mathbf{o}$ attention values $\varepsilon_\text{unc}$, and between the corrected $\mu_n(\mathbf{V};m)$ and ground-truth attention values $\varepsilon_\text{cor}$ across all $h$ attention heads
\begin{equation}
    \varepsilon_{\text{unc}}\defeq\frac{1}{h}\cdot\sum_{i=1}^h\left|\mathbf{o}^{(i)}-\mathbf{o}_{\bm{I}}^{(i)}\right|,\quad\varepsilon_{\text{cor}}\defeq\frac{1}{h}\cdot\sum_{i=1}^h\left|\mathbf{o}^{(i)}-\mu_n^{(i)}(\mathbf{V};m)\right|.
\end{equation}
Finally, we compute the error $\varepsilon_\text{top}$ between the ground-truth and H2O's estimation of the attention.

In order to provide a fair comparison of $\varepsilon_\text{cor}$ and $\varepsilon_\text{unc}$ against $\varepsilon_\text{top}$, we constrain the compression ratio to be the same for every attention head.
To do so, for each attention head, we select a number of samples that achieves, on average, the desired compression ratio.
In \Cref{app:ccp}, we discuss how this problem relates to the classical coupon collector problem and provide a tractable approximation with negligible error in practice.

\Cref{fig:bias-vs-variance} shows the average errors across compression ratios $r\in[0,1]$ and temperatures $\tau\in\{1,2,5,10,50\}$ for \llama{}\texttt{\color{palette-red}\textbf{.2-3B}} \citep{llama}; prompt $\mathbf{x}$ and first token $y$ are described in \Cref{app:bias-vs-variance-exp}.
Lower compression ratio $r$ implies a higher number of samples $m$.
When $m$ is high, bias and variance are low; when $m$ is extremely low, variance is very high, causing $\varepsilon_\text{cor}$ to be much higher than $\varepsilon_\text{unc}$.
This is where temperature scaling provides an improvement by reducing the variance at a cost to bias.
As the middle ranges of $r$ show, increasing bias when variance is already low harms estimation quality.
The $\varepsilon_\text{unc}$ curve does not match with $\varepsilon_\text{top}$ because eviction is done stochastically, while H2O greedily evicts the top-$k$.
This shows that even probabilistic eviction by itself allows for less error.
As $r$ approaches one, the probabilistic choices all collapse to the top-$k$, which is when $\varepsilon_\text{unc}$ and $\varepsilon_\text{top}$ meet.

In this section, we showed how correction allows for a trade-off between bias and variance.
However, in practice bias plays a much larger role in attention error.
This motivates the shift from low-variance high-bias to higher-variance low-bias estimators.
As we shall see in the next section, corrected probabilistic eviction tends to produce more robust estimators compared to existing zero-variance unbounded bias KV eviction strategies due to reduced bias.

\section{Robustness}\label{sec:robustness}

In this section, we show that zero-variance estimators suffer from a heuristic misalignment, where the fixed bias introduced by these heuristics can cause catastrophic failures in certain tasks.
To do so, we compare probabilistic eviction against four popular existing KV eviction methods: StreamingLLM \citep{xiao24}, SnapKV \citep{li24}, TOVA \citep{oren24} and H2O \citep{zhang23}.
All of these methods are implemented in KVPress, which we use for evaluation and our own implementation \citep{kvpress}.
We evaluate \llama{}\texttt{\color{palette-red}\textbf{.2-3B}} \citep{llama} and \qwen{}\texttt{\color{palette-dgreen}\textbf{-4B}} \citep{qwen} on two standard KV eviction benchmark datasets: LongBench \citep{longbench} and RULER \citep{ruler}.
Due to compute constraints, we evaluate RULER on a subset of 130 examples and on all instances of the HotpotQA, QASPER and TriviaQA splits of LongBench whose context prompt is shorter than \num{3000} tokens.

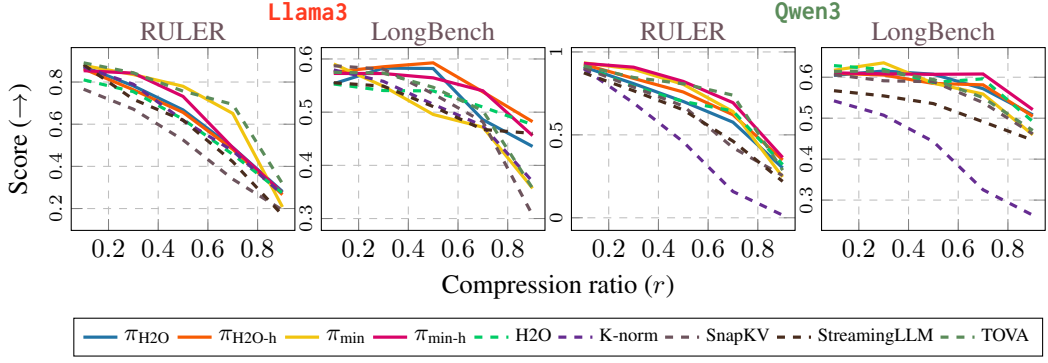
\begin{figure}[t]
\centering%
\begin{tikzpicture}
    \begin{groupplot}[
        group style={group size=4 by 1, horizontal sep=0.35cm},
        height=4cm,
        width=0.325\textwidth,
        grid style=dashed,
        cycle list name=cpalette,
        xmajorgrids=true, ymajorgrids=true,
        every legend image post/.append style={scale=0.75},
        xmin=0.05, xmax=0.95,
        title style={yshift=-0.35cm,font={\strut{}\color{palette-gray}}},
        y tick label style={font=\scriptsize, rotate=90},
        yticklabel={\pgfmathparse{\tick*0.01}\pgfmathprintnumber{\pgfmathresult}},
    ]

    \nextgroupplot[ylabel={Score ($\to$)}, title={RULER}]
        \addplot+[very thick] table[x=target_cr,y=overall]{\srleah};
        \addplot+[very thick] table[x=target_cr,y=overall]{\srleahn};
        \addplot+[very thick] table[x=target_cr,y=overall]{\srleamv};
        \addplot+[very thick] table[x=target_cr,y=overall]{\srleamvfh};
        \addplot+[dashed,very thick] table[x=target_cr,y=overall]{\srlho};
        \addplot+[dashed,very thick] table[x=target_cr,y=overall]{\srlkn};
        \addplot+[dashed,very thick] table[x=target_cr,y=overall]{\srlsnap};
        \addplot+[dashed,very thick] table[x=target_cr,y=overall]{\srlsllm};
        \addplot+[dashed,very thick] table[x=target_cr,y=overall]{\srltova};
        \coordinate (sc1) at (rel axis cs:0,1);
        \coordinate (src1) at (rel axis cs:0,1);

    \nextgroupplot[title={LongBench}]
        \addplot+[very thick] table[x=target_cr,y=avg_score]{\slbleah};
        \addplot+[very thick] table[x=target_cr,y=avg_score]{\slbleahn};
        \addplot+[very thick] table[x=target_cr,y=avg_score]{\slbleamv};
        \addplot+[very thick] table[x=target_cr,y=avg_score]{\slbleamvfh};
        \addplot+[dashed,very thick] table[x=target_cr,y=avg_score]{\slblho};
        \addplot+[dashed,very thick] table[x=target_cr,y=avg_score]{\slblkn};
        \addplot+[dashed,very thick] table[x=target_cr,y=avg_score]{\slblsnap};
        \addplot+[dashed,very thick] table[x=target_cr,y=avg_score]{\slblsllm};
        \addplot+[dashed,very thick] table[x=target_cr,y=avg_score]{\slbltova};
        \coordinate (sc2) at (rel axis cs:1,0);

    \nextgroupplot[title={RULER}]
        \addplot+[very thick] table[x=target_cr,y=overall]{\srqeah};
        \addplot+[very thick] table[x=target_cr,y=overall]{\srqeahn};
        \addplot+[very thick] table[x=target_cr,y=overall]{\srqeamv};
        \addplot+[very thick] table[x=target_cr,y=overall]{\srqeamvfh};
        \addplot+[dashed,very thick] table[x=target_cr,y=overall]{\srqho};
        \addplot+[dashed,very thick] table[x=target_cr,y=overall]{\srqkn};
        \addplot+[dashed,very thick] table[x=target_cr,y=overall]{\srqsnap};
        \addplot+[dashed,very thick] table[x=target_cr,y=overall]{\srqsllm};
        \addplot+[dashed,very thick] table[x=target_cr,y=overall]{\srqtova};
        \coordinate (sc3) at (rel axis cs:0.0,0);

    \nextgroupplot[title={LongBench},
                   legend cell align=left,
                   legend style={font={\footnotesize},fill=none,draw=black,anchor=center,align=left},
                   legend to name=scoreleg,
                   legend columns=9]
        \addplot+[very thick] table[x=target_cr,y=avg_score]{\slbqeah};
        \addplot+[very thick] table[x=target_cr,y=avg_score]{\slbqeahn};
        \addplot+[very thick] table[x=target_cr,y=avg_score]{\slbqeamv};
        \addplot+[very thick] table[x=target_cr,y=avg_score]{\slbqeamvfh};
        \addplot+[dashed,very thick] table[x=target_cr,y=avg_score]{\slbqho};
        \addplot+[dashed,very thick] table[x=target_cr,y=avg_score]{\slbqkn};
        \addplot+[dashed,very thick] table[x=target_cr,y=avg_score]{\slbqsnap};
        \addplot+[dashed,very thick] table[x=target_cr,y=avg_score]{\slbqsllm};
        \addplot+[dashed,very thick] table[x=target_cr,y=avg_score]{\slbqtova};
        \legend{$\pihto$, $\pihtoh$, $\pimin$, $\piminh$, {\scriptsize{}H2O}, {\scriptsize{}K-norm}, {\scriptsize{}SnapKV}, {\scriptsize{}StreamingLLM}, {\scriptsize{}TOVA}}
        \coordinate (sc4) at (rel axis cs:1,1);
        \coordinate (src4) at (rel axis cs:1,0);
    \end{groupplot}
    \begin{scope}
        \coordinate (p) at ($(sc1)!0.5!(sc4)$);
        \node[below] (xlab) at (p |- current bounding box.south) {Compression ratio ($r$)};
        \node[below] (leg) at (xlab.south) {\pgfplotslegendfromname{robleg}};
        \coordinate (qa) at ($(src1)!0.5!(sc2)$);
        \node[below,anchor=south] at ($(qa |- current bounding box.north) + (0, -0.25)$) {\llama{}};
        \coordinate (qb) at ($(sc3)!0.5!(src4)$);
        \node[below,anchor=south] at ($(qb |- current bounding box.north) + (0, -0.525)$) {\qwen{}};
    \end{scope}
\end{tikzpicture}
\vspace{-0.5cm}
\caption{\textbf{Probabilistic eviction achieves competitive performance.} Probabilistic eviction with correction performs best at lower to middle ranges, when there are enough samples to achieve a lower variance and bias. Solid lines show our approach, dashed lines show existing eviction methods.}\label{fig:robustness-scores}
\end{figure}

To evaluate probabilistic eviction, we propose four different proposal distributions: $\pimin$, $\piminh$, $\pihto$ and $\pihtoh$ which we formally define in \Cref{app:proposals} and provide proper probabilistic semantics.
Intuitively, $\pimin$ computes the minimum variance proposal distribution for the last query at eviction time $t$, $\pihto$ computes a probabilistic version of H2O where scores are normalized and used as the proposal distribution, $\piminh$ computes the minimum-variance proposal for a distribution where all previous time steps are marginalized and weighted by a harmonic prior $h_j$, and $\pihtoh$ computes H2O with a harmonic prior $h_j$.

In these experiments, we constrain the compression ratio at the global level: given a compression ratio $r$, we choose a number of samples $m$ such that, on average, the aggregated compression ratio of all KV caches is $r$.
This is achieved by a similar mechanism to the one for locally constraining each attention head to a specific compression ratio.
\Cref{app:ccp-global} shows how this can be done and \Cref{app:ccp-global-error} provides empirical evidence that the error between the achieved compression ratio and target compression ratio is, in practice, insignificant.
Although the overall compression ratio is fixed, every attention head will adaptively choose a suitable compression ratio depending on its distribution.
This allows for attention heads that require a lower compression ratio to automatically consume the budget of heads that do not.
Interestingly, this results in compression ratios that follow a pattern of lower compression at lower layers, and higher at higher layers (see \Cref{app:compression-pattern,fig:pyramid}).
This is in agreement with the findings of \citet{cai2025pyramidkv}, which suggest that allocating more budget to lower layers and less to higher is preferable.

\begin{figure}[t]
\centering%
\begin{tikzpicture}
    \begin{groupplot}[
        group style={group size=4 by 1, horizontal sep=0.35cm},
        height=4cm,
        width=0.325\textwidth,
        grid style=dashed,
        cycle list name=cpalette,
        xmajorgrids=true, ymajorgrids=true,
        every legend image post/.append style={scale=0.75},
        xmin=0.05, xmax=0.95,
        title style={font=\small},
        y tick label style={font=\scriptsize, rotate=90},
        title style={yshift=-0.35cm,font={\strut{}\color{palette-gray}}},
    ]

    \nextgroupplot[ylabel={Win score ($\to$)}, title={RULER}]
        \addplot+[very thick] table[x=target_cr,y=ranking]{\rleah};
        \addplot+[very thick] table[x=target_cr,y=ranking]{\rleahn};
        \addplot+[very thick] table[x=target_cr,y=ranking]{\rleamv};
        \addplot+[very thick] table[x=target_cr,y=ranking]{\rleamvfh};
        \addplot+[dashed,very thick] table[x=target_cr,y=ranking]{\rlho};
        \addplot+[dashed,very thick] table[x=target_cr,y=ranking]{\rlkn};
        \addplot+[dashed,very thick] table[x=target_cr,y=ranking]{\rlsnap};
        \addplot+[dashed,very thick] table[x=target_cr,y=ranking]{\rlsllm};
        \addplot+[dashed,very thick] table[x=target_cr,y=ranking]{\rltova};
        \coordinate (c1) at (rel axis cs:0,1);
        \coordinate (rc1) at (rel axis cs:0,1);

    \nextgroupplot[title={LongBench}]
        \addplot+[very thick] table[x=target_cr,y=ranking]{\lbleah};
        \addplot+[very thick] table[x=target_cr,y=ranking]{\lbleahn};
        \addplot+[very thick] table[x=target_cr,y=ranking]{\lbleamv};
        \addplot+[very thick] table[x=target_cr,y=ranking]{\lbleamvfh};
        \addplot+[dashed,very thick] table[x=target_cr,y=ranking]{\lblho};
        \addplot+[dashed,very thick] table[x=target_cr,y=ranking]{\lblkn};
        \addplot+[dashed,very thick] table[x=target_cr,y=ranking]{\lblsnap};
        \addplot+[dashed,very thick] table[x=target_cr,y=ranking]{\lblsllm};
        \addplot+[dashed,very thick] table[x=target_cr,y=ranking]{\lbltova};
        \coordinate (c2) at (rel axis cs:1,0);

    \nextgroupplot[title={RULER}]
        \addplot+[very thick] table[x=target_cr,y=ranking]{\rqeah};
        \addplot+[very thick] table[x=target_cr,y=ranking]{\rqeahn};
        \addplot+[very thick] table[x=target_cr,y=ranking]{\rqeamv};
        \addplot+[very thick] table[x=target_cr,y=ranking]{\rqeamvfh};
        \addplot+[dashed,very thick] table[x=target_cr,y=ranking]{\rqho};
        \addplot+[dashed,very thick] table[x=target_cr,y=ranking]{\rqkn};
        \addplot+[dashed,very thick] table[x=target_cr,y=ranking]{\rqsnap};
        \addplot+[dashed,very thick] table[x=target_cr,y=ranking]{\rqsllm};
        \addplot+[dashed,very thick] table[x=target_cr,y=ranking]{\rqtova};
        \coordinate (c3) at (rel axis cs:0.0,0);

    \nextgroupplot[title={LongBench},
                   legend cell align=left,
                   legend style={font=\small,fill=none,draw=black,anchor=center,align=left},
                   legend to name=robleg,
                   legend columns=9]
        \addplot+[very thick] table[x=target_cr,y=ranking]{\lbqeah};
        \addplot+[very thick] table[x=target_cr,y=ranking]{\lbqeahn};
        \addplot+[very thick] table[x=target_cr,y=ranking]{\lbqeamv};
        \addplot+[very thick] table[x=target_cr,y=ranking]{\lbqeamvfh};
        \addplot+[dashed,very thick] table[x=target_cr,y=ranking]{\lbqho};
        \addplot+[dashed,very thick] table[x=target_cr,y=ranking]{\lbqkn};
        \addplot+[dashed,very thick] table[x=target_cr,y=ranking]{\lbqsnap};
        \addplot+[dashed,very thick] table[x=target_cr,y=ranking]{\lbqsllm};
        \addplot+[dashed,very thick] table[x=target_cr,y=ranking]{\lbqtova};
        \legend{$\pihto$, $\pihtoh$, $\pimin$, $\piminh$, {\scriptsize{}H2O}, {\scriptsize{}K-norm}, {\scriptsize{}SnapKV}, {\scriptsize{}StreamingLLM}, {\scriptsize{}TOVA}}
        \coordinate (c4) at (rel axis cs:1,1);
        \coordinate (rc4) at (rel axis cs:1,0);
    \end{groupplot}
    \begin{scope}
        \coordinate (p) at ($(c1)!0.5!(c4)$);
        \node[below] (xlab) at (p |- current bounding box.south) {Compression ratio ($r$)};
        \node[below] (leg) at (xlab.south) {\pgfplotslegendfromname{robleg}};
        \coordinate (qa) at ($(rc1)!0.5!(c2)$);
        \node[below,anchor=south] at ($(qa |- current bounding box.north) + (0, -0.25)$) {\llama{}};
        \coordinate (qb) at ($(c3)!0.5!(rc4)$);
        \node[below,anchor=south] at ($(qb |- current bounding box.north) + (0, -0.525)$) {\qwen{}};
    \end{scope}
\end{tikzpicture}
\vspace{-0.5cm}
\caption{\textbf{Probabilistic eviction can make eviction more robust for different tasks.} Win score values for four proposals ($\pihto,\pihtoh,\pimin,\piminh$) and five top-$k$ estimators (H2O, K-norm, SnapKV, StreamingLLM, TOVA). $\piminh$ achieves best ranking across datasets and models, showing more robustness when minimizing variance and lowering bias.}\label{fig:robustness-ranking}
\end{figure}
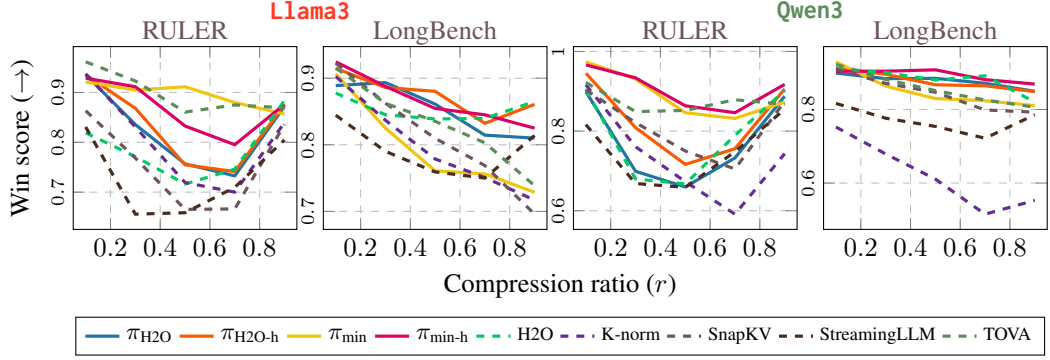

\Cref{fig:robustness-scores} shows the average total score in LongBench and RULER for the four previously mentioned proposals and five top-$k$ existing eviction methods.
This is computed by averaging the score at each split, weighing them by the fraction of examples in that split and then adding all of these weighted scores together.
The average score of each split is given by a scorer function suited to that split $f_{\text{split}}:\mathcal{L}^\ast\to[0,1]$ that takes in the text generated by the (KV evicted) LLM and attributes a number in $[0,1]$.
For example, HotpotQA uses F1 as the scorer function, while the NIAH (needle-in-a-haystack) split of RULER uses exact string matching.
Results show that probabilistic eviction and correction achieve better performance or are at least as competitive as the state-of-the-art.

Because these average total scores aggregate over many tasks, they can hide the behavior of eviction methods for certain tasks.
For example, on \llama{} StreamingLLM ranks first on the QuestionAnswer split, yet dead last in MultiKey-NIAH (see \Cref{fig:individual-llama}).
Similarly, on \qwen{} K-norm is consistently in the top-3 on the MultiKey-NIAH split, yet last in CommonWords (see \Cref{fig:individual-qwen}).
To account for these catastrophic failures, we measure the pairwise win scores of each method.
Specifically, we assign a rank from 0 to $n - 1$ to each method based on how well it performs compared to the other $n - 1$ methods. The method with the lowest score gets a rank of 0, while the method with the highest score gets a rank of $n - 1$. If multiple methods achieve the same score, they are all given the rank corresponding to a win against the other tied methods. Ranking is repeated for all the examples in the task, after which the ranks are normalized by the number of examples in the task to ensure all tasks contribute equally. We refer to the normalized average rank across tasks as the \textit{win score}.

\Cref{fig:robustness-ranking} shows the win score for each proposal and eviction method.
Notably, $\piminh$ consistently ranks at the top, showing that it is most robust compared to other methods.
Probabilistic eviction tends to do better at low to middle ranges, as there are enough samples to decrease the bias and variance.
However, even at higher compression it performs competitively against other eviction methods.
Note that when the compression ratio is extremely high (i.e.\ close to 1.0), all methods do equally poorly, as most of the prompt has been evicted.
This is shown as most, if not all, curves increasing in \Cref{fig:robustness-ranking}, as ties are also rewarded in this scoring.

\section{Conclusion}

The goal of this paper was to formalize KV cache eviction and in doing so find principled KV eviction estimators.
As a result, we showed that KV eviction is \textsf{NP}-complete, and revealed that through a probabilistic interpretation KV eviction is nothing more than an expectation estimation problem.
In our quest to characterize this, we found that existing KV eviction methods are zero-variance yet unbounded bias estimators and that a good KV eviction strategy should reduce bias and variance in order to best approximate future expectations.
To this end, we proposed a self-normalized importance sampling estimator with reasonable upper bounds on both bias and variance.
We then empirically showed that this class of estimators achieves competitive performance against state-of-the-art eviction methods yet are more robust at different tasks.
This paper opens up an exciting and promising avenue to developing new KV eviction estimators.
We hope that framing KV eviction as an estimation problem and connecting the field of statistical methods with KV cache eviction generates fruitful new research not only within KV eviction but also across other KV cache compression techniques.

\bibliography{refs}

\appendix
\crefalias{section}{appendix}

\section{Related Work}\label{app:related-work}

KV eviction has received significant interest for the purpose of speeding up inference in large language models.
Following the taxonomy given by \citet{chen2025pitfalls}, we categorize KV eviction methods as position-based, attention-based, embedding-based, and hybrid.
The most notable examples of each respectively are StreamingLLM \citep{xiao24}, TOVA \citep{oren24}, K-norm \citep{devoto24}, and SnapKV \citep{li24}.
All these methods can be framed as top-$k$ score eviction methods with zero variance and unbounded bias.
The problem of correcting for the distortion caused by eviction has not been, as far as we know, properly addressed, although it has been recognized in the literature: \citet{zhang23} acknowledge the issue in passing, while in contemporary work,  \citet{zweiger26} identify this problem but only as a motivation for KV cache compaction, a more compute-intensive strategy for KV cache compression.
The probabilistic interpretation of attention as expectation has previously appeared in the literature as a Bayesian perspective to attention \citep{singh23,bianchessi26}, although with no connection to KV eviction.

\section{Hardness}\label{app:hardness}

\hardness*%
\begin{proof}
    We first show hardness.
    We shall reduce from the \textsc{Partition} problem.
    Assume $n$ is even (without loss of generality); set $d=1$, $r=\frac{1}{2}$, $\varepsilon=0$ and $\mathbf{K}=\mathbf{Q}=\mathbf{1}_{n\times d}$, where here we denote $\mathbf{1}_{m\times n}$ as the all-ones matrix in $\mathbb{Q}^{m\times n}$, and similarly $\mathbf{1}_n$ as the all-one vector in $\mathbb{Q}^n$.
    Thus, the problem boils down to deciding whether
    \begin{equation*}
        \text{softmax}\left(\textbf{1}_{|\bm{I}|}\right)\cdot\mathbf{V}_{\bm{I}}=\text{softmax}\left(\textbf{1}_n\right)\cdot\mathbf{V}.
    \end{equation*}
    Note that the softmax terms on the left and on the right are equivalent to uniform distributions over $|\bm{I}|$ and $n$ elements respectively.
    Thus, we may rewrite the above equation (after expanding the matrix multiplications) as
    \begin{equation*}
        \frac{1}{\lfloor\frac{n}{2}\rfloor}\cdot\sum_{v\in\mathbf{V}_{\bm{I}}}v=\frac{1}{n}\cdot\sum_{u\in\mathbf{V}}u.
    \end{equation*}
    From the previous equation, it follows immediately that
    \begin{equation*}
        \sum_{v\in\mathbf{V}_{\bm{I}}}v=\sum_{u\in\mathbf{V}_{\overline{\bm{I}}}}u.
    \end{equation*}
    We have thus shown that any instance of \textsc{KVEviction} is at least as hard as any instance of \textsc{Partition}.
    Since \textsc{Partition} is \textsf{NP}-hard, \textsc{KVEviction} is also \textsf{NP}-hard.

    We now show membership in $\textsf{NP}$.
    The main issue is computing the softmax function in a Turing machine.
    We invoke \citet{brent10}, where the authors show how to compute exponents and division up to arbitrary precision in time polynomial in the number of bits.
    In our case, the number of bits needed is directly specified by $\varepsilon$, meaning that one only needs to compute the softmax up to the precision of parameter $\varepsilon$.
    Thus, it suffices to define the verifier to be \Cref{eq:kveviction-problem}, which can be done in time polynomial in the precision of $\varepsilon$, $n$ and $d$.
    Therefore \textsc{KVEviction} is in \textsf{NP} and so is \textsf{NP}-complete.
\end{proof}

\section{Correctness}\label{app:correctness}

\correctness*%
\begin{proof}
    We want to show that
    \begin{equation}
        \lim_{m\to\infty}\hat{\mu}\left(\mathbf{V};\tilde{p}_n,\pi,m\right)\cdot\frac{\hat{z}_{1:t}}{\hat{z}_{1:t}+z_{t+1:n}}+\mathbb{E}_{p_{n}^ {(t+1:n)}}\left[\mathbf{V}\right]\cdot\frac{z_{t+1:n}}{\hat{z}_{1:t}+z_{t+1:n}}=\mathbb{E}_{p_n}\left[\mathbf{V}\right].
    \end{equation}
    The first thing to note is that both $\mathbb{E}_{p_{n}^{(t+1:n)}}\left[\mathbf{V}\right]$ and $z_{t+1:n}$ can be computed exactly as we have access to $\left[\tilde{p}_{n}(\mathbf{v}_i)\right]_{i=t+1}^{n}$.
    Thus, we need to show that both (i) $\hat{\mu}\left(\mathbf{V};\tilde{p}_{n},\pi,m\right)$ and (ii) $\hat{z}_{1:t}$ correctly compute the desired estimands in expectation.
    
    We first show (i)
    \begin{align}
        \lim_{m\to\infty}\hat{\mu}\left(\mathbf{V};\tilde{p}_{n},\pi,m\right)
        &=\lim_{m\to\infty}\frac{\frac{1}{m}\cdot\sum_{i=1}^m\frac{\tilde{p}_{n}(\mathbf{v}^{(i)})}{\pi(\mathbf{v}^{(i)})}\cdot\mathbf{v}^{(i)}}{\frac{1}{m}\cdot\sum_{i=1}^m\frac{\tilde{p}_{n}(\mathbf{v}^{(i)})}{\pi(\mathbf{v}^{(i)})}}
        =\lim_{m\to\infty}\frac{\frac{1}{m}\cdot\sum_{i=1}^t c(\mathbf{v}_i)\cdot\frac{\tilde{p}_{n}(\mathbf{v}_i)}{\pi(\mathbf{v}_i)}\cdot\mathbf{v}_i}{\frac{1}{m}\cdot\sum_{i=1}^t c(\mathbf{v}_i)\cdot\frac{\tilde{p}_{n}(\mathbf{v}_i)}{\pi(\mathbf{v}_i)}}\label{eq:correctness-first-first}\\
        &=\frac{\sum_{i=1}^t \pi(\mathbf{v}_i)\cdot\frac{\tilde{p}_{n}(\mathbf{v}_i)}{\pi(\mathbf{v}_i)}\cdot\mathbf{v}_i}{\sum_{i=1}^t \pi(\mathbf{v}_i)\cdot\frac{\tilde{p}_{n}(\mathbf{v}_i)}{\pi(\mathbf{v}_i)}}
        =\frac{\mathbb{E}_{p_{n}^{(1:t)}}\left[\mathbf{V}\right]\cdot z}{z}=\mathbb{E}_{p_{n}^ {(1:t)}}\left[\mathbf{V}\right],\label{eq:correctness-first-second}
    \end{align}
    where the superscript $\mathbf{v}^{(i)}$ indicates the $i$-th sampled entry from $\pi$ and the step from (\ref{eq:correctness-first-first}) to (\ref{eq:correctness-first-second}) is due to our assumption of $\supp(\pi)\supseteq\supp(p_n)$. 
    
    The same applies to (ii)
    \begin{equation}
        \lim_{m\to\infty}\hat{z}_{1:t}=\lim_{m\to\infty}\frac{1}{m}\cdot\sum_{i=1}^m\frac{\tilde{p}_{n}(\mathbf{v}^{(i)})}{\pi(\mathbf{v}^{(i)})}=\lim_{m\to\infty}\frac{1}{m}\cdot\sum_{i=1}^tc(\mathbf{v}_i)\cdot\frac{\tilde{p}_{n}(\mathbf{v}_i)}{\pi(\mathbf{v}_i)}=\sum_{i=1}^t\pi(\mathbf{v}_i)\cdot\frac{\tilde{p}_{n}(\mathbf{v}_i)}{\pi(\mathbf{v}_i)}=z_{1:t}.
    \end{equation}
    Now that we have shown convergence of these estimators, equality follows directly
    \begin{equation}\label{eq:exp-equality}
        \lim_{m\to\infty}\mu_n(\mathbf{V};m)=\mathbb{E}_{p_{n}^{(1:t)}}\left[\mathbf{V}\right]\cdot\frac{z_{1:t}}{z}+\mathbb{E}_{p_{n}^{(t+1:n)}}\left[\mathbf{V}\right]\cdot\frac{z_{t+1:n}}{z}
        =\frac{\mathbb{E}_{\tilde{p}_{n}^{(1:t)}}\left[\mathbf{V}\right]+\mathbb{E}_{\tilde{p}_{n}^{(t+1:n)}}\left[\mathbf{V}\right]}{z}=\mathbb{E}_{p_{n}}\left[\mathbf{V}\right].
    \end{equation}
\end{proof}

\section{Score-based Eviction Methods as Probabilistic Eviction Policies}\label{app:subsumed-methods}

Any eviction method based on scores can easily be subsumed by probabilistic eviction.
In fact, there exists an infinite number of proposal distributions that share the same top-$k$ as the score strategy.
In this section we show a few examples and how they can be turned into proposal distributions.

\textbf{StreamingLLM.} StreamingLLM selects the first four entries and a rolling window of the last $k$ KV cache entries to keep, and evicts the rest.
Any proposal distributions that have the first four and last $k$ entries as modes are obviously equivalent in terms of their top-$k$ modes; for example, a mixture of any monotonically increasing distribution with a distribution whose support consists of the first four entries.

\textbf{H2O.} H2O scores are based on the probabilities $p(\mathbf{V})$ and computed as
\begin{equation}
    \text{score}_{\text{H2O}}(\mathbf{v}_i)\defeq\sum_{j=1}^n p_j(\mathbf{v}_i).
\end{equation}
Intuitively, each column in the softmax matrix is summed out.
A proposal distribution can be achieved by simply normalizing these scores
\begin{equation}
    \pi_{\text{H2O}}(\mathbf{v}_i)=\frac{\text{score}_\text{H2O}(\mathbf{v}_i)}{\sum_{j=1}^n\text{score}_{\text{H2O}}(\mathbf{v}_j)}.
\end{equation}

\textbf{K-norm.} The K-norm scores are computed from the negative $L^2$ norm of the $\mathbf{V}$ matrix.
This gives an $n$-dimensional vector, with a score for each entry
\begin{equation*}
    \text{score}_\text{knorm}(\mathbf{v}_i)={\lVert\mathbf{v}_i\rVert}_2=\sqrt{\sum_{j=1}^d \mathbf{v}_{ij}^2}.
\end{equation*}
Because $\mathbf{V}$ may contain negative numbers, we apply a softmax instead of the usual normalization.
Thus, the proposal distribution for K-norm is defined as
\begin{equation}
    \pi_{\text{knorm}}(\mathbf{v}_i)=\frac{\exp\left(\text{score}_{\text{knorm}}(\mathbf{v}_i)\right)}{\sum_{j=1}^t \exp\left(\text{score}_{\text{knorm}}(\mathbf{v}_j)\right)}.
\end{equation}

\textbf{TOVA.} Scores are given by the average $p_t(\mathbf{V})$ (i.e.\ the last query's distribution) across all heads in a layer.
Let $\bm{h}_\mathcal{L}=\{h_1,h_2,\dots,h_l\}$ be the set of all attention heads in a given layer $\mathcal{L}$.
The score for each attention head $h_x\in\bm{h}_\mathcal{L}$ is given by
\begin{equation}
    \text{score}_\text{TOVA}\left(\mathbf{v}_i^{(h_x)}\right)\defeq \frac{1}{l}\cdot\sum_{i=1}^l p_t(\mathbf{v}_i).
\end{equation}
The proposal distribution for TOVA is
\begin{equation}
    \pi_{\text{TOVA}}(\mathbf{v}_i)=\frac{\text{score}_{\text{TOVA}}(\mathbf{v}_i)}{\sum_{j=1}^n\text{score}_\text{TOVA}(\mathbf{v}_j)}.
\end{equation}

\textbf{SnapKV.} The process of obtaining the SnapKV scores is more involved and requires computing an average pooling of $p(\mathbf{V})$ for a given kernel and window size.
In short, SnapKV computes a column average of $p(\mathbf{V})$ similar to H2O, applies a 1-dimensional pooling on this tensor, and then averages this again at the same dimension.
We define the proposal distribution for SnapKV to be the normalized scores across entries, similar to the previous proposals.

\subsection{Adjusting the Proposal Distribution}

Any of these proposals can be equipped with other mechanisms for adjusting these probabilities.
As an example, it is possible to apply a temperature scaling factor $\tau$ on these proposal distributions
\begin{equation*}
    \pi^\tau(\mathbf{v}_i)=\frac{\exp\left(\frac{1}{\tau}\cdot\text{score}(\mathbf{v}_i)\right)}{\sum_{j=1}^t\exp\left(\frac{1}{\tau}\cdot\text{score}(\mathbf{v}_j)\right)}
\end{equation*}
in order to decrease variance, or mix the proposal with a uniform in order to decrease bias
\begin{equation*}
    \pi'(\mathbf{v}_i)=\alpha\cdot\pi(\mathbf{v}_i)+(1-\alpha)\cdot\frac{1}{t}\text{, where }\alpha\in[0,1].
\end{equation*}

\section{Prompt and Question for MAE Error}\label{app:bias-vs-variance-exp}

The prompt $\mathbf{x}$ consists of the tokenization of the following text:
\begin{quote}
The Roman Empire was one of the largest and most enduring political entities in ancient
history, reaching its greatest extent under Emperor Trajan in 117 CE when it spanned from
Britain to Mesopotamia. Rome evolved from a monarchy to a republic before Julius Caesar's
assassination in 44 BCE precipitated civil wars. Augustus emerged victorious and became the
first emperor in 27 BCE, initiating a period of relative peace known as the Pax Romana.
Roman law, architecture, and administrative practices shaped European civilization for
centuries. Latin, the language of the Romans, evolved into the Romance languages including
French, Spanish, Italian, Portuguese, and Romanian. The western empire collapsed in 476 CE,
while the eastern half survived as the Byzantine Empire for nearly a thousand more years.
\end{quote}
While the first token $y$ is given by the first element in the tokenized sequence of the following string:
\begin{quote}
Who was the first Roman emperor?
\end{quote}

\section{On the Relationship Between Number of Samples and Compression Ratio}\label{app:ccp}

It is often useful, at a practical level, to control how much of the KV cache is evicted.
For example, a practitioner might want to preserve as much generation quality given their GPU memory budget.
This materializes as a so-called \emph{compression ratio} $r$ that describes the percentage of compression of a KV cache: when $r=0$, no entries are evicted, $r=0.5$ indicates that half of the KV cache is evicted, and $r=1$ means all of the entries are evicted.
The probabilistic eviction procedure described in \Cref{sec:prob-interp} makes no such distinction: given a number of samples $m$, it automatically sets a compression ratio based on the number of unique entries it has sampled.
This problem is probabilistic in nature: in order to set an $r$, one would require finding an $m$ s.t.\ a compression ratio of $r$ is achieved on average.

\begin{problem}[\textsc{ExpectedSamples}]~\\\label{problem:number-of-samples}%
    \textbf{Input.} Sequence length $n$, proposal $\pi$, target compression ratio $r^* \in [0, 1)$, error tolerance $\varepsilon \geq 0$.
    
    \textbf{Output.} An $m\in\mathbb{N}$ such that
    \begin{equation}
        |\mathbb{E}_{\pi}[J(m)] - j^*| \leq \varepsilon,
    \end{equation}
    where $j^*  = (1 - r^*) \cdot n$, and $J(m)$ is a random variable denoting the number of unique values sampled after drawing $m$ samples with replacement from $\pi$.
\end{problem}

This can be reduced from the classical \citeauthor{erdos61} coupon collector problem \citep{erdos61}.

\begin{problem}[Coupon Collector's Problem---\textsc{CCP} \citep{erdos61}]~\\\label{problem:coupon-collector}%
    \textbf{Input.} Set of coupons $\bm{a}=\{a_1,a_2,\dots,a_n\}$, each with probability $p(A=a_i)$ of being issued, and a target number of unique coupons $j^*\in\mathbb{N}$.
    
    \textbf{Output.} The expectation of the number of coupons that need to be drawn from $\bm{a}$ with replacement until at least $j^*$ unique coupons are sampled. 
\end{problem}

In the KV eviction setting, the $n$ sequence positions correspond to coupons, sampling from $\pi$ corresponds to drawing coupons with probabilities $p(A=a_i)$, the number of samples $m$ corresponds to the number of draws, and $j^*$ is the number of distinct sampled positions/values.
As these probabilities are discrete, the solution \citep{FLAJOLET1992207} is

\begin{equation}\label{eq:coupon-collector-discrete}
    \mathbb{E}[m]
    = \sum_{q=0}^{j^*-1} \left( (-1)^{j^*-1-q}
    \binom{n-q-1}{n-j^*}
    \sum_{\lvert \bm{s} \rvert = q} \frac{1}{1 - \rho_{\bm{s}}}\right),
\end{equation}
where the inner sum is over all subsets $\bm{s} \subseteq \{1,\dots,n\}$ with $|\bm{s}| = q$, and $\rho_{\bm{s}} = \sum_{i \in \bm{s}} {p(A=a_i)}$.

\begin{complexity}
    Evaluating \cref{eq:coupon-collector-discrete} naively requires exponential time. 
\end{complexity}

\begin{proof}
    Let $\alpha = (-1)^{j^*-1-q}\binom{n-q-1}{n-j^*}$ and $\beta = \sum_{|\bm{s}|=q}{\frac{1}{1 - \rho_{\bm{s}}}}$.
    \cref{eq:coupon-collector-discrete} can now be written as
    \begin{equation}
        \mathbb{E}[m]
        = \sum_{q=0}^{j^*-1}\alpha\beta.
    \end{equation}

    The amortized complexity of evaluating $\alpha$ is $O(1)$ if we use a lookup table. For $\beta$, there are $\binom{n}{q}$ subsets $\bm{s}$ such that $|\bm{s}| = q$, and each subset sums over $q$ probabilities. As $n$ increases and $q$ varies, evaluating $\beta$ scales exponentially. Therefore, evaluating \cref{eq:coupon-collector-discrete} naïvely requires \textit{at least} exponential time.

\end{proof}

Although \cref{eq:coupon-collector-discrete} gives us the exact solution we need, it is computationally infeasible to compute it in the naïve way.
In fact, as far as we know, this problem has not been shown to either be in \textsf{P} \emph{or} \textsf{NP}-hard.
Our conjecture is that \Cref{problem:coupon-collector} is in fact intractable, and thus cannot be computed exactly in polynomial time.

To solve \Cref{problem:number-of-samples} efficiently, we consider a different approach. Intuitively, we do not need to know the \textbf{exact} number of samples $m \in \mathbb{R}_{\ge0}$ to draw $j^*$ unique values. For the purposes of \textsc{KVEvict}, $m$ is a non-negative integer, which reduces the search space massively, especially given a suitable upper bound. Therefore, we can utilize a trial and error approach optimized by binary search. Instead of obtaining $m$ \textbf{directly} given $n$, $\pi$, and $r^*$, we repeatedly \textbf{guess} what $m$ should be and (quickly) verify if/how our guess $\hat{m}$ needs to be updated for the next guess by comparing the expected number of unique values sampled $\mathbb{E}_{\pi}[J(\hat{m})]$ to the target $j^*$.

\begin{algorithm}[h]
    \begin{algorithmic}[1]
        \Require Sequence length $n$, proposal distribution $\pi(\mathbf{v}) = [\pi(\mathbf{v}_1), \pi(\mathbf{v}_2), ..., \pi(\mathbf{v}_n)]$, target compression ratio $r^* \in [0, 1]$, and error tolerance $\varepsilon = 0.5$.
        \Ensure A non-negative integer $m$ such that $|\mathbb{E}_{\pi}[J(m)] - j^*| \le \varepsilon$, where $j^*  = (1 - r^*) \cdot n$, and $J(m)$ is the random variable denoting the number of unique values sampled after drawing $m$ samples with replacement from $\pi$.
        \State $j^* \gets (1 - r^*) \cdot n$ \Comment{Target number of unique values sampled}
        \State $\alpha \gets 10^{12}$ \Comment{Large constant; existence of a valid $\alpha$ is guaranteed}
        \State $\text{\texttt{low}} \gets 0, \text{\texttt{high}} \gets \alpha$ \Comment{by Part 2 of \Cref{thm:number-of-samples-correct}}
        \While{$\text{\texttt{low}} \le \text{\texttt{high}}$} \Comment{Binary search}
            \State $\hat{m} \gets \lfloor(\text{\texttt{low}} + \text{\texttt{high}}) \div 2\rfloor$
            \State $\mathbb{E}_{\pi}[J(\hat{m})] \gets \sum_{i=1}^{n}{\left(1 - (1- \pi(v_i))^{\hat{m}}\right)}$ \Comment{Expected number of unique values sampled}
            \If{$|\mathbb{E}_{\pi}[J(\hat{m})] - j^*| \le \varepsilon$} \Comment{Target $j^*$ reached}
                \State $m \gets \hat{m}$
                \State \textbf{break}
            \ElsIf{$\mathbb{E}_{\pi}[J(\hat{m})] < j^*$}
                \State $\text{\texttt{low}} \gets \hat{m} + 1$
            \Else
                \State $\text{\texttt{high}} \gets \hat{m} - 1$
            \EndIf
        \EndWhile
        \State \textbf{return} $m$
    \end{algorithmic}
    \caption{\textsc{Expected Samples}}\label{alg:number-of-samples}
\end{algorithm}

\begin{restatable}{thm}{samplescorrectness}\label[theorem]{thm:number-of-samples-correct}
    \Cref{alg:number-of-samples} solves \Cref{problem:number-of-samples}  (\textsc{Expected Samples}) with error $|\mathbb{E}_{\pi}[J(m)] - j^*| \le \varepsilon$, where $\varepsilon = 0.5$; assuming full support on $\pi$.
\end{restatable}
\begin{proof}~\\
    \begin{enumerate}[leftmargin=0.75cm,topsep=0pt,itemsep=-0.125ex]
        \item{
        We first show that line 6 in \cref{alg:number-of-samples} gives us the expected number of unique values sampled \citep{1950447}
    
        \begin{equation}\label{eq:expected-unique-values}
            \mathbb{E}_{\pi}[J(m)] = \sum_{i=1}^{n}{\left(1 - (1- \pi(v_i))^{m}\right).}
        \end{equation}
    
        Let $J_i$ be the indicator random variable for position $i$: $J_i = 1$ if $v_i$ is sampled, $0$ otherwise. For the $i$-th position,
        \begin{equation}
            \mathbb{E}_\pi[J_i(m)] = p(J_i = 1) = 1 - (1 - \pi(v_i))^m.
        \end{equation}

        We want the total expected number of unique values sampled $J(m) = \sum_{i=1}^n J_i$. By linearity of expectation,
        \begin{equation}
            \mathbb{E}_\pi[J(m)] = \mathbb{E}_\pi[J_1(m) + J_2(m) + \cdots + J_n(m)] = \sum_{i=1}^{n}\mathbb{E}_\pi[J_i(m)] = \sum_{i=1}^{n}{\left(1 - (1- \pi(v_i))^{m}\right).}
        \end{equation}
        }
        
        \item{
        We show that i) $\mathbb{E}_{\pi}[J(m)]$ is non-decreasing in $m$ and ii) that there exists a finite upper bound $\alpha$ such that $\mathbb{E}_{\pi}[J(\alpha)] \ge j^*$. 

        i) Each term $(1-(1-\pi(v_i))^m)$ is non-decreasing in $m$ as $\pi(v_i) \in (0,1]$ implies $(1-\pi(v_i))^m \in [0,1)$, making $(1-\pi(v_i))^m$ non-increasing. Hence, $\mathbb{E}_{\pi}[J(m)]$ is non-decreasing in $m$.

        ii) Assuming $\pi$ has full support ($\pi(v_i) > 0$ for all $i$), since each $(1-\pi(v_i)) < 1$, we have $(1-\pi(v_i))^m \to 0$ as $m \to \infty$, so $\mathbb{E}_{\pi}[J(m)] \to n \ge j^*$. Hence, a finite $\alpha$ with $\mathbb{E}_{\pi}[J(\alpha)] \ge j^*$ always exists. In practice, we simply set $\alpha = 10^{12}$ rather than computing $\alpha$ from $\pi$, which suffices in all realistic settings.
        }
    
        \item{Finally, we show $|\mathbb{E}_{\pi}[J(m)] - j^*| \le \varepsilon$, where $\varepsilon = 0.5$. The difference between consecutive number of samples $m \in \mathbb{Z}$ is given by

        \begin{equation}
            \mathbb{E}_{\pi}[J(m+1)] - \mathbb{E}_{\pi}[J(m)]
            = \sum_{i=1}^{n}{\left(1 - (1- \pi(v_i))^{m+1}\right) - \sum_{i=1}^{n}{\left(1 - (1- \pi(v_i))^{m}\right)}}
        \end{equation}

        \begin{equation}
            = \sum_{i=1}^{n}{\left((1 - (1- \pi(v_i))^{m+1}) - (1 - (1- \pi(v_i))^{m})\right)}
        \end{equation}
        }

        \begin{equation}
            = \sum_{i=1}^{n}\left({(1- \pi(v_i))^{m}-(1- \pi(v_i))^{m+1}}\right)
        \end{equation}

        \begin{equation}
            =\sum_{i=1}^{n}(1 - \pi(v_i))^{m}(1 - (1 - \pi(v_i)))
            = \sum_{i=1}^{n}(1 - \pi(v_i))^{m}\pi(v_i).
        \end{equation}

        Because $0 < \pi(v_i) \le 1$ $\forall$ $i=1...n$,  $(1- \pi(v_i))^{m} \le 1$. Therefore, 

        \begin{equation}
            \mathbb{E}_{\pi}[J(m+1)] - \mathbb{E}_{\pi}[J(m)] = \sum_{i=1}^{n}(1 - \pi(v_i))^{m}\pi(v_i) \le \sum_{i=1}^{n}\pi(v_i) = 1.
        \end{equation}

        We now prove by strong induction that for every $j^* \in [0,\, \mathbb{E}_{\pi}[J(m)]]$, there exists $m' \le m$ with $|\mathbb{E}_{\pi}[J(m')] - j^*| \le 0.5$.

        \textbf{Base case} ($m = 0$). $\mathbb{E}_{\pi}[J(0)] = 0$, so the only $j^*$ in $[0, 0]$ is $j^* = 0$, and $|\mathbb{E}_{\pi}[J(0)] - 0| = 0 \le 0.5$.

        \textbf{Inductive step}. Assume that the claim holds for all $m' < m$. Let $j^* \in [0,\, \mathbb{E}_{\pi}[J(m)]]$.
        \begin{itemize}
            \item If $j^* \le \mathbb{E}_{\pi}[J(m-1)]$, there must be some $m' \le m-1$ with $|\mathbb{E}_{\pi}[J(m')] - j^*| \le 0.5$ following the assumption.
            \item If $j^* \in \bigl(\mathbb{E}_{\pi}[J(m-1)],\, \mathbb{E}_{\pi}[J(m)]\bigr]$,  then
            \begin{equation}
                \underbrace{\left(\mathbb{E}_{\pi}[J(m)] - j^*\right)}_{\ge\, 0}
                +\underbrace{\left(j^* - \mathbb{E}_{\pi}[J(m-1)]\right)}_{\ge\, 0}
                = \mathbb{E}_{\pi}[J(m)] - \mathbb{E}_{\pi}[J(m-1)] \le 1,
            \end{equation}
            so at least one term is $\le 0.5$, meaning either $m$ or $m-1$ satisfies the bound.
        \end{itemize}
        \textbf{Conclusion}. By strong induction, for every $j^* \in [0,\, \mathbb{E}_{\pi}[J(m)]]$, there exists $m' \le m$ with $|\mathbb{E}_{\pi}[J(m')] - j^*| \le 0.5$.
    
    \end{enumerate}
\end{proof}

\begin{remark}
    \Cref{thm:number-of-samples-correct} assumes $\pi$ has full support, which holds in practice as softmax assigns some probability to every token. However, if attention is highly peaked, there can be many $\mathbf{v}_i$ where $\pi(\mathbf{v}_i) \approx 0$, causing these tokens to be rarely sampled. As such, achieving low compression ratios (large $j^*$) may require an impractically large $m$. While this is theoretically an issue, empirically it does not arise at reasonable compression ratios ($r^* \ge 0.1$), where tokens with near-zero attention weight are inconsequential anyway.
\end{remark}

\begin{complexity}
    \cref{alg:number-of-samples} is $O(n\log(\alpha))$.
\end{complexity}
\begin{proof}
    Evaluating $\mathbb{E}_{\pi}[J(\hat{m})]$ in \cref{eq:expected-unique-values} requires $O(n)$ time as we sum over $i=1...n$. The binary search is bounded by $\alpha$, resulting in $O(log(\alpha))$ complexity. For each trial in binary search, we evaluate $\mathbb{E}_{\pi}[J(\hat{m})]$, so the total complexity is $O(nlog(\alpha))$.
\end{proof}

\subsection{Globally soft constraining the compression ratio}\label{app:ccp-global}

We now consider the case of constraining the compression ratio at a global level.
Before, we set the number of samples $m$ to be such that the compression ratio of an attention head would be on average a target compression ratio $r^\ast$.
To do this, we needed to be able to compute the expectation $\mathbb{E}_\pi[J(m)]$.
This can be easily done as we have access to $\pi(v_i)$, as we saw in the previous subsection.
However, this is not optimal.
By constraining each attention head, we are effectively forcing each attention head to select a possibly unnatural choice of compression ratio; for example, consider the case where the sequence length $n=4$, $\pi=(0.8, 0.0998, 0.0001, 0.0001)$ and $m=100$: the expected compression ratio for this $m$ is
\begin{equation*}
    1 - \frac{\mathbb{E}_\pi[J(m)]}{n}\approx 1 - \frac{2}{4}=0.5.
\end{equation*}
But if we set a target compression ratio $r^\ast=0.75$, then we require $m\approx7000$ samples in order to achieve this without unbiasing the attention head output (on average).
A similar argument can be made to lowering the compression ratio: setting $r^\ast=0.25$ requires us to go as low as $m=2$ samples, which can introduce a lot of variance.

Thus, an attractive alternative is to (soft) constraint the compression ratio at a global level: instead of per-head, we set a fixed number of samples for all attention heads such that the average compression ratio across all heads\footnote{This is reasonable since all attention heads have the same size.} is (on average) a target compression ratio $r^\ast$.
Doing so is easy; we are interested in the expression $\mathbb{E}\left[\frac{1}{h}\cdot\sum_{i=1}^h J^{(i)}(m)\,\middle|\,\bm{x}\right]$, where $J^{(i)}(m)$ is the number of unique sampled indices for attention head $i$ after $m$ draws and $\bm{x}$ is the prompt.
By linearity of expectation,
\begin{equation}\label{eq:exp-kept-indices-global}
    \mathbb{E}\left[\frac{1}{h}\cdot\sum_{i=1}^h J^{(i)}(m)\,\middle|\,\bm{x}\right]=\frac{1}{h}\cdot\sum_{i=1}^h\mathbb{E}\left[J^{(i)}(m)\,\middle|\,\bm{x}\right]=\frac{1}{h}\cdot\sum_{i=1}^h \mathbb{E}_{\pi^{(i)}}\left[J^{(i)}(m)\,\middle|\,\bm{x}\right].
\end{equation}
Note that the last equality is exactly the same expression we previously computed (we omitted the conditioning on the prompt earlier).
Thus, \Cref{alg:number-of-samples} can be run at the global level in order to select an adequate number of samples for a target compression ratio by only replacing $\mathbb{E}_\pi[J(m)\mid\bm{x}]$ with \Cref{eq:exp-kept-indices-global}.

\begin{figure}[t]
\centering%
\begin{tikzpicture}
    \begin{groupplot}[
        group style={group size=3 by 3, horizontal sep=0.4cm, vertical sep=0.4cm},
        height=4cm,
        width=0.39\textwidth,
        grid style=dashed,
        cycle list name=cpalette,
        xmajorgrids=true, ymajorgrids=true,
        every legend image post/.append style={scale=0.75},
        xmin=0.05, xmax=0.95,
        title style={yshift=-0.35cm,font={\strut{}\color{palette-gray}}},
        yticklabel={\pgfmathparse{\tick*0.01}\pgfmathprintnumber{\pgfmathresult}},
        y tick label style={font=\scriptsize, rotate=90},
    ]
        \nextgroupplot[title={HotpotQA},xmajorticks=false,ylabel={Score ($\to$)}]
            \addplot+[very thick] table[x=target_cr,y=avg_score,col sep=comma] {plots/longbench/llama3_3b/normal/per_dataset/hotpotqa_e/estimated_attention__h2o_norm_scores.csv};
            \addplot+[very thick] table[x=target_cr,y=avg_score,col sep=comma] {plots/longbench/llama3_3b/normal/per_dataset/hotpotqa_e/estimated_attention__h2o_scores.csv};
            \addplot+[very thick] table[x=target_cr,y=avg_score,col sep=comma] {plots/longbench/llama3_3b/normal/per_dataset/hotpotqa_e/estimated_attention__min_var_full_harmonic_scores.csv};
            \addplot+[very thick] table[x=target_cr,y=avg_score,col sep=comma] {plots/longbench/llama3_3b/normal/per_dataset/qasper_e/estimated_attention__min_var_scores.csv};
            \addplot+[dashed,very thick] table[x=target_cr,y=avg_score,col sep=comma] {plots/longbench/llama3_3b/normal/per_dataset/qasper_e/h2o_scores.csv};
            \addplot+[dashed,very thick] table[x=target_cr,y=avg_score,col sep=comma] {plots/longbench/llama3_3b/normal/per_dataset/qasper_e/knorm_scores.csv};
            \addplot+[dashed,very thick] table[x=target_cr,y=avg_score,col sep=comma] {plots/longbench/llama3_3b/normal/per_dataset/qasper_e/snapkv_scores.csv};
            \addplot+[dashed,very thick] table[x=target_cr,y=avg_score,col sep=comma] {plots/longbench/llama3_3b/normal/per_dataset/qasper_e/streaming_llm_scores.csv};
            \addplot+[dashed,very thick] table[x=target_cr,y=avg_score,col sep=comma] {plots/longbench/llama3_3b/normal/per_dataset/qasper_e/tova_scores.csv};
        \nextgroupplot[title={QASPER},xmajorticks=false]
            \addplot+[very thick] table[x=target_cr,y=avg_score,col sep=comma] {plots/longbench/llama3_3b/normal/per_dataset/qasper_e/estimated_attention__h2o_norm_scores.csv};
            \addplot+[very thick] table[x=target_cr,y=avg_score,col sep=comma] {plots/longbench/llama3_3b/normal/per_dataset/qasper_e/estimated_attention__h2o_scores.csv};
            \addplot+[very thick] table[x=target_cr,y=avg_score,col sep=comma] {plots/longbench/llama3_3b/normal/per_dataset/qasper_e/estimated_attention__min_var_full_harmonic_scores.csv};
            \addplot+[very thick] table[x=target_cr,y=avg_score,col sep=comma] {plots/longbench/llama3_3b/normal/per_dataset/qasper_e/estimated_attention__min_var_scores.csv};
            \addplot+[dashed,very thick] table[x=target_cr,y=avg_score,col sep=comma] {plots/longbench/llama3_3b/normal/per_dataset/qasper_e/h2o_scores.csv};
            \addplot+[dashed,very thick] table[x=target_cr,y=avg_score,col sep=comma] {plots/longbench/llama3_3b/normal/per_dataset/qasper_e/knorm_scores.csv};
            \addplot+[dashed,very thick] table[x=target_cr,y=avg_score,col sep=comma] {plots/longbench/llama3_3b/normal/per_dataset/qasper_e/snapkv_scores.csv};
            \addplot+[dashed,very thick] table[x=target_cr,y=avg_score,col sep=comma] {plots/longbench/llama3_3b/normal/per_dataset/qasper_e/streaming_llm_scores.csv};
            \addplot+[dashed,very thick] table[x=target_cr,y=avg_score,col sep=comma] {plots/longbench/llama3_3b/normal/per_dataset/qasper_e/tova_scores.csv};
        \nextgroupplot[title={TriviaQA},xmajorticks=false]
            \addplot+[very thick] table[x=target_cr,y=avg_score,col sep=comma] {plots/longbench/llama3_3b/normal/per_dataset/triviaqa_e/estimated_attention__h2o_norm_scores.csv};
            \addplot+[very thick] table[x=target_cr,y=avg_score,col sep=comma] {plots/longbench/llama3_3b/normal/per_dataset/triviaqa_e/estimated_attention__h2o_scores.csv};
            \addplot+[very thick] table[x=target_cr,y=avg_score,col sep=comma] {plots/longbench/llama3_3b/normal/per_dataset/triviaqa_e/estimated_attention__min_var_full_harmonic_scores.csv};
            \addplot+[very thick] table[x=target_cr,y=avg_score,col sep=comma] {plots/longbench/llama3_3b/normal/per_dataset/triviaqa_e/estimated_attention__min_var_scores.csv};
            \addplot+[dashed,very thick] table[x=target_cr,y=avg_score,col sep=comma] {plots/longbench/llama3_3b/normal/per_dataset/triviaqa_e/h2o_scores.csv};
            \addplot+[dashed,very thick] table[x=target_cr,y=avg_score,col sep=comma] {plots/longbench/llama3_3b/normal/per_dataset/triviaqa_e/knorm_scores.csv};
            \addplot+[dashed,very thick] table[x=target_cr,y=avg_score,col sep=comma] {plots/longbench/llama3_3b/normal/per_dataset/triviaqa_e/snapkv_scores.csv};
            \addplot+[dashed,very thick] table[x=target_cr,y=avg_score,col sep=comma] {plots/longbench/llama3_3b/normal/per_dataset/triviaqa_e/streaming_llm_scores.csv};
            \addplot+[dashed,very thick] table[x=target_cr,y=avg_score,col sep=comma] {plots/longbench/llama3_3b/normal/per_dataset/triviaqa_e/tova_scores.csv};
        \nextgroupplot[title={FrequentWords},xmajorticks=false,ylabel={Score ($\to)$}]
            \addplot+[very thick] table[x=target_cr,y=avg_score,col sep=comma] {plots/ruler/llama3_3b/normal/per_task/fwe/estimated_attention__h2o_norm_scores.csv};
            \addplot+[very thick] table[x=target_cr,y=avg_score,col sep=comma] {plots/ruler/llama3_3b/normal/per_task/fwe/estimated_attention__h2o_scores.csv};
            \addplot+[very thick] table[x=target_cr,y=avg_score,col sep=comma] {plots/ruler/llama3_3b/normal/per_task/fwe/estimated_attention__min_var_full_harmonic_scores.csv};
            \addplot+[very thick] table[x=target_cr,y=avg_score,col sep=comma] {plots/ruler/llama3_3b/normal/per_task/fwe/estimated_attention__min_var_scores.csv};
            \addplot+[dashed,very thick] table[x=target_cr,y=avg_score,col sep=comma] {plots/ruler/llama3_3b/normal/per_task/fwe/h2o_scores.csv};
            \addplot+[dashed,very thick] table[x=target_cr,y=avg_score,col sep=comma] {plots/ruler/llama3_3b/normal/per_task/fwe/knorm_scores.csv};
            \addplot+[dashed,very thick] table[x=target_cr,y=avg_score,col sep=comma] {plots/ruler/llama3_3b/normal/per_task/fwe/snapkv_scores.csv};
            \addplot+[dashed,very thick] table[x=target_cr,y=avg_score,col sep=comma] {plots/ruler/llama3_3b/normal/per_task/fwe/streaming_llm_scores.csv};
            \addplot+[dashed,very thick] table[x=target_cr,y=avg_score,col sep=comma] {plots/ruler/llama3_3b/normal/per_task/fwe/tova_scores.csv};
        \nextgroupplot[title={MultiKey-NIAH},xmajorticks=false]
            \addplot+[very thick] table[x=target_cr,y=avg_score,col sep=comma] {plots/ruler/llama3_3b/normal/per_task/niah_multikey_1/estimated_attention__h2o_norm_scores.csv};
            \addplot+[very thick] table[x=target_cr,y=avg_score,col sep=comma] {plots/ruler/llama3_3b/normal/per_task/niah_multikey_1/estimated_attention__h2o_scores.csv};
            \addplot+[very thick] table[x=target_cr,y=avg_score,col sep=comma] {plots/ruler/llama3_3b/normal/per_task/niah_multikey_1/estimated_attention__min_var_full_harmonic_scores.csv};
            \addplot+[very thick] table[x=target_cr,y=avg_score,col sep=comma] {plots/ruler/llama3_3b/normal/per_task/niah_multikey_1/estimated_attention__min_var_scores.csv};
            \addplot+[dashed,very thick] table[x=target_cr,y=avg_score,col sep=comma] {plots/ruler/llama3_3b/normal/per_task/niah_multikey_1/h2o_scores.csv};
            \addplot+[dashed,very thick] table[x=target_cr,y=avg_score,col sep=comma] {plots/ruler/llama3_3b/normal/per_task/niah_multikey_1/knorm_scores.csv};
            \addplot+[dashed,very thick] table[x=target_cr,y=avg_score,col sep=comma] {plots/ruler/llama3_3b/normal/per_task/niah_multikey_1/snapkv_scores.csv};
            \addplot+[dashed,very thick] table[x=target_cr,y=avg_score,col sep=comma] {plots/ruler/llama3_3b/normal/per_task/niah_multikey_1/streaming_llm_scores.csv};
            \addplot+[dashed,very thick] table[x=target_cr,y=avg_score,col sep=comma] {plots/ruler/llama3_3b/normal/per_task/niah_multikey_1/tova_scores.csv};
        \nextgroupplot[title={QuestionAnswer},xmajorticks=false]
            \addplot+[very thick] table[x=target_cr,y=avg_score,col sep=comma] {plots/ruler/llama3_3b/normal/per_task/qa_1/estimated_attention__h2o_norm_scores.csv};
            \addplot+[very thick] table[x=target_cr,y=avg_score,col sep=comma] {plots/ruler/llama3_3b/normal/per_task/qa_1/estimated_attention__h2o_scores.csv};
            \addplot+[very thick] table[x=target_cr,y=avg_score,col sep=comma] {plots/ruler/llama3_3b/normal/per_task/qa_1/estimated_attention__min_var_full_harmonic_scores.csv};
            \addplot+[very thick] table[x=target_cr,y=avg_score,col sep=comma] {plots/ruler/llama3_3b/normal/per_task/qa_1/estimated_attention__min_var_scores.csv};
            \addplot+[dashed,very thick] table[x=target_cr,y=avg_score,col sep=comma] {plots/ruler/llama3_3b/normal/per_task/qa_1/h2o_scores.csv};
            \addplot+[dashed,very thick] table[x=target_cr,y=avg_score,col sep=comma] {plots/ruler/llama3_3b/normal/per_task/qa_1/knorm_scores.csv};
            \addplot+[dashed,very thick] table[x=target_cr,y=avg_score,col sep=comma] {plots/ruler/llama3_3b/normal/per_task/qa_1/snapkv_scores.csv};
            \addplot+[dashed,very thick] table[x=target_cr,y=avg_score,col sep=comma] {plots/ruler/llama3_3b/normal/per_task/qa_1/streaming_llm_scores.csv};
            \addplot+[dashed,very thick] table[x=target_cr,y=avg_score,col sep=comma] {plots/ruler/llama3_3b/normal/per_task/qa_1/tova_scores.csv};
        \nextgroupplot[title={CommonWords},ylabel={Score ($\to$)}]
            \addplot+[very thick] table[x=target_cr,y=avg_score,col sep=comma] {plots/ruler/llama3_3b/normal/per_task/cwe/estimated_attention__h2o_norm_scores.csv};
            \addplot+[very thick] table[x=target_cr,y=avg_score,col sep=comma] {plots/ruler/llama3_3b/normal/per_task/cwe/estimated_attention__h2o_scores.csv};
            \addplot+[very thick] table[x=target_cr,y=avg_score,col sep=comma] {plots/ruler/llama3_3b/normal/per_task/cwe/estimated_attention__min_var_full_harmonic_scores.csv};
            \addplot+[very thick] table[x=target_cr,y=avg_score,col sep=comma] {plots/ruler/llama3_3b/normal/per_task/cwe/estimated_attention__min_var_scores.csv};
            \addplot+[dashed,very thick] table[x=target_cr,y=avg_score,col sep=comma] {plots/ruler/llama3_3b/normal/per_task/cwe/h2o_scores.csv};
            \addplot+[dashed,very thick] table[x=target_cr,y=avg_score,col sep=comma] {plots/ruler/llama3_3b/normal/per_task/cwe/knorm_scores.csv};
            \addplot+[dashed,very thick] table[x=target_cr,y=avg_score,col sep=comma] {plots/ruler/llama3_3b/normal/per_task/cwe/snapkv_scores.csv};
            \addplot+[dashed,very thick] table[x=target_cr,y=avg_score,col sep=comma] {plots/ruler/llama3_3b/normal/per_task/cwe/streaming_llm_scores.csv};
            \addplot+[dashed,very thick] table[x=target_cr,y=avg_score,col sep=comma] {plots/ruler/llama3_3b/normal/per_task/cwe/tova_scores.csv};
            \coordinate (c1) at (rel axis cs:-0.275,1);
        \nextgroupplot[title={VariableTracking}]
            \addplot+[very thick] table[x=target_cr,y=avg_score,col sep=comma] {plots/ruler/llama3_3b/normal/per_task/vt/estimated_attention__h2o_norm_scores.csv};
            \addplot+[very thick] table[x=target_cr,y=avg_score,col sep=comma] {plots/ruler/llama3_3b/normal/per_task/vt/estimated_attention__h2o_scores.csv};
            \addplot+[very thick] table[x=target_cr,y=avg_score,col sep=comma] {plots/ruler/llama3_3b/normal/per_task/vt/estimated_attention__min_var_full_harmonic_scores.csv};
            \addplot+[very thick] table[x=target_cr,y=avg_score,col sep=comma] {plots/ruler/llama3_3b/normal/per_task/vt/estimated_attention__min_var_scores.csv};
            \addplot+[dashed,very thick] table[x=target_cr,y=avg_score,col sep=comma] {plots/ruler/llama3_3b/normal/per_task/vt/h2o_scores.csv};
            \addplot+[dashed,very thick] table[x=target_cr,y=avg_score,col sep=comma] {plots/ruler/llama3_3b/normal/per_task/vt/knorm_scores.csv};
            \addplot+[dashed,very thick] table[x=target_cr,y=avg_score,col sep=comma] {plots/ruler/llama3_3b/normal/per_task/vt/snapkv_scores.csv};
            \addplot+[dashed,very thick] table[x=target_cr,y=avg_score,col sep=comma] {plots/ruler/llama3_3b/normal/per_task/vt/streaming_llm_scores.csv};
            \addplot+[dashed,very thick] table[x=target_cr,y=avg_score,col sep=comma] {plots/ruler/llama3_3b/normal/per_task/vt/tova_scores.csv};
        \nextgroupplot[title={MultiQuery-NIAH}, legend to name=leg,
                   legend cell align=left,
                   legend style={font=\small,fill=none,draw=black,anchor=center,align=left},
                   legend columns=9]
            \addplot+[very thick] table[x=target_cr,y=avg_score,col sep=comma] {plots/ruler/llama3_3b/normal/per_task/niah_multiquery/estimated_attention__h2o_norm_scores.csv};
            \addplot+[very thick] table[x=target_cr,y=avg_score,col sep=comma] {plots/ruler/llama3_3b/normal/per_task/niah_multiquery/estimated_attention__h2o_scores.csv};
            \addplot+[very thick] table[x=target_cr,y=avg_score,col sep=comma] {plots/ruler/llama3_3b/normal/per_task/niah_multiquery/estimated_attention__min_var_full_harmonic_scores.csv};
            \addplot+[very thick] table[x=target_cr,y=avg_score,col sep=comma] {plots/ruler/llama3_3b/normal/per_task/niah_multiquery/estimated_attention__min_var_scores.csv};
            \addplot+[dashed,very thick] table[x=target_cr,y=avg_score,col sep=comma] {plots/ruler/llama3_3b/normal/per_task/niah_multiquery/h2o_scores.csv};
            \addplot+[dashed,very thick] table[x=target_cr,y=avg_score,col sep=comma] {plots/ruler/llama3_3b/normal/per_task/niah_multiquery/knorm_scores.csv};
            \addplot+[dashed,very thick] table[x=target_cr,y=avg_score,col sep=comma] {plots/ruler/llama3_3b/normal/per_task/niah_multiquery/snapkv_scores.csv};
            \addplot+[dashed,very thick] table[x=target_cr,y=avg_score,col sep=comma] {plots/ruler/llama3_3b/normal/per_task/niah_multiquery/streaming_llm_scores.csv};
            \addplot+[dashed,very thick] table[x=target_cr,y=avg_score,col sep=comma] {plots/ruler/llama3_3b/normal/per_task/niah_multiquery/tova_scores.csv};
            \legend{$\pihto$, $\pihtoh$, $\pimin$, $\piminh$, {\footnotesize{}H2O}, {\footnotesize{}K-norm}, {\footnotesize{}SnapKV}, {\footnotesize{}StreamingLLM}, {\footnotesize{}TOVA}}
            \coordinate (c3) at (rel axis cs:1,1);
    \end{groupplot}
    \begin{scope}
        \coordinate (p) at ($(c1)!0.5!(c3)$);
        \node[below] (xlab) at (p |- current bounding box.south) {Compression ratio ($r$)};
        \node[below] (leg) at (xlab.south) {\pgfplotslegendfromname{leg}};
    \end{scope}
\end{tikzpicture}
\caption{\textbf{Scores for each individual dataset split on \llama{}.} Solid lines show the score for the four proposals defined in \Cref{sec:robustness}. Dashed lines show the five top-$k$ eviction methods.}\label{fig:individual-llama}
\end{figure}
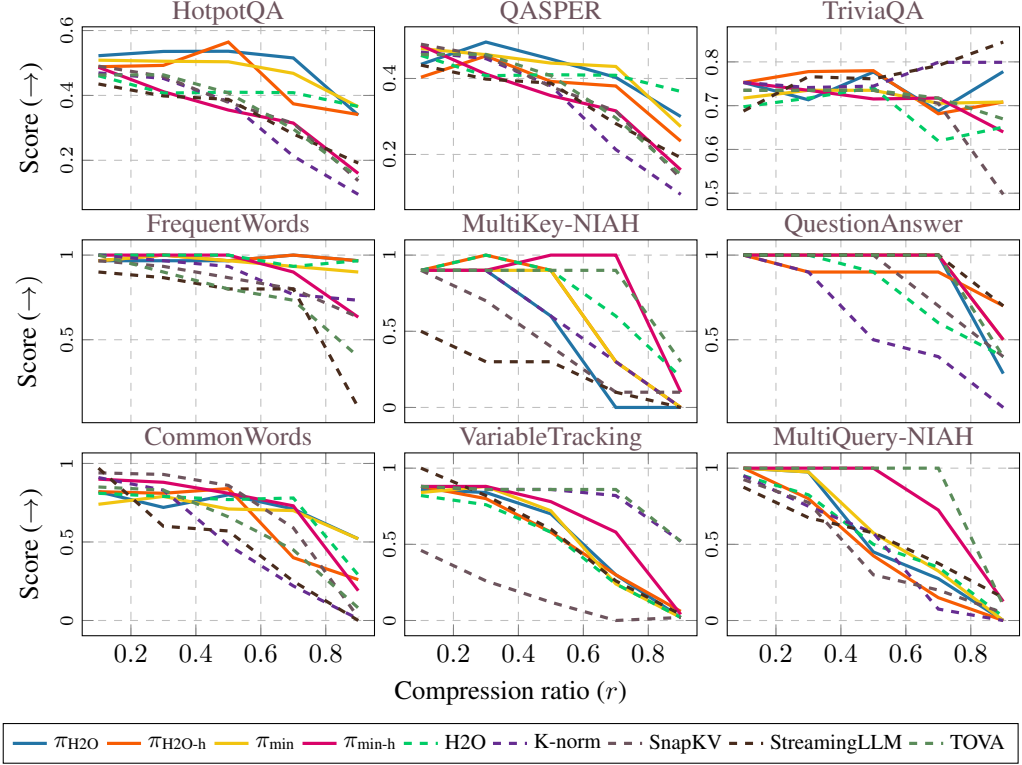

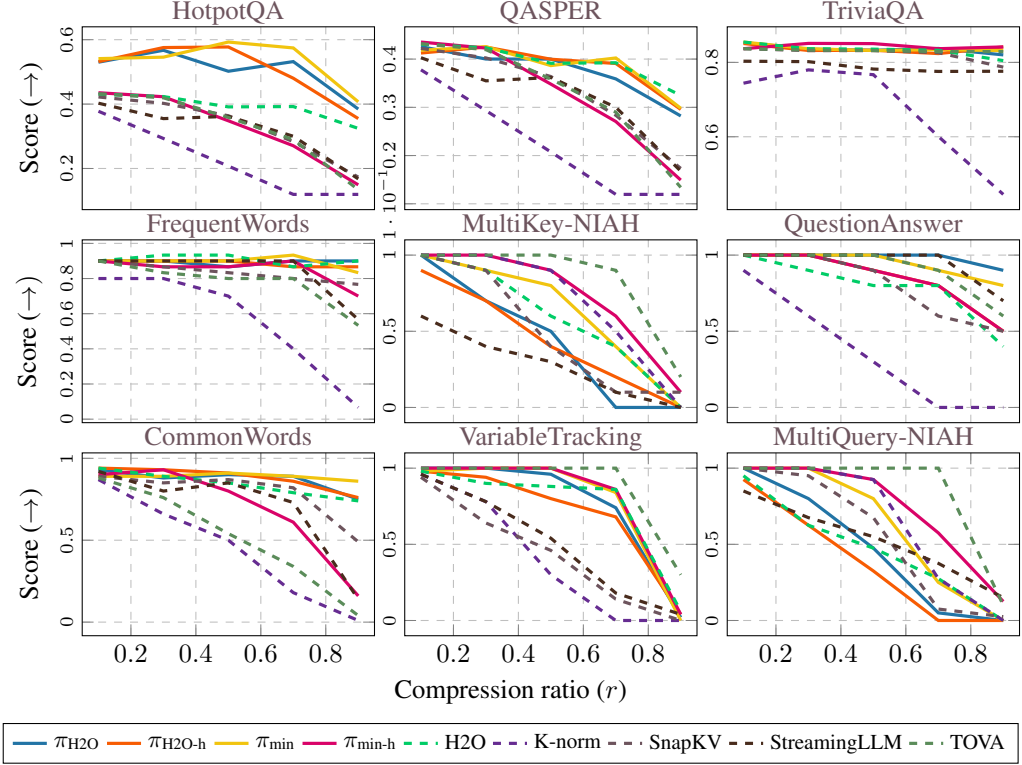
\begin{figure}[t]
\centering%
\begin{tikzpicture}
    \begin{groupplot}[
        group style={group size=3 by 3, horizontal sep=0.4cm, vertical sep=0.4cm},
        height=4cm,
        width=0.39\textwidth,
        grid style=dashed,
        cycle list name=cpalette,
        xmajorgrids=true, ymajorgrids=true,
        every legend image post/.append style={scale=0.75},
        xmin=0.05, xmax=0.95,
        title style={yshift=-0.35cm,font={\strut{}\color{palette-gray}}},
        yticklabel={\pgfmathparse{\tick*0.01}\pgfmathprintnumber{\pgfmathresult}},
        y tick label style={font=\scriptsize, rotate=90},
    ]
        \nextgroupplot[title={HotpotQA},xmajorticks=false,ylabel={Score ($\to$)}]
            \addplot+[very thick] table[x=target_cr,y=avg_score,col sep=comma] {plots/longbench/qwen3_4b/normal/per_dataset/hotpotqa_e/estimated_attention__h2o_norm_scores.csv};
            \addplot+[very thick] table[x=target_cr,y=avg_score,col sep=comma] {plots/longbench/qwen3_4b/normal/per_dataset/hotpotqa_e/estimated_attention__h2o_scores.csv};
            \addplot+[very thick] table[x=target_cr,y=avg_score,col sep=comma] {plots/longbench/qwen3_4b/normal/per_dataset/hotpotqa_e/estimated_attention__min_var_full_harmonic_scores.csv};
            \addplot+[very thick] table[x=target_cr,y=avg_score,col sep=comma] {plots/longbench/qwen3_4b/normal/per_dataset/qasper_e/estimated_attention__min_var_scores.csv};
            \addplot+[dashed,very thick] table[x=target_cr,y=avg_score,col sep=comma] {plots/longbench/qwen3_4b/normal/per_dataset/qasper_e/h2o_scores.csv};
            \addplot+[dashed,very thick] table[x=target_cr,y=avg_score,col sep=comma] {plots/longbench/qwen3_4b/normal/per_dataset/qasper_e/knorm_scores.csv};
            \addplot+[dashed,very thick] table[x=target_cr,y=avg_score,col sep=comma] {plots/longbench/qwen3_4b/normal/per_dataset/qasper_e/snapkv_scores.csv};
            \addplot+[dashed,very thick] table[x=target_cr,y=avg_score,col sep=comma] {plots/longbench/qwen3_4b/normal/per_dataset/qasper_e/streaming_llm_scores.csv};
            \addplot+[dashed,very thick] table[x=target_cr,y=avg_score,col sep=comma] {plots/longbench/qwen3_4b/normal/per_dataset/qasper_e/tova_scores.csv};
        \nextgroupplot[title={QASPER},xmajorticks=false]
            \addplot+[very thick] table[x=target_cr,y=avg_score,col sep=comma] {plots/longbench/qwen3_4b/normal/per_dataset/qasper_e/estimated_attention__h2o_norm_scores.csv};
            \addplot+[very thick] table[x=target_cr,y=avg_score,col sep=comma] {plots/longbench/qwen3_4b/normal/per_dataset/qasper_e/estimated_attention__h2o_scores.csv};
            \addplot+[very thick] table[x=target_cr,y=avg_score,col sep=comma] {plots/longbench/qwen3_4b/normal/per_dataset/qasper_e/estimated_attention__min_var_full_harmonic_scores.csv};
            \addplot+[very thick] table[x=target_cr,y=avg_score,col sep=comma] {plots/longbench/qwen3_4b/normal/per_dataset/qasper_e/estimated_attention__min_var_scores.csv};
            \addplot+[dashed,very thick] table[x=target_cr,y=avg_score,col sep=comma] {plots/longbench/qwen3_4b/normal/per_dataset/qasper_e/h2o_scores.csv};
            \addplot+[dashed,very thick] table[x=target_cr,y=avg_score,col sep=comma] {plots/longbench/qwen3_4b/normal/per_dataset/qasper_e/knorm_scores.csv};
            \addplot+[dashed,very thick] table[x=target_cr,y=avg_score,col sep=comma] {plots/longbench/qwen3_4b/normal/per_dataset/qasper_e/snapkv_scores.csv};
            \addplot+[dashed,very thick] table[x=target_cr,y=avg_score,col sep=comma] {plots/longbench/qwen3_4b/normal/per_dataset/qasper_e/streaming_llm_scores.csv};
            \addplot+[dashed,very thick] table[x=target_cr,y=avg_score,col sep=comma] {plots/longbench/qwen3_4b/normal/per_dataset/qasper_e/tova_scores.csv};
        \nextgroupplot[title={TriviaQA},xmajorticks=false]
            \addplot+[very thick] table[x=target_cr,y=avg_score,col sep=comma] {plots/longbench/qwen3_4b/normal/per_dataset/triviaqa_e/estimated_attention__h2o_norm_scores.csv};
            \addplot+[very thick] table[x=target_cr,y=avg_score,col sep=comma] {plots/longbench/qwen3_4b/normal/per_dataset/triviaqa_e/estimated_attention__h2o_scores.csv};
            \addplot+[very thick] table[x=target_cr,y=avg_score,col sep=comma] {plots/longbench/qwen3_4b/normal/per_dataset/triviaqa_e/estimated_attention__min_var_full_harmonic_scores.csv};
            \addplot+[very thick] table[x=target_cr,y=avg_score,col sep=comma] {plots/longbench/qwen3_4b/normal/per_dataset/triviaqa_e/estimated_attention__min_var_scores.csv};
            \addplot+[dashed,very thick] table[x=target_cr,y=avg_score,col sep=comma] {plots/longbench/qwen3_4b/normal/per_dataset/triviaqa_e/h2o_scores.csv};
            \addplot+[dashed,very thick] table[x=target_cr,y=avg_score,col sep=comma] {plots/longbench/qwen3_4b/normal/per_dataset/triviaqa_e/knorm_scores.csv};
            \addplot+[dashed,very thick] table[x=target_cr,y=avg_score,col sep=comma] {plots/longbench/qwen3_4b/normal/per_dataset/triviaqa_e/snapkv_scores.csv};
            \addplot+[dashed,very thick] table[x=target_cr,y=avg_score,col sep=comma] {plots/longbench/qwen3_4b/normal/per_dataset/triviaqa_e/streaming_llm_scores.csv};
            \addplot+[dashed,very thick] table[x=target_cr,y=avg_score,col sep=comma] {plots/longbench/qwen3_4b/normal/per_dataset/triviaqa_e/tova_scores.csv};
        \nextgroupplot[title={FrequentWords},xmajorticks=false,ylabel={Score ($\to)$}]
            \addplot+[very thick] table[x=target_cr,y=avg_score,col sep=comma] {plots/ruler/qwen3_4b/normal/per_task/fwe/estimated_attention__h2o_norm_scores.csv};
            \addplot+[very thick] table[x=target_cr,y=avg_score,col sep=comma] {plots/ruler/qwen3_4b/normal/per_task/fwe/estimated_attention__h2o_scores.csv};
            \addplot+[very thick] table[x=target_cr,y=avg_score,col sep=comma] {plots/ruler/qwen3_4b/normal/per_task/fwe/estimated_attention__min_var_full_harmonic_scores.csv};
            \addplot+[very thick] table[x=target_cr,y=avg_score,col sep=comma] {plots/ruler/qwen3_4b/normal/per_task/fwe/estimated_attention__min_var_scores.csv};
            \addplot+[dashed,very thick] table[x=target_cr,y=avg_score,col sep=comma] {plots/ruler/qwen3_4b/normal/per_task/fwe/h2o_scores.csv};
            \addplot+[dashed,very thick] table[x=target_cr,y=avg_score,col sep=comma] {plots/ruler/qwen3_4b/normal/per_task/fwe/knorm_scores.csv};
            \addplot+[dashed,very thick] table[x=target_cr,y=avg_score,col sep=comma] {plots/ruler/qwen3_4b/normal/per_task/fwe/snapkv_scores.csv};
            \addplot+[dashed,very thick] table[x=target_cr,y=avg_score,col sep=comma] {plots/ruler/qwen3_4b/normal/per_task/fwe/streaming_llm_scores.csv};
            \addplot+[dashed,very thick] table[x=target_cr,y=avg_score,col sep=comma] {plots/ruler/qwen3_4b/normal/per_task/fwe/tova_scores.csv};
        \nextgroupplot[title={MultiKey-NIAH},xmajorticks=false]
            \addplot+[very thick] table[x=target_cr,y=avg_score,col sep=comma] {plots/ruler/qwen3_4b/normal/per_task/niah_multikey_1/estimated_attention__h2o_norm_scores.csv};
            \addplot+[very thick] table[x=target_cr,y=avg_score,col sep=comma] {plots/ruler/qwen3_4b/normal/per_task/niah_multikey_1/estimated_attention__h2o_scores.csv};
            \addplot+[very thick] table[x=target_cr,y=avg_score,col sep=comma] {plots/ruler/qwen3_4b/normal/per_task/niah_multikey_1/estimated_attention__min_var_full_harmonic_scores.csv};
            \addplot+[very thick] table[x=target_cr,y=avg_score,col sep=comma] {plots/ruler/qwen3_4b/normal/per_task/niah_multikey_1/estimated_attention__min_var_scores.csv};
            \addplot+[dashed,very thick] table[x=target_cr,y=avg_score,col sep=comma] {plots/ruler/qwen3_4b/normal/per_task/niah_multikey_1/h2o_scores.csv};
            \addplot+[dashed,very thick] table[x=target_cr,y=avg_score,col sep=comma] {plots/ruler/qwen3_4b/normal/per_task/niah_multikey_1/knorm_scores.csv};
            \addplot+[dashed,very thick] table[x=target_cr,y=avg_score,col sep=comma] {plots/ruler/qwen3_4b/normal/per_task/niah_multikey_1/snapkv_scores.csv};
            \addplot+[dashed,very thick] table[x=target_cr,y=avg_score,col sep=comma] {plots/ruler/qwen3_4b/normal/per_task/niah_multikey_1/streaming_llm_scores.csv};
            \addplot+[dashed,very thick] table[x=target_cr,y=avg_score,col sep=comma] {plots/ruler/qwen3_4b/normal/per_task/niah_multikey_1/tova_scores.csv};
        \nextgroupplot[title={QuestionAnswer},xmajorticks=false]
            \addplot+[very thick] table[x=target_cr,y=avg_score,col sep=comma] {plots/ruler/qwen3_4b/normal/per_task/qa_1/estimated_attention__h2o_norm_scores.csv};
            \addplot+[very thick] table[x=target_cr,y=avg_score,col sep=comma] {plots/ruler/qwen3_4b/normal/per_task/qa_1/estimated_attention__h2o_scores.csv};
            \addplot+[very thick] table[x=target_cr,y=avg_score,col sep=comma] {plots/ruler/qwen3_4b/normal/per_task/qa_1/estimated_attention__min_var_full_harmonic_scores.csv};
            \addplot+[very thick] table[x=target_cr,y=avg_score,col sep=comma] {plots/ruler/qwen3_4b/normal/per_task/qa_1/estimated_attention__min_var_scores.csv};
            \addplot+[dashed,very thick] table[x=target_cr,y=avg_score,col sep=comma] {plots/ruler/qwen3_4b/normal/per_task/qa_1/h2o_scores.csv};
            \addplot+[dashed,very thick] table[x=target_cr,y=avg_score,col sep=comma] {plots/ruler/qwen3_4b/normal/per_task/qa_1/knorm_scores.csv};
            \addplot+[dashed,very thick] table[x=target_cr,y=avg_score,col sep=comma] {plots/ruler/qwen3_4b/normal/per_task/qa_1/snapkv_scores.csv};
            \addplot+[dashed,very thick] table[x=target_cr,y=avg_score,col sep=comma] {plots/ruler/qwen3_4b/normal/per_task/qa_1/streaming_llm_scores.csv};
            \addplot+[dashed,very thick] table[x=target_cr,y=avg_score,col sep=comma] {plots/ruler/qwen3_4b/normal/per_task/qa_1/tova_scores.csv};
        \nextgroupplot[title={CommonWords},ylabel={Score ($\to$)}]
            \addplot+[very thick] table[x=target_cr,y=avg_score,col sep=comma] {plots/ruler/qwen3_4b/normal/per_task/cwe/estimated_attention__h2o_norm_scores.csv};
            \addplot+[very thick] table[x=target_cr,y=avg_score,col sep=comma] {plots/ruler/qwen3_4b/normal/per_task/cwe/estimated_attention__h2o_scores.csv};
            \addplot+[very thick] table[x=target_cr,y=avg_score,col sep=comma] {plots/ruler/qwen3_4b/normal/per_task/cwe/estimated_attention__min_var_full_harmonic_scores.csv};
            \addplot+[very thick] table[x=target_cr,y=avg_score,col sep=comma] {plots/ruler/qwen3_4b/normal/per_task/cwe/estimated_attention__min_var_scores.csv};
            \addplot+[dashed,very thick] table[x=target_cr,y=avg_score,col sep=comma] {plots/ruler/qwen3_4b/normal/per_task/cwe/h2o_scores.csv};
            \addplot+[dashed,very thick] table[x=target_cr,y=avg_score,col sep=comma] {plots/ruler/qwen3_4b/normal/per_task/cwe/knorm_scores.csv};
            \addplot+[dashed,very thick] table[x=target_cr,y=avg_score,col sep=comma] {plots/ruler/qwen3_4b/normal/per_task/cwe/snapkv_scores.csv};
            \addplot+[dashed,very thick] table[x=target_cr,y=avg_score,col sep=comma] {plots/ruler/qwen3_4b/normal/per_task/cwe/streaming_llm_scores.csv};
            \addplot+[dashed,very thick] table[x=target_cr,y=avg_score,col sep=comma] {plots/ruler/qwen3_4b/normal/per_task/cwe/tova_scores.csv};
            \coordinate (c1) at (rel axis cs:-0.275,1);
        \nextgroupplot[title={VariableTracking}]
            \addplot+[very thick] table[x=target_cr,y=avg_score,col sep=comma] {plots/ruler/qwen3_4b/normal/per_task/vt/estimated_attention__h2o_norm_scores.csv};
            \addplot+[very thick] table[x=target_cr,y=avg_score,col sep=comma] {plots/ruler/qwen3_4b/normal/per_task/vt/estimated_attention__h2o_scores.csv};
            \addplot+[very thick] table[x=target_cr,y=avg_score,col sep=comma] {plots/ruler/qwen3_4b/normal/per_task/vt/estimated_attention__min_var_full_harmonic_scores.csv};
            \addplot+[very thick] table[x=target_cr,y=avg_score,col sep=comma] {plots/ruler/qwen3_4b/normal/per_task/vt/estimated_attention__min_var_scores.csv};
            \addplot+[dashed,very thick] table[x=target_cr,y=avg_score,col sep=comma] {plots/ruler/qwen3_4b/normal/per_task/vt/h2o_scores.csv};
            \addplot+[dashed,very thick] table[x=target_cr,y=avg_score,col sep=comma] {plots/ruler/qwen3_4b/normal/per_task/vt/knorm_scores.csv};
            \addplot+[dashed,very thick] table[x=target_cr,y=avg_score,col sep=comma] {plots/ruler/qwen3_4b/normal/per_task/vt/snapkv_scores.csv};
            \addplot+[dashed,very thick] table[x=target_cr,y=avg_score,col sep=comma] {plots/ruler/qwen3_4b/normal/per_task/vt/streaming_llm_scores.csv};
            \addplot+[dashed,very thick] table[x=target_cr,y=avg_score,col sep=comma] {plots/ruler/qwen3_4b/normal/per_task/vt/tova_scores.csv};
        \nextgroupplot[title={MultiQuery-NIAH}, legend to name=leg,
                   legend cell align=left,
                   legend style={font=\small,fill=none,draw=black,anchor=center,align=left},
                   legend columns=9]
            \addplot+[very thick] table[x=target_cr,y=avg_score,col sep=comma] {plots/ruler/qwen3_4b/normal/per_task/niah_multiquery/estimated_attention__h2o_norm_scores.csv};
            \addplot+[very thick] table[x=target_cr,y=avg_score,col sep=comma] {plots/ruler/qwen3_4b/normal/per_task/niah_multiquery/estimated_attention__h2o_scores.csv};
            \addplot+[very thick] table[x=target_cr,y=avg_score,col sep=comma] {plots/ruler/qwen3_4b/normal/per_task/niah_multiquery/estimated_attention__min_var_full_harmonic_scores.csv};
            \addplot+[very thick] table[x=target_cr,y=avg_score,col sep=comma] {plots/ruler/qwen3_4b/normal/per_task/niah_multiquery/estimated_attention__min_var_scores.csv};
            \addplot+[dashed,very thick] table[x=target_cr,y=avg_score,col sep=comma] {plots/ruler/qwen3_4b/normal/per_task/niah_multiquery/h2o_scores.csv};
            \addplot+[dashed,very thick] table[x=target_cr,y=avg_score,col sep=comma] {plots/ruler/qwen3_4b/normal/per_task/niah_multiquery/knorm_scores.csv};
            \addplot+[dashed,very thick] table[x=target_cr,y=avg_score,col sep=comma] {plots/ruler/qwen3_4b/normal/per_task/niah_multiquery/snapkv_scores.csv};
            \addplot+[dashed,very thick] table[x=target_cr,y=avg_score,col sep=comma] {plots/ruler/qwen3_4b/normal/per_task/niah_multiquery/streaming_llm_scores.csv};
            \addplot+[dashed,very thick] table[x=target_cr,y=avg_score,col sep=comma] {plots/ruler/qwen3_4b/normal/per_task/niah_multiquery/tova_scores.csv};
            \legend{$\pihto$, $\pihtoh$, $\pimin$, $\piminh$, {\footnotesize{}H2O}, {\footnotesize{}K-norm}, {\footnotesize{}SnapKV}, {\footnotesize{}StreamingLLM}, {\footnotesize{}TOVA}}
            \coordinate (c3) at (rel axis cs:1,1);
    \end{groupplot}
    \begin{scope}
        \coordinate (p) at ($(c1)!0.5!(c3)$);
        \node[below] (xlab) at (p |- current bounding box.south) {Compression ratio ($r$)};
        \node[below] (leg) at (xlab.south) {\pgfplotslegendfromname{leg}};
    \end{scope}
\end{tikzpicture}
\caption{\textbf{Scores for each individual dataset split on \qwen{}.} Solid lines show the score for the four proposals defined in \Cref{sec:robustness}. Dashed lines show the five top-$k$ eviction methods.}\label{fig:individual-qwen}
\end{figure}

\subsection{Target vs Effective Compression Ratio}\label{app:ccp-global-error}

\Cref{tab:eff-cr} empirically shows that when we set the target compression ratio in \Cref{alg:number-of-samples}, we obtain an effective compression ratio with minimal error. 

\begin{table}[b]
\centering
\resizebox{\textwidth}{!}{%
\begin{tabular}{c|cccc}
\hline\hline
& \multicolumn{4}{c}{Mean effective $r$ ($\pm$ std)} \\
Target $r$ & $\pihto$ & $\pihtoh$ & $\pimin$ & $\piminh$ \\
\midrule
0.1 & 0.0999 $\pm$ 0.0003 & 0.1000 $\pm$ 0.0003 & 0.1000 $\pm$ 0.0002 & 0.0999 $\pm$ 0.0003 \\
0.3 & 0.2999 $\pm$ 0.0004 & 0.3001 $\pm$ 0.0004 & 0.3000 $\pm$ 0.0004 & 0.2999 $\pm$ 0.0005 \\
0.5 & 0.4999 $\pm$ 0.0004 & 0.5000 $\pm$ 0.0004 & 0.5000 $\pm$ 0.0004 & 0.5001 $\pm$ 0.0004 \\
0.7 & 0.6999 $\pm$ 0.0004 & 0.7000 $\pm$ 0.0004 & 0.7000 $\pm$ 0.0004 & 0.7001 $\pm$ 0.0004 \\
0.9 & 0.9000 $\pm$ 0.0003 & 0.9000 $\pm$ 0.0002 & 0.8999 $\pm$ 0.0002 & 0.9000 $\pm$ 0.0002 \\
\hline\hline
\end{tabular}}
\caption{\textbf{Difference between target and effective compression ratio is insignificant} for \llama{} on the RULER dataset. Results are similar for \qwen{} and LongBench.}\label{tab:eff-cr}
\end{table}

\section{Experimental Details}\label{app:exp-details}

We use the evaluation pipeline and implementations of SnapKV, TOVA, H2O, StreamingLLM, and K-norm available in KVPress v0.5.3 \citep{kvpress}.
All experiments are run on two NVIDIA RTX A6000 (48GB) and eight NVIDIA RTX A5000 (24GB).
All experiments can be run on a 24GB GPU, taking approximately $3$ to $5$ seconds to completely run each example in either LongBench or RULER.
The full research project required more compute than what is reported in the paper due to preliminary experiments as well as initial experimentation that either contained errors in implementation or were not as competitive.

\begin{figure}[t]
\centering%
\begin{tikzpicture}
    \begin{groupplot}[
        group style={group size=5 by 2, horizontal sep=0.1cm, vertical sep=0.2cm},
        height=3.75cm,
        width=0.29\textwidth,
        grid style=dashed,
        cycle list name=cpalette,
        xmajorgrids=true, ymajorgrids=true,
        every legend image post/.append style={scale=0.75},
        xmin=0.05, xmax=0.95,
        title style={yshift=-0.35cm,font={\strut{}\color{palette-gray}}},
        y tick label style={font=\scriptsize, rotate=90},
    ]

    \nextgroupplot[title={$\tau=1$},
                   ylabel={Win score ($\to$)},xmajorticks=false]
        \addplot+[very thick] table[x=target_cr,y=ranking,col sep=comma]{plots/ruler/llama3_3b/ist1.0/estimated_attention__h2o.csv};
        \addplot+[very thick] table[x=target_cr,y=ranking,col sep=comma]{plots/ruler/llama3_3b/ist1.0/estimated_attention__h2o_norm.csv};
        \addplot+[very thick] table[x=target_cr,y=ranking,col sep=comma]{plots/ruler/llama3_3b/ist1.0/estimated_attention__min_var.csv};
        \addplot+[very thick] table[x=target_cr,y=ranking,col sep=comma]{plots/ruler/llama3_3b/ist1.0/estimated_attention__min_var_full_harmonic.csv};
        \addplot+[dashed,very thick] table[x=target_cr,y=ranking,col sep=comma]{plots/ruler/llama3_3b/ist1.0/h2o.csv};
        \addplot+[dashed,very thick] table[x=target_cr,y=ranking,col sep=comma]{plots/ruler/llama3_3b/ist1.0/knorm.csv};
        \addplot+[dashed,very thick] table[x=target_cr,y=ranking,col sep=comma]{plots/ruler/llama3_3b/ist1.0/snapkv.csv};
        \addplot+[dashed,very thick] table[x=target_cr,y=ranking,col sep=comma]{plots/ruler/llama3_3b/ist1.0/streaming_llm.csv};
        \addplot+[dashed,very thick] table[x=target_cr,y=ranking,col sep=comma]{plots/ruler/llama3_3b/ist1.0/tova.csv};
        \coordinate (istc1) at (rel axis cs:-0.275,1);

    \nextgroupplot[title={$\tau=2$},
                   ymajorticks=false,xmajorticks=false]
        \addplot+[very thick] table[x=target_cr,y=ranking,col sep=comma]{plots/ruler/llama3_3b/ist2.0/estimated_attention__h2o.csv};
        \addplot+[very thick] table[x=target_cr,y=ranking,col sep=comma]{plots/ruler/llama3_3b/ist2.0/estimated_attention__h2o_norm.csv};
        \addplot+[very thick] table[x=target_cr,y=ranking,col sep=comma]{plots/ruler/llama3_3b/ist2.0/estimated_attention__min_var.csv};
        \addplot+[very thick] table[x=target_cr,y=ranking,col sep=comma]{plots/ruler/llama3_3b/ist2.0/estimated_attention__min_var_full_harmonic.csv};
        \addplot+[dashed,very thick] table[x=target_cr,y=ranking,col sep=comma]{plots/ruler/llama3_3b/ist2.0/h2o.csv};
        \addplot+[dashed,very thick] table[x=target_cr,y=ranking,col sep=comma]{plots/ruler/llama3_3b/ist2.0/knorm.csv};
        \addplot+[dashed,very thick] table[x=target_cr,y=ranking,col sep=comma]{plots/ruler/llama3_3b/ist2.0/snapkv.csv};
        \addplot+[dashed,very thick] table[x=target_cr,y=ranking,col sep=comma]{plots/ruler/llama3_3b/ist2.0/streaming_llm.csv};
        \addplot+[dashed,very thick] table[x=target_cr,y=ranking,col sep=comma]{plots/ruler/llama3_3b/ist2.0/tova.csv};

    \nextgroupplot[title={$\tau=3$},
                   ymajorticks=false,xmajorticks=false]
        \addplot+[very thick] table[x=target_cr,y=ranking,col sep=comma]{plots/ruler/llama3_3b/ist3.0/estimated_attention__h2o.csv};
        \addplot+[very thick] table[x=target_cr,y=ranking,col sep=comma]{plots/ruler/llama3_3b/ist3.0/estimated_attention__h2o_norm.csv};
        \addplot+[very thick] table[x=target_cr,y=ranking,col sep=comma]{plots/ruler/llama3_3b/ist3.0/estimated_attention__min_var.csv};
        \addplot+[very thick] table[x=target_cr,y=ranking,col sep=comma]{plots/ruler/llama3_3b/ist3.0/estimated_attention__min_var_full_harmonic.csv};
        \addplot+[dashed,very thick] table[x=target_cr,y=ranking,col sep=comma]{plots/ruler/llama3_3b/ist3.0/h2o.csv};
        \addplot+[dashed,very thick] table[x=target_cr,y=ranking,col sep=comma]{plots/ruler/llama3_3b/ist3.0/knorm.csv};
        \addplot+[dashed,very thick] table[x=target_cr,y=ranking,col sep=comma]{plots/ruler/llama3_3b/ist3.0/snapkv.csv};
        \addplot+[dashed,very thick] table[x=target_cr,y=ranking,col sep=comma]{plots/ruler/llama3_3b/ist3.0/streaming_llm.csv};
        \addplot+[dashed,very thick] table[x=target_cr,y=ranking,col sep=comma]{plots/ruler/llama3_3b/ist3.0/tova.csv};

    \nextgroupplot[title={$\tau=4$},
                   ymajorticks=false,xmajorticks=false]
        \addplot+[very thick] table[x=target_cr,y=ranking,col sep=comma]{plots/ruler/llama3_3b/ist4.0/estimated_attention__h2o.csv};
        \addplot+[very thick] table[x=target_cr,y=ranking,col sep=comma]{plots/ruler/llama3_3b/ist4.0/estimated_attention__h2o_norm.csv};
        \addplot+[very thick] table[x=target_cr,y=ranking,col sep=comma]{plots/ruler/llama3_3b/ist4.0/estimated_attention__min_var.csv};
        \addplot+[very thick] table[x=target_cr,y=ranking,col sep=comma]{plots/ruler/llama3_3b/ist4.0/estimated_attention__min_var_full_harmonic.csv};
        \addplot+[dashed,very thick] table[x=target_cr,y=ranking,col sep=comma]{plots/ruler/llama3_3b/ist4.0/h2o.csv};
        \addplot+[dashed,very thick] table[x=target_cr,y=ranking,col sep=comma]{plots/ruler/llama3_3b/ist4.0/knorm.csv};
        \addplot+[dashed,very thick] table[x=target_cr,y=ranking,col sep=comma]{plots/ruler/llama3_3b/ist4.0/snapkv.csv};
        \addplot+[dashed,very thick] table[x=target_cr,y=ranking,col sep=comma]{plots/ruler/llama3_3b/ist4.0/streaming_llm.csv};
        \addplot+[dashed,very thick] table[x=target_cr,y=ranking,col sep=comma]{plots/ruler/llama3_3b/ist4.0/tova.csv};

    \nextgroupplot[title={$\tau=5$},
                   ymajorticks=false,xmajorticks=false]
        \addplot+[very thick] table[x=target_cr,y=ranking,col sep=comma]{plots/ruler/llama3_3b/ist5.0/estimated_attention__h2o.csv};
        \addplot+[very thick] table[x=target_cr,y=ranking,col sep=comma]{plots/ruler/llama3_3b/ist5.0/estimated_attention__h2o_norm.csv};
        \addplot+[very thick] table[x=target_cr,y=ranking,col sep=comma]{plots/ruler/llama3_3b/ist5.0/estimated_attention__min_var.csv};
        \addplot+[very thick] table[x=target_cr,y=ranking,col sep=comma]{plots/ruler/llama3_3b/ist5.0/estimated_attention__min_var_full_harmonic.csv};
        \addplot+[dashed,very thick] table[x=target_cr,y=ranking,col sep=comma]{plots/ruler/llama3_3b/ist5.0/h2o.csv};
        \addplot+[dashed,very thick] table[x=target_cr,y=ranking,col sep=comma]{plots/ruler/llama3_3b/ist5.0/knorm.csv};
        \addplot+[dashed,very thick] table[x=target_cr,y=ranking,col sep=comma]{plots/ruler/llama3_3b/ist5.0/snapkv.csv};
        \addplot+[dashed,very thick] table[x=target_cr,y=ranking,col sep=comma]{plots/ruler/llama3_3b/ist5.0/streaming_llm.csv};
        \addplot+[dashed,very thick] table[x=target_cr,y=ranking,col sep=comma]{plots/ruler/llama3_3b/ist5.0/tova.csv};
        \pgfplotsextra{\coordinate (istrllamaRow) at (rel axis cs:1,0.5);}

    \nextgroupplot[ylabel={Win score ($\to$)}]
        \addplot+[very thick] table[x=target_cr,y=ranking,col sep=comma]{plots/ruler/qwen3_4b/ist1.0/estimated_attention__h2o.csv};
        \addplot+[very thick] table[x=target_cr,y=ranking,col sep=comma]{plots/ruler/qwen3_4b/ist1.0/estimated_attention__h2o_norm.csv};
        \addplot+[very thick] table[x=target_cr,y=ranking,col sep=comma]{plots/ruler/qwen3_4b/ist1.0/estimated_attention__min_var.csv};
        \addplot+[very thick] table[x=target_cr,y=ranking,col sep=comma]{plots/ruler/qwen3_4b/ist1.0/estimated_attention__min_var_full_harmonic.csv};
        \addplot+[dashed,very thick] table[x=target_cr,y=ranking,col sep=comma]{plots/ruler/qwen3_4b/ist1.0/h2o.csv};
        \addplot+[dashed,very thick] table[x=target_cr,y=ranking,col sep=comma]{plots/ruler/qwen3_4b/ist1.0/knorm.csv};
        \addplot+[dashed,very thick] table[x=target_cr,y=ranking,col sep=comma]{plots/ruler/qwen3_4b/ist1.0/snapkv.csv};
        \addplot+[dashed,very thick] table[x=target_cr,y=ranking,col sep=comma]{plots/ruler/qwen3_4b/ist1.0/streaming_llm.csv};
        \addplot+[dashed,very thick] table[x=target_cr,y=ranking,col sep=comma]{plots/ruler/qwen3_4b/ist1.0/tova.csv};

    \nextgroupplot[ymajorticks=false]
        \addplot+[very thick] table[x=target_cr,y=ranking,col sep=comma]{plots/ruler/qwen3_4b/ist2.0/estimated_attention__h2o.csv};
        \addplot+[very thick] table[x=target_cr,y=ranking,col sep=comma]{plots/ruler/qwen3_4b/ist2.0/estimated_attention__h2o_norm.csv};
        \addplot+[very thick] table[x=target_cr,y=ranking,col sep=comma]{plots/ruler/qwen3_4b/ist2.0/estimated_attention__min_var.csv};
        \addplot+[very thick] table[x=target_cr,y=ranking,col sep=comma]{plots/ruler/qwen3_4b/ist2.0/estimated_attention__min_var_full_harmonic.csv};
        \addplot+[dashed,very thick] table[x=target_cr,y=ranking,col sep=comma]{plots/ruler/qwen3_4b/ist2.0/h2o.csv};
        \addplot+[dashed,very thick] table[x=target_cr,y=ranking,col sep=comma]{plots/ruler/qwen3_4b/ist2.0/knorm.csv};
        \addplot+[dashed,very thick] table[x=target_cr,y=ranking,col sep=comma]{plots/ruler/qwen3_4b/ist2.0/snapkv.csv};
        \addplot+[dashed,very thick] table[x=target_cr,y=ranking,col sep=comma]{plots/ruler/qwen3_4b/ist2.0/streaming_llm.csv};
        \addplot+[dashed,very thick] table[x=target_cr,y=ranking,col sep=comma]{plots/ruler/qwen3_4b/ist2.0/tova.csv};

    \nextgroupplot[ymajorticks=false]
        \addplot+[very thick] table[x=target_cr,y=ranking,col sep=comma]{plots/ruler/qwen3_4b/ist3.0/estimated_attention__h2o.csv};
        \addplot+[very thick] table[x=target_cr,y=ranking,col sep=comma]{plots/ruler/qwen3_4b/ist3.0/estimated_attention__h2o_norm.csv};
        \addplot+[very thick] table[x=target_cr,y=ranking,col sep=comma]{plots/ruler/qwen3_4b/ist3.0/estimated_attention__min_var.csv};
        \addplot+[very thick] table[x=target_cr,y=ranking,col sep=comma]{plots/ruler/qwen3_4b/ist3.0/estimated_attention__min_var_full_harmonic.csv};
        \addplot+[dashed,very thick] table[x=target_cr,y=ranking,col sep=comma]{plots/ruler/qwen3_4b/ist3.0/h2o.csv};
        \addplot+[dashed,very thick] table[x=target_cr,y=ranking,col sep=comma]{plots/ruler/qwen3_4b/ist3.0/knorm.csv};
        \addplot+[dashed,very thick] table[x=target_cr,y=ranking,col sep=comma]{plots/ruler/qwen3_4b/ist3.0/snapkv.csv};
        \addplot+[dashed,very thick] table[x=target_cr,y=ranking,col sep=comma]{plots/ruler/qwen3_4b/ist3.0/streaming_llm.csv};
        \addplot+[dashed,very thick] table[x=target_cr,y=ranking,col sep=comma]{plots/ruler/qwen3_4b/ist3.0/tova.csv};

    \nextgroupplot[ymajorticks=false]
        \addplot+[very thick] table[x=target_cr,y=ranking,col sep=comma]{plots/ruler/qwen3_4b/ist4.0/estimated_attention__h2o.csv};
        \addplot+[very thick] table[x=target_cr,y=ranking,col sep=comma]{plots/ruler/qwen3_4b/ist4.0/estimated_attention__h2o_norm.csv};
        \addplot+[very thick] table[x=target_cr,y=ranking,col sep=comma]{plots/ruler/qwen3_4b/ist4.0/estimated_attention__min_var.csv};
        \addplot+[very thick] table[x=target_cr,y=ranking,col sep=comma]{plots/ruler/qwen3_4b/ist4.0/estimated_attention__min_var_full_harmonic.csv};
        \addplot+[dashed,very thick] table[x=target_cr,y=ranking,col sep=comma]{plots/ruler/qwen3_4b/ist4.0/h2o.csv};
        \addplot+[dashed,very thick] table[x=target_cr,y=ranking,col sep=comma]{plots/ruler/qwen3_4b/ist4.0/knorm.csv};
        \addplot+[dashed,very thick] table[x=target_cr,y=ranking,col sep=comma]{plots/ruler/qwen3_4b/ist4.0/snapkv.csv};
        \addplot+[dashed,very thick] table[x=target_cr,y=ranking,col sep=comma]{plots/ruler/qwen3_4b/ist4.0/streaming_llm.csv};
        \addplot+[dashed,very thick] table[x=target_cr,y=ranking,col sep=comma]{plots/ruler/qwen3_4b/ist4.0/tova.csv};

    \nextgroupplot[ymajorticks=false,
                   legend to name=istleg2,
                   legend cell align=left,
                   legend style={font=\small,fill=none,draw=black,anchor=center,align=left},
                   legend columns=9]
        \addplot+[very thick] table[x=target_cr,y=ranking,col sep=comma]{plots/ruler/qwen3_4b/ist5.0/estimated_attention__h2o.csv};
        \addplot+[very thick] table[x=target_cr,y=ranking,col sep=comma]{plots/ruler/qwen3_4b/ist5.0/estimated_attention__h2o_norm.csv};
        \addplot+[very thick] table[x=target_cr,y=ranking,col sep=comma]{plots/ruler/qwen3_4b/ist5.0/estimated_attention__min_var.csv};
        \addplot+[very thick] table[x=target_cr,y=ranking,col sep=comma]{plots/ruler/qwen3_4b/ist5.0/estimated_attention__min_var_full_harmonic.csv};
        \addplot+[dashed,very thick] table[x=target_cr,y=ranking,col sep=comma]{plots/ruler/qwen3_4b/ist5.0/h2o.csv};
        \addplot+[dashed,very thick] table[x=target_cr,y=ranking,col sep=comma]{plots/ruler/qwen3_4b/ist5.0/knorm.csv};
        \addplot+[dashed,very thick] table[x=target_cr,y=ranking,col sep=comma]{plots/ruler/qwen3_4b/ist5.0/snapkv.csv};
        \addplot+[dashed,very thick] table[x=target_cr,y=ranking,col sep=comma]{plots/ruler/qwen3_4b/ist5.0/streaming_llm.csv};
        \addplot+[dashed,very thick] table[x=target_cr,y=ranking,col sep=comma]{plots/ruler/qwen3_4b/ist5.0/tova.csv};
        \legend{$\pihto$, $\pihtoh$, $\pimin$, $\piminh$, H2O, K-norm, SnapKV, StreamingLLM, TOVA}
        \coordinate (istc4) at (rel axis cs:1,1);
        \pgfplotsextra{\coordinate (istrqwenRow) at (rel axis cs:1,0.5);}

    \end{groupplot}
    \begin{scope}
        \coordinate (p) at ($(istc1)!0.5!(istc4)$);
        \node[below] (xlab) at (p |- current bounding box.south) {Compression ratio ($r$)};
        \node[below] (leg) at (xlab.south) {\pgfplotslegendfromname{robleg}};
        \node[rotate=270, anchor=north, font=\small, yshift=0.45cm] at (istrllamaRow) {\llama{}};
        \node[rotate=270, anchor=north, font=\small, yshift=0.45cm] at (istrqwenRow) {\qwen{}};
    \end{scope}
\end{tikzpicture}
\vspace{-0.25cm}
\caption{\textbf{Importance sampling temperature scaling pairwise win scores.} \llama{} (top) and \qwen{} (bottom) pairwise win scores at different importance sampling temperature scale values $\tau$.}\label{fig:ist-ruler-ranking}
\end{figure}
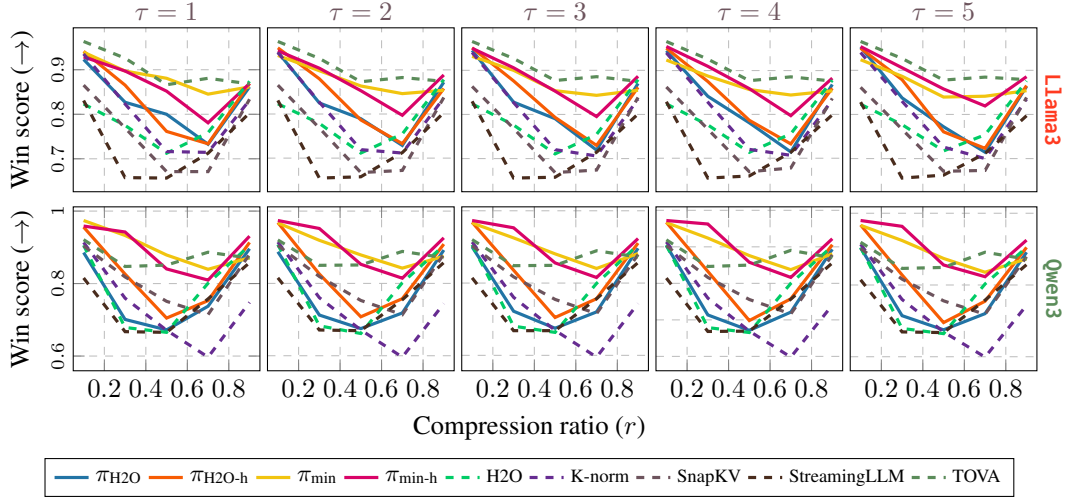

\section{Minimum Variance Proposals and Priors}\label{app:proposals}

We define the following four proposal distributions
\begin{align}
    \pimin(\mathbf{v}_i)\defeq \frac{p_t(\mathbf{v}_i)\cdot\lVert\mathbf{v}_i\rVert}{\sum_{j=1}^t p_t(\mathbf{v}_j)\cdot\lVert\mathbf{v}_j\rVert},\qquad&
    \piminh(\mathbf{v}_i)\defeq\frac{\sum_{j=1}^t p_j(\mathbf{v}_i)\cdot h_j\cdot\lVert\mathbf{v}_i\rVert}{\sum_{l=1}^t\sum_{j=1}^t p_j(\mathbf{v}_l)\cdot h_j\cdot\lVert\mathbf{v}_l\rVert},\\
    \pihto(\mathbf{v}_i)\defeq\frac{\sum_{j=1}^t p_j(\mathbf{v}_i)}{\sum_{l=1}^t \sum_{j=1}^t p_j(\mathbf{v}_l)},\qquad
    &\pihtoh(\mathbf{v}_i)\defeq\frac{\sum_{j=1}^t p_j(\mathbf{v}_i)\cdot h_j}{\sum_{l=1}^t \sum_{j=1}^t p_j(\mathbf{v}_l)\cdot h_j};
\end{align}
where $h_j\defeq\left[(t-j+1)\cdot\mathcal{H}_t\right]^{-1}$ and $\mathcal{H}_t$ is the $t$-th harmonic number.

In this section, we discuss what exactly are the semantics of these distributions.
Proposal $\pimin$ is well known folklore result in the statistical methods literature \citep{owen13}.
It is simply the proposal distribution that minimizes the variance for the expectation estimator for the last query attention value, i.e.\ $\mathbb{E}_{p_t}[\mathbf{V}]$: the expectation at eviction time.
Notably, it is a proposal that uses the values of $\mathbf{V}$ in order to choose which entries to evict.
None of the top-$k$ eviction methods we compare and study make use of both $p(\mathbf{V})$ \emph{and} $\mathbf{V}$; see \Cref{app:subsumed-methods} for a description on how these eviction strategies compute their score.

Proposal $\piminh$ extends this to all distributions $\{p_i(\mathbf{V})\}_{i=1}^t$.
To be precise, recall that the softmax matrix in \Cref{eq:prob-interp} $p(\mathbf{V})$ at eviction time $t$ (i.e.\ during the prefilling step) is $t\times t$.
The first proposal $\pimin$ ignores all of the first $[1..t-1]$ rows, even though they may contain useful information to decide on which entries to evict.
We interpret these rows as a random variable $T$ that, when used to condition the probability of eviction, tells you which entries are best to evict.
That is,
\begin{equation}
    \softmax\left(\mathbf{Q}\cdot\mathbf{K}\tran\right)=\left[p(\mathbf{V}|T=i)\right]_{i=1}^t.
\end{equation}
This assumes that there exists a distribution on $T$.
We interpret this distribution $p(T)$ as some prior knowledge on how meaningful are each $p(\mathbf{V}|T=i)$ as $i$ distances from $t$.
In both $\piminh$ and $\pi$, we set this to
\begin{equation}
    h_j\defeq p(T=j)=\frac{1}{(t-j+1)\cdot\mathcal{H}_t}.
\end{equation}
The semantics of this prior is simple: we assume that each row has exponentially diminishing importance: we weigh each row according to the following sequence $\mathbf{h}=\left[\frac{1}{t},\frac{1}{t-1},\dots,\frac{1}{3},\frac{1}{2},1\right]$ and then normalize to get a probability distribution.
The normalizing constant of sequence $\mathbf{h}$ is the harmonic number $\mathcal{H}_t=\sum_{k=1}^t\frac{1}{k}$, and the probability of each row is $p(T=i)=\left[(t-j+1)\cdot\mathcal{H}_h\right]^{-1}$

To then retrieve the probability of each entry $p(\mathbf{v}_i)$, we marginalize over all possible time steps
\begin{equation}
    p(\mathbf{v}_i)=\sum_{j=1}^t p(\mathbf{v}_i,T=j)=\sum_{j=1}^t p(\mathbf{v}_i|T=j)\cdot p(T=j),
\end{equation}
which when using $p(T=j)=h_j$ for either the minimum variance proposal or the H2O proposal, yields $\piminh$ and $\pihtoh$ respectively.
Note that any distribution for $p(T)$ is valid and results in a different proposal.
In fact, any proposal that ignores the first $[1..t-1]$ rows contains a prior that is degenerated at $T=t$, with zero probability everywhere else.

\section{Additional Experiments}\label{app:importance-sampling-temp}

In this section, we show additional experiments that supplement the claims in \Cref{sec:robustness}.

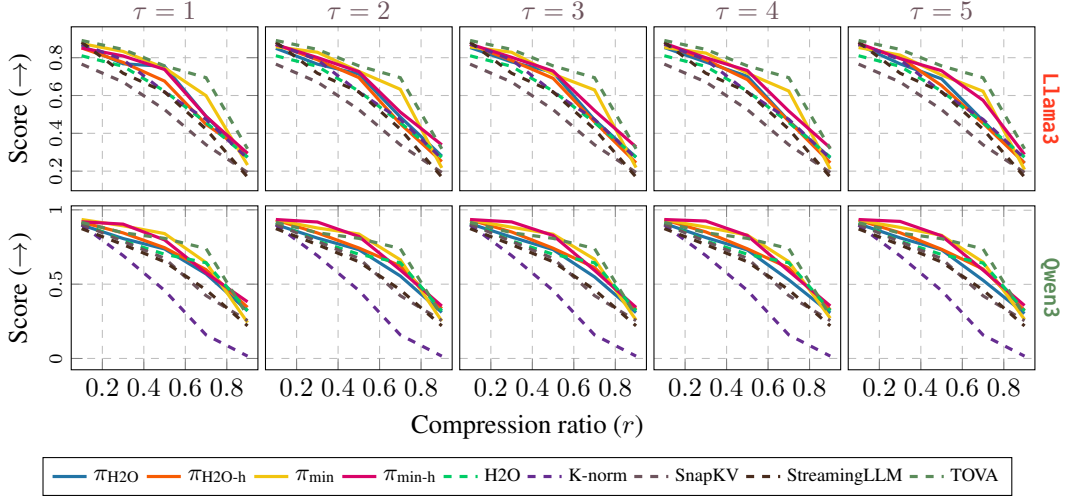
\begin{figure}[t]
\centering%
\begin{tikzpicture}
    \begin{groupplot}[
        group style={group size=5 by 2, horizontal sep=0.1cm, vertical sep=0.2cm},
        height=3.75cm,
        width=0.29\textwidth,
        grid style=dashed,
        cycle list name=cpalette,
        xmajorgrids=true, ymajorgrids=true,
        every legend image post/.append style={scale=0.75},
        xmin=0.05, xmax=0.95,
        title style={font=\small},
        y tick label style={font=\scriptsize, rotate=90},
        title style={yshift=-0.35cm,font={\strut{}\color{palette-gray}}},
        yticklabel={\pgfmathparse{\tick*0.01}\pgfmathprintnumber{\pgfmathresult}},
    ]

    \nextgroupplot[title={$\tau=1$},
                   ylabel={Score ($\to$)},xmajorticks=false]
        \addplot+[very thick] table[x=target_cr,y=overall,col sep=comma]{plots/ruler/llama3_3b/ist1.0/estimated_attention__h2o_scores.csv};
        \addplot+[very thick] table[x=target_cr,y=overall,col sep=comma]{plots/ruler/llama3_3b/ist1.0/estimated_attention__h2o_norm_scores.csv};
        \addplot+[very thick] table[x=target_cr,y=overall,col sep=comma]{plots/ruler/llama3_3b/ist1.0/estimated_attention__min_var_scores.csv};
        \addplot+[very thick] table[x=target_cr,y=overall,col sep=comma]{plots/ruler/llama3_3b/ist1.0/estimated_attention__min_var_full_harmonic_scores.csv};
        \addplot+[dashed,very thick] table[x=target_cr,y=overall,col sep=comma]{plots/ruler/llama3_3b/ist1.0/h2o_scores.csv};
        \addplot+[dashed,very thick] table[x=target_cr,y=overall,col sep=comma]{plots/ruler/llama3_3b/ist1.0/knorm_scores.csv};
        \addplot+[dashed,very thick] table[x=target_cr,y=overall,col sep=comma]{plots/ruler/llama3_3b/ist1.0/snapkv_scores.csv};
        \addplot+[dashed,very thick] table[x=target_cr,y=overall,col sep=comma]{plots/ruler/llama3_3b/ist1.0/streaming_llm_scores.csv};
        \addplot+[dashed,very thick] table[x=target_cr,y=overall,col sep=comma]{plots/ruler/llama3_3b/ist1.0/tova_scores.csv};
        \coordinate (istc1) at (rel axis cs:-0.275,1);

    \nextgroupplot[title={$\tau=2$},
                   ymajorticks=false,xmajorticks=false]
        \addplot+[very thick] table[x=target_cr,y=overall,col sep=comma]{plots/ruler/llama3_3b/ist2.0/estimated_attention__h2o_scores.csv};
        \addplot+[very thick] table[x=target_cr,y=overall,col sep=comma]{plots/ruler/llama3_3b/ist2.0/estimated_attention__h2o_norm_scores.csv};
        \addplot+[very thick] table[x=target_cr,y=overall,col sep=comma]{plots/ruler/llama3_3b/ist2.0/estimated_attention__min_var_scores.csv};
        \addplot+[very thick] table[x=target_cr,y=overall,col sep=comma]{plots/ruler/llama3_3b/ist2.0/estimated_attention__min_var_full_harmonic_scores.csv};
        \addplot+[dashed,very thick] table[x=target_cr,y=overall,col sep=comma]{plots/ruler/llama3_3b/ist2.0/h2o_scores.csv};
        \addplot+[dashed,very thick] table[x=target_cr,y=overall,col sep=comma]{plots/ruler/llama3_3b/ist2.0/knorm_scores.csv};
        \addplot+[dashed,very thick] table[x=target_cr,y=overall,col sep=comma]{plots/ruler/llama3_3b/ist2.0/snapkv_scores.csv};
        \addplot+[dashed,very thick] table[x=target_cr,y=overall,col sep=comma]{plots/ruler/llama3_3b/ist2.0/streaming_llm_scores.csv};
        \addplot+[dashed,very thick] table[x=target_cr,y=overall,col sep=comma]{plots/ruler/llama3_3b/ist2.0/tova_scores.csv};

    \nextgroupplot[title={$\tau=3$},
                   ymajorticks=false,xmajorticks=false]
        \addplot+[very thick] table[x=target_cr,y=overall,col sep=comma]{plots/ruler/llama3_3b/ist3.0/estimated_attention__h2o_scores.csv};
        \addplot+[very thick] table[x=target_cr,y=overall,col sep=comma]{plots/ruler/llama3_3b/ist3.0/estimated_attention__h2o_norm_scores.csv};
        \addplot+[very thick] table[x=target_cr,y=overall,col sep=comma]{plots/ruler/llama3_3b/ist3.0/estimated_attention__min_var_scores.csv};
        \addplot+[very thick] table[x=target_cr,y=overall,col sep=comma]{plots/ruler/llama3_3b/ist3.0/estimated_attention__min_var_full_harmonic_scores.csv};
        \addplot+[dashed,very thick] table[x=target_cr,y=overall,col sep=comma]{plots/ruler/llama3_3b/ist3.0/h2o_scores.csv};
        \addplot+[dashed,very thick] table[x=target_cr,y=overall,col sep=comma]{plots/ruler/llama3_3b/ist3.0/knorm_scores.csv};
        \addplot+[dashed,very thick] table[x=target_cr,y=overall,col sep=comma]{plots/ruler/llama3_3b/ist3.0/snapkv_scores.csv};
        \addplot+[dashed,very thick] table[x=target_cr,y=overall,col sep=comma]{plots/ruler/llama3_3b/ist3.0/streaming_llm_scores.csv};
        \addplot+[dashed,very thick] table[x=target_cr,y=overall,col sep=comma]{plots/ruler/llama3_3b/ist3.0/tova_scores.csv};

    \nextgroupplot[title={$\tau=4$},
                   ymajorticks=false,xmajorticks=false]
        \addplot+[very thick] table[x=target_cr,y=overall,col sep=comma]{plots/ruler/llama3_3b/ist4.0/estimated_attention__h2o_scores.csv};
        \addplot+[very thick] table[x=target_cr,y=overall,col sep=comma]{plots/ruler/llama3_3b/ist4.0/estimated_attention__h2o_norm_scores.csv};
        \addplot+[very thick] table[x=target_cr,y=overall,col sep=comma]{plots/ruler/llama3_3b/ist4.0/estimated_attention__min_var_scores.csv};
        \addplot+[very thick] table[x=target_cr,y=overall,col sep=comma]{plots/ruler/llama3_3b/ist4.0/estimated_attention__min_var_full_harmonic_scores.csv};
        \addplot+[dashed,very thick] table[x=target_cr,y=overall,col sep=comma]{plots/ruler/llama3_3b/ist4.0/h2o_scores.csv};
        \addplot+[dashed,very thick] table[x=target_cr,y=overall,col sep=comma]{plots/ruler/llama3_3b/ist4.0/knorm_scores.csv};
        \addplot+[dashed,very thick] table[x=target_cr,y=overall,col sep=comma]{plots/ruler/llama3_3b/ist4.0/snapkv_scores.csv};
        \addplot+[dashed,very thick] table[x=target_cr,y=overall,col sep=comma]{plots/ruler/llama3_3b/ist4.0/streaming_llm_scores.csv};
        \addplot+[dashed,very thick] table[x=target_cr,y=overall,col sep=comma]{plots/ruler/llama3_3b/ist4.0/tova_scores.csv};

    \nextgroupplot[title={$\tau=5$},
                   ymajorticks=false,xmajorticks=false]
        \addplot+[very thick] table[x=target_cr,y=overall,col sep=comma]{plots/ruler/llama3_3b/ist5.0/estimated_attention__h2o_scores.csv};
        \addplot+[very thick] table[x=target_cr,y=overall,col sep=comma]{plots/ruler/llama3_3b/ist5.0/estimated_attention__h2o_norm_scores.csv};
        \addplot+[very thick] table[x=target_cr,y=overall,col sep=comma]{plots/ruler/llama3_3b/ist5.0/estimated_attention__min_var_scores.csv};
        \addplot+[very thick] table[x=target_cr,y=overall,col sep=comma]{plots/ruler/llama3_3b/ist5.0/estimated_attention__min_var_full_harmonic_scores.csv};
        \addplot+[dashed,very thick] table[x=target_cr,y=overall,col sep=comma]{plots/ruler/llama3_3b/ist5.0/h2o_scores.csv};
        \addplot+[dashed,very thick] table[x=target_cr,y=overall,col sep=comma]{plots/ruler/llama3_3b/ist5.0/knorm_scores.csv};
        \addplot+[dashed,very thick] table[x=target_cr,y=overall,col sep=comma]{plots/ruler/llama3_3b/ist5.0/snapkv_scores.csv};
        \addplot+[dashed,very thick] table[x=target_cr,y=overall,col sep=comma]{plots/ruler/llama3_3b/ist5.0/streaming_llm_scores.csv};
        \addplot+[dashed,very thick] table[x=target_cr,y=overall,col sep=comma]{plots/ruler/llama3_3b/ist5.0/tova_scores.csv};
        \pgfplotsextra{\coordinate (istsllamaRow) at (rel axis cs:1,0.5);}

    \nextgroupplot[ylabel={Score ($\to$)}]
        \addplot+[very thick] table[x=target_cr,y=overall,col sep=comma]{plots/ruler/qwen3_4b/ist1.0/estimated_attention__h2o_scores.csv};
        \addplot+[very thick] table[x=target_cr,y=overall,col sep=comma]{plots/ruler/qwen3_4b/ist1.0/estimated_attention__h2o_norm_scores.csv};
        \addplot+[very thick] table[x=target_cr,y=overall,col sep=comma]{plots/ruler/qwen3_4b/ist1.0/estimated_attention__min_var_scores.csv};
        \addplot+[very thick] table[x=target_cr,y=overall,col sep=comma]{plots/ruler/qwen3_4b/ist1.0/estimated_attention__min_var_full_harmonic_scores.csv};
        \addplot+[dashed,very thick] table[x=target_cr,y=overall,col sep=comma]{plots/ruler/qwen3_4b/ist1.0/h2o_scores.csv};
        \addplot+[dashed,very thick] table[x=target_cr,y=overall,col sep=comma]{plots/ruler/qwen3_4b/ist1.0/knorm_scores.csv};
        \addplot+[dashed,very thick] table[x=target_cr,y=overall,col sep=comma]{plots/ruler/qwen3_4b/ist1.0/snapkv_scores.csv};
        \addplot+[dashed,very thick] table[x=target_cr,y=overall,col sep=comma]{plots/ruler/qwen3_4b/ist1.0/streaming_llm_scores.csv};
        \addplot+[dashed,very thick] table[x=target_cr,y=overall,col sep=comma]{plots/ruler/qwen3_4b/ist1.0/tova_scores.csv};

    \nextgroupplot[ymajorticks=false]
        \addplot+[very thick] table[x=target_cr,y=overall,col sep=comma]{plots/ruler/qwen3_4b/ist2.0/estimated_attention__h2o_scores.csv};
        \addplot+[very thick] table[x=target_cr,y=overall,col sep=comma]{plots/ruler/qwen3_4b/ist2.0/estimated_attention__h2o_norm_scores.csv};
        \addplot+[very thick] table[x=target_cr,y=overall,col sep=comma]{plots/ruler/qwen3_4b/ist2.0/estimated_attention__min_var_scores.csv};
        \addplot+[very thick] table[x=target_cr,y=overall,col sep=comma]{plots/ruler/qwen3_4b/ist2.0/estimated_attention__min_var_full_harmonic_scores.csv};
        \addplot+[dashed,very thick] table[x=target_cr,y=overall,col sep=comma]{plots/ruler/qwen3_4b/ist2.0/h2o_scores.csv};
        \addplot+[dashed,very thick] table[x=target_cr,y=overall,col sep=comma]{plots/ruler/qwen3_4b/ist2.0/knorm_scores.csv};
        \addplot+[dashed,very thick] table[x=target_cr,y=overall,col sep=comma]{plots/ruler/qwen3_4b/ist2.0/snapkv_scores.csv};
        \addplot+[dashed,very thick] table[x=target_cr,y=overall,col sep=comma]{plots/ruler/qwen3_4b/ist2.0/streaming_llm_scores.csv};
        \addplot+[dashed,very thick] table[x=target_cr,y=overall,col sep=comma]{plots/ruler/qwen3_4b/ist2.0/tova_scores.csv};

    \nextgroupplot[ymajorticks=false]
        \addplot+[very thick] table[x=target_cr,y=overall,col sep=comma]{plots/ruler/qwen3_4b/ist3.0/estimated_attention__h2o_scores.csv};
        \addplot+[very thick] table[x=target_cr,y=overall,col sep=comma]{plots/ruler/qwen3_4b/ist3.0/estimated_attention__h2o_norm_scores.csv};
        \addplot+[very thick] table[x=target_cr,y=overall,col sep=comma]{plots/ruler/qwen3_4b/ist3.0/estimated_attention__min_var_scores.csv};
        \addplot+[very thick] table[x=target_cr,y=overall,col sep=comma]{plots/ruler/qwen3_4b/ist3.0/estimated_attention__min_var_full_harmonic_scores.csv};
        \addplot+[dashed,very thick] table[x=target_cr,y=overall,col sep=comma]{plots/ruler/qwen3_4b/ist3.0/h2o_scores.csv};
        \addplot+[dashed,very thick] table[x=target_cr,y=overall,col sep=comma]{plots/ruler/qwen3_4b/ist3.0/knorm_scores.csv};
        \addplot+[dashed,very thick] table[x=target_cr,y=overall,col sep=comma]{plots/ruler/qwen3_4b/ist3.0/snapkv_scores.csv};
        \addplot+[dashed,very thick] table[x=target_cr,y=overall,col sep=comma]{plots/ruler/qwen3_4b/ist3.0/streaming_llm_scores.csv};
        \addplot+[dashed,very thick] table[x=target_cr,y=overall,col sep=comma]{plots/ruler/qwen3_4b/ist3.0/tova_scores.csv};

    \nextgroupplot[ymajorticks=false]
        \addplot+[very thick] table[x=target_cr,y=overall,col sep=comma]{plots/ruler/qwen3_4b/ist4.0/estimated_attention__h2o_scores.csv};
        \addplot+[very thick] table[x=target_cr,y=overall,col sep=comma]{plots/ruler/qwen3_4b/ist4.0/estimated_attention__h2o_norm_scores.csv};
        \addplot+[very thick] table[x=target_cr,y=overall,col sep=comma]{plots/ruler/qwen3_4b/ist4.0/estimated_attention__min_var_scores.csv};
        \addplot+[very thick] table[x=target_cr,y=overall,col sep=comma]{plots/ruler/qwen3_4b/ist4.0/estimated_attention__min_var_full_harmonic_scores.csv};
        \addplot+[dashed,very thick] table[x=target_cr,y=overall,col sep=comma]{plots/ruler/qwen3_4b/ist4.0/h2o_scores.csv};
        \addplot+[dashed,very thick] table[x=target_cr,y=overall,col sep=comma]{plots/ruler/qwen3_4b/ist4.0/knorm_scores.csv};
        \addplot+[dashed,very thick] table[x=target_cr,y=overall,col sep=comma]{plots/ruler/qwen3_4b/ist4.0/snapkv_scores.csv};
        \addplot+[dashed,very thick] table[x=target_cr,y=overall,col sep=comma]{plots/ruler/qwen3_4b/ist4.0/streaming_llm_scores.csv};
        \addplot+[dashed,very thick] table[x=target_cr,y=overall,col sep=comma]{plots/ruler/qwen3_4b/ist4.0/tova_scores.csv};

    \nextgroupplot[ymajorticks=false,
                   legend to name=istsleg2,
                   legend cell align=left,
                   legend style={font=\small,fill=none,draw=black,anchor=center,align=left},
                   legend columns=9]
        \addplot+[very thick] table[x=target_cr,y=overall,col sep=comma]{plots/ruler/qwen3_4b/ist5.0/estimated_attention__h2o_scores.csv};
        \addplot+[very thick] table[x=target_cr,y=overall,col sep=comma]{plots/ruler/qwen3_4b/ist5.0/estimated_attention__h2o_norm_scores.csv};
        \addplot+[very thick] table[x=target_cr,y=overall,col sep=comma]{plots/ruler/qwen3_4b/ist5.0/estimated_attention__min_var_scores.csv};
        \addplot+[very thick] table[x=target_cr,y=overall,col sep=comma]{plots/ruler/qwen3_4b/ist5.0/estimated_attention__min_var_full_harmonic_scores.csv};
        \addplot+[dashed,very thick] table[x=target_cr,y=overall,col sep=comma]{plots/ruler/qwen3_4b/ist5.0/h2o_scores.csv};
        \addplot+[dashed,very thick] table[x=target_cr,y=overall,col sep=comma]{plots/ruler/qwen3_4b/ist5.0/knorm_scores.csv};
        \addplot+[dashed,very thick] table[x=target_cr,y=overall,col sep=comma]{plots/ruler/qwen3_4b/ist5.0/snapkv_scores.csv};
        \addplot+[dashed,very thick] table[x=target_cr,y=overall,col sep=comma]{plots/ruler/qwen3_4b/ist5.0/streaming_llm_scores.csv};
        \addplot+[dashed,very thick] table[x=target_cr,y=overall,col sep=comma]{plots/ruler/qwen3_4b/ist5.0/tova_scores.csv};
        \legend{$\pihto$, $\pihtoh$, $\pimin$, $\piminh$, H2O, K-norm, SnapKV, StreamingLLM, TOVA}
        \coordinate (istc4) at (rel axis cs:1,1);
        \pgfplotsextra{\coordinate (istsqwenRow) at (rel axis cs:1,0.5);}

    \end{groupplot}
    \begin{scope}
        \coordinate (p) at ($(istc1)!0.5!(istc4)$);
        \node[below] (xlab) at (p |- current bounding box.south) {Compression ratio ($r$)};
        \node[below] (leg) at (xlab.south) {\pgfplotslegendfromname{robleg}};
        \node[rotate=270, anchor=north, font=\small, yshift=0.45cm] at (istsllamaRow) {\llama{}};
        \node[rotate=270, anchor=north, font=\small, yshift=0.45cm] at (istsqwenRow) {\qwen{}};
    \end{scope}
\end{tikzpicture}
\vspace{-0.25cm}
\caption{\textbf{Importance sampling temperature scaling average total scores.} \llama{} (top) and \qwen{} (bottom) average total scores at different importance sampling temperature scale values $\tau$.}\label{fig:ist-ruler-scores}
\end{figure}

\subsection{Per Dataset Split Scores}

We show the eviction scores for each of the four proposals defined in \Cref{sec:robustness} and the five top-$k$ eviction methods---namely, Streaming LLM, SnapKV, TOVA, H2O and K-norm---across all of the splits of LongBench and some of RULER's.
Due to space and time constraints, we only show six out of the 13 different splits in RULER.

\Cref{fig:individual-llama} shows the scores for \llama{} and \Cref{fig:individual-qwen} shows the scores for \qwen{}.
In both cases, probabilistic eviction with correction tends to achieve better score performance.
Notably, some eviction methods that do well in certain tasks.
For example, on \llama{} StreamingLLM ranks first on the QuestionAnswer split, yet dead last in MultiKey-NIAH.
Similarly, on \qwen{} K-norm is consistently in the top-3 on the MultiKey-NIAH split, yet last in CommonWords.

\subsection{Importance Sampling Temperature}

To further investigate the trade-off between bias and variance, we apply importance sampling temperature scaling to the same experimental setting as described in \Cref{sec:robustness}.
\Cref{fig:ist-ruler-scores,fig:ist-ruler-ranking} show the scores and pairwise win scores of RULER on \llama{} and \qwen{} across temperature values of $\tau\in\{1,2,3,4,5\}$.
The plots show that, in most cases, the decrease in variance was not worth the increase in bias.
This suggests that bias is the more important issue to resolve in KV eviction.

\subsection{On the Behavior of Attention Head Compression Ratio}\label{app:compression-pattern}

Given a fixed number of samples $m$, probabilistic eviction will automatically set the appropriate compression ratio of each head in accordance with that head's proposal distribution it samples from.
This means that different attention heads might have very distinct compression ratios then others.
Our findings show that this causes a distinct pattern in the distribution of attention head compression ratios.
We find that lower (closer to the input) layers tend to be allocated more of the cache budget, while at higher (closer to the output) layers, less budget.
This is a behavior that is reminiscent of \citet{cai2025pyramidkv}'s PyramidKV, where lower layers are given more of the KV cache budget.

\Cref{fig:pyramid} shows this pattern for different proposal distributions.
Notably, $\piminh$ tends to act more conservatively, while $\pihto$ spikes rapidly at certain layers.

\begin{figure}[t]
\begin{tikzpicture}
     \begin{groupplot}[
        group style={group size=4 by 2, horizontal sep=0.1cm, vertical sep=0.2cm},
        height=4.5cm,
        width=0.33\textwidth,
        grid style=dashed,
        cycle list name=cpalette,
        xmajorgrids=true, ymajorgrids=true,
        every legend image post/.append style={scale=0.75},
        title style={font=\small},
        y tick label style={font=\scriptsize, rotate=90},
        title style={yshift=-0.35cm,font={\strut{}\color{palette-gray}}},
        ymajorticks=false,
        xmin=-1, xmax=32,
    ]
        \nextgroupplot[title={$\pimin$},ymajorticks=true,ylabel={Compression ratio},xmajorticks=false]           
            \addplot[on layer=pre main,dashed,thick,palette-gray!60!white,domain=-1:32] {0.1};
            \addplot[on layer=pre main,dashed,thick,palette-gray!60!white,domain=-1:32] {0.25};
            \addplot[on layer=pre main,dashed,thick,palette-gray!60!white,domain=-1:32] {0.5};
            \addplot[on layer=pre main,dashed,thick,palette-gray!60!white,domain=-1:32] {0.75};
            \addplot[on layer=pre main,dashed,thick,palette-gray!60!white,domain=-1:32] {0.88};
            \pgfplotsset{cycle list shift=-5}
        
            \addplot+[very thick,x filter/.expression={\thisrow{target_cr}==0.1 ? x : nan}] table[x=layer,y=avg_effective_cr,col sep=comma] {plots/cr_behavior/llama_3_8b_min_var.csv};
            \addplot+[very thick,x filter/.expression={\thisrow{target_cr}==0.25 ? x : nan}] table[x=layer,y=avg_effective_cr,col sep=comma] {plots/cr_behavior/llama_3_8b_min_var.csv};
            \addplot+[very thick,x filter/.expression={\thisrow{target_cr}==0.5 ? x : nan}] table[x=layer,y=avg_effective_cr,col sep=comma] {plots/cr_behavior/llama_3_8b_min_var.csv};
            \addplot+[very thick,x filter/.expression={\thisrow{target_cr}==0.75 ? x : nan}] table[x=layer,y=avg_effective_cr,col sep=comma] {plots/cr_behavior/llama_3_8b_min_var.csv};
            \addplot+[very thick,x filter/.expression={\thisrow{target_cr}==0.88 ? x : nan}] table[x=layer,y=avg_effective_cr,col sep=comma] {plots/cr_behavior/llama_3_8b_min_var.csv}; 

            \coordinate (c1) at (rel axis cs:-0.275,1);
        \nextgroupplot[title={$\piminh$},xmajorticks=false]
            \addplot[on layer=pre main,dashed,thick,palette-gray!60!white,domain=-1:32] {0.1};
            \addplot[on layer=pre main,dashed,thick,palette-gray!60!white,domain=-1:32] {0.25};
            \addplot[on layer=pre main,dashed,thick,palette-gray!60!white,domain=-1:32] {0.5};
            \addplot[on layer=pre main,dashed,thick,palette-gray!60!white,domain=-1:32] {0.75};
            \addplot[on layer=pre main,dashed,thick,palette-gray!60!white,domain=-1:32] {0.88};
            \pgfplotsset{cycle list shift=-5}
        
            \addplot+[very thick,x filter/.expression={\thisrow{target_cr}==0.1 ? x : nan}] table[x=layer,y=avg_effective_cr,col sep=comma] {plots/cr_behavior/llama_3_8b_min_var_full_harmonic.csv};
            \addplot+[very thick,x filter/.expression={\thisrow{target_cr}==0.25 ? x : nan}] table[x=layer,y=avg_effective_cr,col sep=comma] {plots/cr_behavior/llama_3_8b_min_var_full_harmonic.csv};
            \addplot+[very thick,x filter/.expression={\thisrow{target_cr}==0.5 ? x : nan}] table[x=layer,y=avg_effective_cr,col sep=comma] {plots/cr_behavior/llama_3_8b_min_var_full_harmonic.csv};
            \addplot+[very thick,x filter/.expression={\thisrow{target_cr}==0.75 ? x : nan}] table[x=layer,y=avg_effective_cr,col sep=comma] {plots/cr_behavior/llama_3_8b_min_var_full_harmonic.csv};
            \addplot+[very thick,x filter/.expression={\thisrow{target_cr}==0.88 ? x : nan}] table[x=layer,y=avg_effective_cr,col sep=comma] {plots/cr_behavior/llama_3_8b_min_var_full_harmonic.csv}; 
        \nextgroupplot[title={$\pihto$},xmajorticks=false]
            \addplot[on layer=pre main,dashed,thick,palette-gray!60!white,domain=-1:32] {0.1};
            \addplot[on layer=pre main,dashed,thick,palette-gray!60!white,domain=-1:32] {0.25};
            \addplot[on layer=pre main,dashed,thick,palette-gray!60!white,domain=-1:32] {0.5};
            \addplot[on layer=pre main,dashed,thick,palette-gray!60!white,domain=-1:32] {0.75};
            \addplot[on layer=pre main,dashed,thick,palette-gray!60!white,domain=-1:32] {0.88};
            \pgfplotsset{cycle list shift=-5}
        
            \addplot+[very thick,x filter/.expression={\thisrow{target_cr}==0.1 ? x : nan}] table[x=layer,y=avg_effective_cr,col sep=comma] {plots/cr_behavior/llama_3_8b_h2o.csv};
            \addplot+[very thick,x filter/.expression={\thisrow{target_cr}==0.25 ? x : nan}] table[x=layer,y=avg_effective_cr,col sep=comma] {plots/cr_behavior/llama_3_8b_h2o.csv};
            \addplot+[very thick,x filter/.expression={\thisrow{target_cr}==0.5 ? x : nan}] table[x=layer,y=avg_effective_cr,col sep=comma] {plots/cr_behavior/llama_3_8b_h2o.csv};
            \addplot+[very thick,x filter/.expression={\thisrow{target_cr}==0.75 ? x : nan}] table[x=layer,y=avg_effective_cr,col sep=comma] {plots/cr_behavior/llama_3_8b_h2o.csv};
            \addplot+[very thick,x filter/.expression={\thisrow{target_cr}==0.88 ? x : nan}] table[x=layer,y=avg_effective_cr,col sep=comma] {plots/cr_behavior/llama_3_8b_h2o.csv}; 
        \nextgroupplot[title={$\pihtoh$},xmajorticks=false]
            \addplot[on layer=pre main,dashed,thick,palette-gray!60!white,domain=-1:32] {0.1};
            \addplot[on layer=pre main,dashed,thick,palette-gray!60!white,domain=-1:32] {0.25};
            \addplot[on layer=pre main,dashed,thick,palette-gray!60!white,domain=-1:32] {0.5};
            \addplot[on layer=pre main,dashed,thick,palette-gray!60!white,domain=-1:32] {0.75};
            \addplot[on layer=pre main,dashed,thick,palette-gray!60!white,domain=-1:32] {0.88};
            \pgfplotsset{cycle list shift=-5}
        
            \addplot+[very thick,x filter/.expression={\thisrow{target_cr}==0.1 ? x : nan}] table[x=layer,y=avg_effective_cr,col sep=comma] {plots/cr_behavior/llama_3_8b_h2o_norm.csv};
            \addplot+[very thick,x filter/.expression={\thisrow{target_cr}==0.25 ? x : nan}] table[x=layer,y=avg_effective_cr,col sep=comma] {plots/cr_behavior/llama_3_8b_h2o_norm.csv};
            \addplot+[very thick,x filter/.expression={\thisrow{target_cr}==0.5 ? x : nan}] table[x=layer,y=avg_effective_cr,col sep=comma] {plots/cr_behavior/llama_3_8b_h2o_norm.csv};
            \addplot+[very thick,x filter/.expression={\thisrow{target_cr}==0.75 ? x : nan}] table[x=layer,y=avg_effective_cr,col sep=comma] {plots/cr_behavior/llama_3_8b_h2o_norm.csv};
            \addplot+[very thick,x filter/.expression={\thisrow{target_cr}==0.88 ? x : nan}] table[x=layer,y=avg_effective_cr,col sep=comma] {plots/cr_behavior/llama_3_8b_h2o_norm.csv}; 

            \pgfplotsextra{\coordinate (llamaRow) at (rel axis cs:1,0.5);}
            
        \nextgroupplot[ymajorticks=true,ylabel={Compression ratio}]           
            \addplot[on layer=pre main,dashed,thick,palette-gray!60!white,domain=-1:32] {0.1};
            \addplot[on layer=pre main,dashed,thick,palette-gray!60!white,domain=-1:32] {0.25};
            \addplot[on layer=pre main,dashed,thick,palette-gray!60!white,domain=-1:32] {0.5};
            \addplot[on layer=pre main,dashed,thick,palette-gray!60!white,domain=-1:32] {0.75};
            \addplot[on layer=pre main,dashed,thick,palette-gray!60!white,domain=-1:32] {0.88};
            \pgfplotsset{cycle list shift=-5}
        
            \addplot+[very thick,x filter/.expression={\thisrow{target_cr}==0.1 ? x : nan}] table[x=layer,y=avg_effective_cr,col sep=comma] {plots/cr_behavior/qwen_3_8b_min_var.csv};
            \addplot+[very thick,x filter/.expression={\thisrow{target_cr}==0.25 ? x : nan}] table[x=layer,y=avg_effective_cr,col sep=comma] {plots/cr_behavior/qwen_3_8b_min_var.csv};
            \addplot+[very thick,x filter/.expression={\thisrow{target_cr}==0.5 ? x : nan}] table[x=layer,y=avg_effective_cr,col sep=comma] {plots/cr_behavior/qwen_3_8b_min_var.csv};
            \addplot+[very thick,x filter/.expression={\thisrow{target_cr}==0.75 ? x : nan}] table[x=layer,y=avg_effective_cr,col sep=comma] {plots/cr_behavior/qwen_3_8b_min_var.csv};
            \addplot+[very thick,x filter/.expression={\thisrow{target_cr}==0.88 ? x : nan}] table[x=layer,y=avg_effective_cr,col sep=comma] {plots/cr_behavior/qwen_3_8b_min_var.csv}; 
        \nextgroupplot
            \addplot[on layer=pre main,dashed,thick,palette-gray!60!white,domain=-1:32] {0.1};
            \addplot[on layer=pre main,dashed,thick,palette-gray!60!white,domain=-1:32] {0.25};
            \addplot[on layer=pre main,dashed,thick,palette-gray!60!white,domain=-1:32] {0.5};
            \addplot[on layer=pre main,dashed,thick,palette-gray!60!white,domain=-1:32] {0.75};
            \addplot[on layer=pre main,dashed,thick,palette-gray!60!white,domain=-1:32] {0.88};
            \pgfplotsset{cycle list shift=-5}
        
            \addplot+[very thick,x filter/.expression={\thisrow{target_cr}==0.1 ? x : nan}] table[x=layer,y=avg_effective_cr,col sep=comma] {plots/cr_behavior/qwen_3_8b_min_var_full_harmonic.csv};
            \addplot+[very thick,x filter/.expression={\thisrow{target_cr}==0.25 ? x : nan}] table[x=layer,y=avg_effective_cr,col sep=comma] {plots/cr_behavior/qwen_3_8b_min_var_full_harmonic.csv};
            \addplot+[very thick,x filter/.expression={\thisrow{target_cr}==0.5 ? x : nan}] table[x=layer,y=avg_effective_cr,col sep=comma] {plots/cr_behavior/qwen_3_8b_min_var_full_harmonic.csv};
            \addplot+[very thick,x filter/.expression={\thisrow{target_cr}==0.75 ? x : nan}] table[x=layer,y=avg_effective_cr,col sep=comma] {plots/cr_behavior/qwen_3_8b_min_var_full_harmonic.csv};
            \addplot+[very thick,x filter/.expression={\thisrow{target_cr}==0.88 ? x : nan}] table[x=layer,y=avg_effective_cr,col sep=comma] {plots/cr_behavior/qwen_3_8b_min_var_full_harmonic.csv}; 
        \nextgroupplot
            \addplot[on layer=pre main,dashed,thick,palette-gray!60!white,domain=-1:32] {0.1};
            \addplot[on layer=pre main,dashed,thick,palette-gray!60!white,domain=-1:32] {0.25};
            \addplot[on layer=pre main,dashed,thick,palette-gray!60!white,domain=-1:32] {0.5};
            \addplot[on layer=pre main,dashed,thick,palette-gray!60!white,domain=-1:32] {0.75};
            \addplot[on layer=pre main,dashed,thick,palette-gray!60!white,domain=-1:32] {0.88};
            \pgfplotsset{cycle list shift=-5}
        
            \addplot+[very thick,x filter/.expression={\thisrow{target_cr}==0.1 ? x : nan}] table[x=layer,y=avg_effective_cr,col sep=comma] {plots/cr_behavior/qwen_3_8b_h2o.csv};
            \addplot+[very thick,x filter/.expression={\thisrow{target_cr}==0.25 ? x : nan}] table[x=layer,y=avg_effective_cr,col sep=comma] {plots/cr_behavior/qwen_3_8b_h2o.csv};
            \addplot+[very thick,x filter/.expression={\thisrow{target_cr}==0.5 ? x : nan}] table[x=layer,y=avg_effective_cr,col sep=comma] {plots/cr_behavior/qwen_3_8b_h2o.csv};
            \addplot+[very thick,x filter/.expression={\thisrow{target_cr}==0.75 ? x : nan}] table[x=layer,y=avg_effective_cr,col sep=comma] {plots/cr_behavior/qwen_3_8b_h2o.csv};
            \addplot+[very thick,x filter/.expression={\thisrow{target_cr}==0.88 ? x : nan}] table[x=layer,y=avg_effective_cr,col sep=comma] {plots/cr_behavior/qwen_3_8b_h2o.csv}; 
        \nextgroupplot[legend to name=leg,
                   legend cell align=left,
                   legend style={font=\small,fill=none,draw=black,anchor=center,align=left},
                   legend columns=5]
            \addplot[on layer=pre main,dashed,thick,palette-gray!60!white,domain=-1:32] {0.1};
            \addplot[on layer=pre main,dashed,thick,palette-gray!60!white,domain=-1:32] {0.25};
            \addplot[on layer=pre main,dashed,thick,palette-gray!60!white,domain=-1:32] {0.5};
            \addplot[on layer=pre main,dashed,thick,palette-gray!60!white,domain=-1:32] {0.75};
            \addplot[on layer=pre main,dashed,thick,palette-gray!60!white,domain=-1:32] {0.88};
            \pgfplotsset{cycle list shift=-5}
        
            \addplot+[very thick,x filter/.expression={\thisrow{target_cr}==0.1 ? x : nan}] table[x=layer,y=avg_effective_cr,col sep=comma] {plots/cr_behavior/qwen_3_8b_h2o_norm.csv};
            \addplot+[very thick,x filter/.expression={\thisrow{target_cr}==0.25 ? x : nan}] table[x=layer,y=avg_effective_cr,col sep=comma] {plots/cr_behavior/qwen_3_8b_h2o_norm.csv};
            \addplot+[very thick,x filter/.expression={\thisrow{target_cr}==0.5 ? x : nan}] table[x=layer,y=avg_effective_cr,col sep=comma] {plots/cr_behavior/qwen_3_8b_h2o_norm.csv};
            \addplot+[very thick,x filter/.expression={\thisrow{target_cr}==0.75 ? x : nan}] table[x=layer,y=avg_effective_cr,col sep=comma] {plots/cr_behavior/qwen_3_8b_h2o_norm.csv};
            \addplot+[very thick,x filter/.expression={\thisrow{target_cr}==0.88 ? x : nan}] table[x=layer,y=avg_effective_cr,col sep=comma] {plots/cr_behavior/qwen_3_8b_h2o_norm.csv}; 

            \legend{,,,,,$r=0.1$,$r=0.25$,$r=0.5$,$r=75$,$r=0.88$}
            
            \coordinate (istc4) at (rel axis cs:1,1);
            \pgfplotsextra{\coordinate (qwenRow) at (rel axis cs:1,0.5);}
   \end{groupplot}
   \begin{scope}
        \coordinate (p) at ($(c1)!0.5!(c4)$);
        \node[below] (xlab) at (p |- current bounding box.south) {Layers};
        \node[below] at (xlab.south) {\pgfplotslegendfromname{leg}};
        \node[rotate=270, anchor=north, font=\small, yshift=0.45cm] at (llamaRow) {\llama{}};
        \node[rotate=270, anchor=north, font=\small, yshift=0.45cm] at (qwenRow) {\qwen{}};
    \end{scope}
\end{tikzpicture}
\caption{\textbf{Pattern of average compression ratio per layer when probabilistically evicting.} Each solid curve shows the average effective compression ratio when given a target compression ratio $r$. The target compression ratio is shown as a dashed dark gray line. The average of all points in a solid curve equals to the target $r$. This pattern shows that lower layers (closer to input) tend to evict less, while higher layers (closer to output) evict more.}\label{fig:pyramid}
\end{figure}
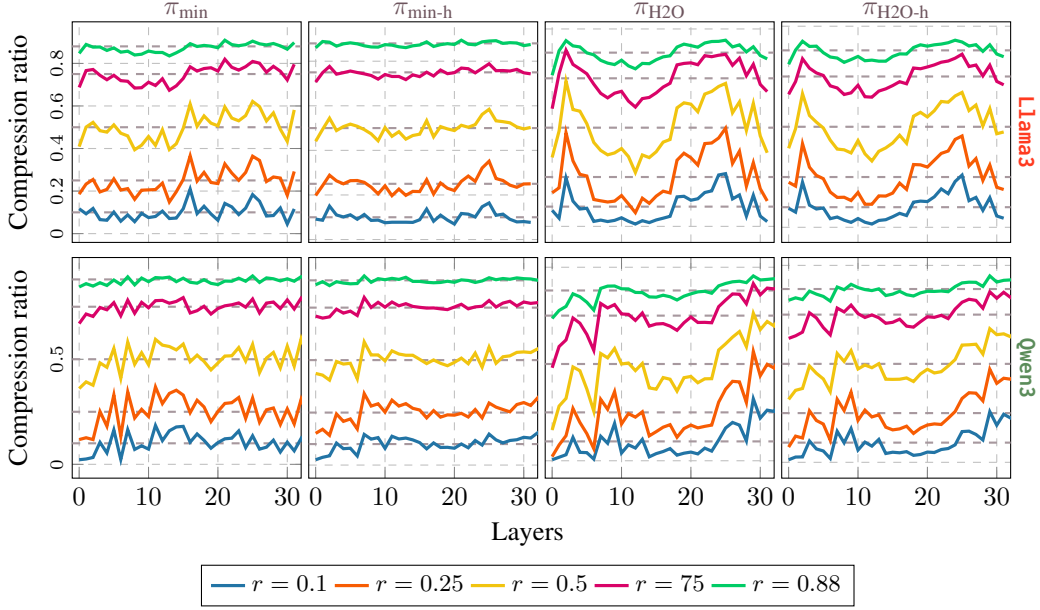

\section{Limitations}\label{app:limitations}

The ideal KV eviction method is not only zero-variance but also zero-bias.
However, it is unrealistic to achieve this, as to do so, correction would have to be perfect, meaning that it can perfectly recreate the same KV cache from less memory than what was evicted.
Our approach is also not perfect: self-normalized importance sampling is biased and has variance greater than zero.
However, this estimator has reasonable upper bounds on bias and variance.
This is in contrast to deterministic KV eviction methods that do not apply correction: these are zero-variance estimators but their bias is unbounded.

The bias of self-normalized importance sampling is upper bounded by $\bigo(m^{-1})$, where $m$ is the number of samples. This upper bound indicates that if there is less compression, then it will be more accurate---which is usually true for other eviction methods.
Note that the inverse is not true: at a reasonable compression ratio, if the distributions at the attention heads are sparse, then the effective number of samples required to achieve a certain precision will be lower, meaning that in those cases even a higher compression ratio can achieve lower error.

In terms of implementation, we show that we are able to achieve distinct compression ratios at different heads.
This requires a properly optimized KV eviction implementation to compute correction using an efficient sparse tensors implementation.
This is addressed by \citet{feng26}, and can be extended to our case.
Further, correction requires an additional memory overhead of $\bigo(k\cdot h\cdot b)$ bytes, where $k$ is the number of kept tokens, $h$ is the total number of attention heads in the LLM, and $b$ is the floating point precision in bytes.
Asymptotically, this is dwarfed by the KV cache size itself and not a real concern, as the cache size is $\bigo(d\cdot k\cdot h\cdot b)$, where $d$ is the embedding dimension.

\end{document}